\PassOptionsToPackage{unicode}{hyperref}
\PassOptionsToPackage{hyphens}{url}
\PassOptionsToPackage{dvipsnames,svgnames,x11names}{xcolor}
\documentclass[
  12pt]{article}

\usepackage{kotex}
\usepackage{amsmath,amssymb}
\usepackage{amsthm}
\usepackage{mathtools}
\usepackage{iftex}
\ifPDFTeX
  \usepackage[T1]{fontenc}
  \usepackage[utf8]{inputenc}
  \usepackage{textcomp} 
\else 
  \usepackage{unicode-math}
  \defaultfontfeatures{Scale=MatchLowercase}
  \defaultfontfeatures[\rmfamily]{Ligatures=TeX,Scale=1}
\fi
\usepackage{lmodern}
\ifPDFTeX\else  
\fi
\IfFileExists{upquote.sty}{\usepackage{upquote}}{}
\IfFileExists{microtype.sty}{
  \usepackage[]{microtype}
  \UseMicrotypeSet[protrusion]{basicmath} 
}{}
\makeatletter
\@ifundefined{KOMAClassName}{
  \IfFileExists{parskip.sty}{%
    \usepackage{parskip}
  }{
    \setlength{\parindent}{0pt}
    \setlength{\parskip}{6pt plus 2pt minus 1pt}}
}{
  \KOMAoptions{parskip=half}}
\makeatother
\usepackage{xcolor}
\makeatletter
\ifx\paragraph\undefined\else
  \let\oldparagraph\paragraph
  \renewcommand{\paragraph}{
    \@ifstar
      \xxxParagraphStar
      \xxxParagraphNoStar
  }
  \newcommand{\xxxParagraphStar}[1]{\oldparagraph*{#1}\mbox{}}
  \newcommand{\xxxParagraphNoStar}[1]{\oldparagraph{#1}\mbox{}}
\fi
\ifx\subparagraph\undefined\else
  \let\oldsubparagraph\subparagraph
  \renewcommand{\subparagraph}{
    \@ifstar
      \xxxSubParagraphStar
      \xxxSubParagraphNoStar
  }
  \newcommand{\xxxSubParagraphStar}[1]{\oldsubparagraph*{#1}\mbox{}}
  \newcommand{\xxxSubParagraphNoStar}[1]{\oldsubparagraph{#1}\mbox{}}
\fi
\makeatother

\usepackage{bbm}
\usepackage{enumitem}
\usepackage{hyperref}
\usepackage{multirow}
\usepackage{algorithm}
\usepackage{algpseudocode} 

\newtheorem{theorem}{Theorem}
\newtheorem{lemma}{Lemma}
\newtheorem{assumption}{Assumption}
\newtheorem{corollary}{Corollary}
\newtheorem{definition}{Definition}
\newtheorem{remark}{Remark}
\newtheorem{proposition}{Proposition}
\newtheorem{example}{Example}

\usepackage{longtable,booktabs,array}
\usepackage{calc} 
\usepackage{etoolbox}
\makeatletter
\patchcmd\longtable{\par}{\if@noskipsec\mbox{}\fi\par}{}{}
\makeatother
\IfFileExists{footnotehyper.sty}{\usepackage{footnotehyper}}{\usepackage{footnote}}
\makesavenoteenv{longtable}
\usepackage{graphicx}
\makeatletter
\def\maxwidth{\ifdim\Gin@nat@width>\linewidth\linewidth\else\Gin@nat@width\fi}
\def\maxheight{\ifdim\Gin@nat@height>\textheight\textheight\else\Gin@nat@height\fi}
\makeatother
\setkeys{Gin}{width=\maxwidth,height=\maxheight,keepaspectratio}
\makeatletter
\def\fps@figure{htbp}
\makeatother

\makeatletter
\@ifpackageloaded{caption}{}{\usepackage{caption}}
\AtBeginDocument{%
\ifdefined\contentsname
  \renewcommand*\contentsname{Table of contents}
\else
  \newcommand\contentsname{Table of contents}
\fi
\ifdefined\listfigurename
  \renewcommand*\listfigurename{List of Figures}
\else
  \newcommand\listfigurename{List of Figures}
\fi
\ifdefined\listtablename
  \renewcommand*\listtablename{List of Tables}
\else
  \newcommand\listtablename{List of Tables}
\fi
\ifdefined\figurename
  \renewcommand*\figurename{Figure}
\else
  \newcommand\figurename{Figure}
\fi
\ifdefined\tablename
  \renewcommand*\tablename{Table}
\else
  \newcommand\tablename{Table}
\fi
}
\@ifpackageloaded{float}{}{\usepackage{float}}
\floatstyle{ruled}
\@ifundefined{c@chapter}{\newfloat{codelisting}{h}{lop}}{\newfloat{codelisting}{h}{lop}[chapter]}
\floatname{codelisting}{Listing}

\makeatother
\makeatletter
\@ifpackageloaded{caption}{}{\usepackage{caption}}
\@ifpackageloaded{subcaption}{}{\usepackage{subcaption}}
\makeatother

\ifLuaTeX
  \usepackage{selnolig}  
\fi
\usepackage[]{natbib}
\usepackage{bookmark}

\IfFileExists{xurl.sty}{\usepackage{xurl}}{} 
\hypersetup{
  pdftitle={Beyond Data Splitting: Full-Data Conformal Prediction by Differential Privacy},
  pdfauthor={Young Hyun Cho and Jordan Awan},
  pdfkeywords={uncertainty quantification, algorithmic stability, prediction sets, private learning, quantile estimation},
  colorlinks=true,
  linkcolor={blue},
  filecolor={Maroon},
  citecolor={Blue},
  urlcolor={Blue},
  pdfcreator={LaTeX}
}

\newcommand{\anon}{1}

\begin{document}

\def\spacingset#1{\renewcommand{\baselinestretch}%
{#1}\small\normalsize} \spacingset{1}


\if1\anon
{
  \title{\bf Beyond Data Splitting: Full-Data Conformal
Prediction by Differential Privacy}
  \author{Young Hyun Cho\thanks{Correspondence: \texttt{younghyuncho@fas.harvard.edu}. This work was primarily conducted while Young Hyun Cho was at Purdue University. Young Hyun Cho's work was partially supported by NIH P50DA054039.}\\
    Department of Statistics, Harvard University
    \and
    Jordan Awan\thanks{Jordan Awan acknowledges support from NSF grants SES-2150615 and SES-2610910.}\\
    Department of Statistics, University of Pittsburgh}
  \maketitle
} \fi

\if0\anon
{
  \bigskip
  \bigskip
  \bigskip
  \begin{center}
    {\LARGE\bf Beyond Data Splitting: Full-Data Conformal
Prediction by Differential Privacy}
\end{center}
  \medskip
} \fi

\bigskip
\begin{abstract}
Uncertainty quantification and privacy protection are increasingly important in data-driven decision making. Conformal prediction relies on exchangeability, typically obtained through data splitting or repeated training. Under DP, data splitting leaves less data for training and calibration, while repeated training accumulates privacy loss. We propose DP-Stabilized Conformal Prediction (DP-SCP), a framework that uses all the data for training and calibration with a single training run. DP-SCP uses the stability of DP to bound the distance to exchangeability, with two variations, DP-SCP-F and DP-SCP-A. DP-SCP-F uses a more conservative calibration target to prioritize finite-sample coverage, whereas DP-SCP-A targets asymptotic validity, allowing for more informative prediction sets. To our knowledge, we are the first to show exactly what DP alone guarantees for coverage in conformal prediction, and we prove that the resulting gap from nominal coverage is sharp and depends on the privacy level. Under regularity assumptions on the model and training mechanism, both variants achieve the nominal \(1-\alpha\) coverage asymptotically. Empirically, DP-SCP generally produces more informative prediction sets than an existing private data-splitting approach under the same privacy guarantee.
\end{abstract}

\noindent%
\vfill

\newpage
\spacingset{1.8} 

\section{Introduction}\label{sec:intro}

As machine learning (ML) methods are increasingly deployed in high-stakes domains such as healthcare and finance, ensuring their reliability has become increasingly important \citep{bastos2024conformal,kladny2025critical,shahbazi2026adaptive}. This calls for two complementary safeguards: uncertainty quantification to assess reliability and privacy protection. Conformal prediction (CP) \citep{vovk2005algorithmic} provides the former, while differential privacy (DP) \citep{dwork2006calibrating} serves as the \emph{de facto} standard for the latter. Yet methods that satisfy both requirements remain underdeveloped, despite their importance for trustworthy ML.

Private CP methods have largely relied on data splitting \citep{angelopoulos2022private,romanus2025differentially}. Under privacy, this division is especially costly. It leaves less data for learning the model and less data for locating the calibration threshold, while both stages already incur privacy noise. Full CP and leave-one-out-type methods avoid a fixed split but rely on repeated model training \citep{vovk2005algorithmic,barber2021predictive}. However, repeated applications of DP mechanisms to the same data accumulate privacy loss, so ensuring an overall privacy guarantee requires stronger perturbation in each application, in addition to the computational cost of retraining. These limitations lead to a central question:

\begin{center}
\emph{Can we use the full data for training and calibration under DP while training only once?}
\end{center}

Using the full data for both training and calibration, however, may not ensure exchangeability. The training data affect the trained model, whereas the test point does not, so the resulting conformity scores need not be exchangeable even when the data are i.i.d. Instead of \textcolor{blue}{ensuring} exchangeability, one can quantify how far the scores depart from an exchangeable counterpart and account for the discrepancy. Prior nonprivate CP work has developed settings where distance to exchangeability can be quantified \citep{tibshirani2019conformal,barber2023conformal,barber2026unifying}. However, quantifying this distance requires assumptions on the data-generating process, which may be hard to verify in practice, especially since CP is designed to be a distribution-free approach.

\begin{figure}[t]
\includegraphics[width=0.9\textwidth]
{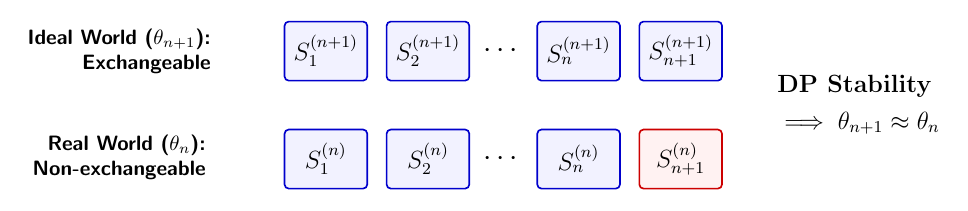}
\caption{Conceptual illustration of the DP stability. If \(\theta_{n+1}\) is trained on the augmented data, all the \(n+1\) scores are generated symmetrically. When \(\theta_n\) is trained on \(D_n\), the first (n) scores are in-sample and \(S_{n+1}^{(n)}\) is out-of-sample. DP-SCP bounds the validity gap.}
\label{fig:stability}
\end{figure}

In this paper, we examine whether a DP mechanism can limit the distance to exchangeability. Intuitively, changing one datapoint, particularly the test point, has a limited effect on the mechanism's output distribution. In Figure~\ref{fig:stability}, the ideal construction trains on \(D_{n+1}\), while the available construction trains on \(D_n\). These datasets differ by a single entry, so DP bounds how much the output distributions can differ as a function of the privacy parameter. We use this bound to derive a coverage guarantee that depends on the nominal coverage under an idealized, exchangeable setting, as well as the privacy parameters.

We leverage the stability of DP under a single-data-point change to develop DP-Stabilized Conformal Prediction (DP-SCP). DP-SCP trains a single private model on all the sensitive data and privately calibrates the scores using those data again. It therefore splits the privacy budget between training and calibration. Moreover, we propose two variants of DP-SCP: DP-SCP-F, which prioritizes finite-sample coverage by a more conservative calibration target, and DP-SCP-A, which uses a less conservative calibration target while retaining asymptotic validity. In empirical studies across real and synthetic data, DP-SCP generally yields more informative prediction sets than private split-based CP under the same DP guarantee, while the behavior of its two variants is consistent with our theoretical results.

\subsection{Related Literature and Our Contributions}

Our work builds upon three lines of research: private CP, nonprivate work on full-data CP and distance to exchangeability, and work linking DP and algorithmic stability.

Private CP has primarily separated training from calibration and focused on privatizing the calibration step \citep{angelopoulos2022private,romanus2025differentially}, with extensions to local and online settings \citep{penso2025privacy,zhang2025online}. In these prior works, DP is treated primarily as a source of noise and statistical cost. Concurrent work develops full-data private CP with coverage guarantees based on DP stability \citep{wu2026differentially}. However, it does not characterize the sharp lower bound on coverage given only the DP guarantee.

Outside DP, full-data CP has primarily relied on retraining to recover the exchangeability needed for validity \citep{vovk2005algorithmic,lei2018distribution,kim2020predictive,barber2021predictive}, with stability-based methods reducing or analyzing this burden \citep{ndiaye2021stable,liang2025algorithmic}. A separate line studies CP beyond exact exchangeability \citep{tibshirani2019conformal,barber2023conformal,cauchois2024robust,oliveira2024split,hore2025conformal,barber2026unifying}. This line quantifies the distance to exchangeability using distributional metrics such as total variation (TV) distance. Since these distances are generally not determined by the observed data alone, explicit coverage bounds require additional structural assumptions, such as restrictions on the data-generating process.

DP need not be viewed only as a statistical cost. Algorithmic stability studies how a learning algorithm responds to data perturbations and uses this behavior to establish generalization \citep{bousquet2002stability,hardt2016train}. A related line views DP as a form of stability and uses it to preserve validity under adaptive or repeated data use \citep{dwork2015generalization,dwork2015preserving,bassily2021algorithmic}. How this perspective can be used in CP remains less understood.

Building on these literatures, we make the following contributions:
\clearpage
{
\setlength{\leftmargini}{1.5em}
\begin{itemize}

\item \textbf{Methodological.} We introduce DP-SCP, a full-data CP framework that uses DP both to protect the data and to quantify the distance to exchangeability. DP-SCP uses the same dataset for training and calibration without data splitting and requires only a single private training run. We develop DP-SCP in two variants to accommodate different priorities: DP-SCP-F uses a more conservative calibration target to protect finite-sample coverage, whereas DP-SCP-A uses a less conservative target to obtain more informative prediction sets while retaining asymptotic validity.

\item \textbf{Theoretical.} To our knowledge, we are the first to leverage DP stability to establish coverage guarantees in CP. More specifically, we prove a sharp lower bound when one knows only the DP guarantee, but not the mechanism details. However, this bound has a coverage gap that depends on the privacy parameter. On the other hand, many common DP training mechanisms have substantially stronger stability properties than are suggested by this lower bound. Under additional regularity assumptions, we show that the coverage converges to the nominal level.

\item \textbf{Empirical.} In the real-data studies, DP-SCP generally yields more informative prediction sets than a private Split CP baseline under the same DP guarantee in both classification and regression. In the synthetic data experiments, we further examine the implications of our theoretical analysis under controlled variations in privacy and sample size.

\end{itemize}
}

The remainder of the paper is organized as follows. Section~\ref{sec:background} reviews the background and motivates using the full data for training and calibration in private CP. Section~\ref{sec:proposed_framework} presents DP-SCP and its private calibration. Section~\ref{sec:theory} develops coverage guarantees from the DP guarantee alone and under additional regularity assumptions on the model and training mechanism. Section~\ref{sec:numerical_study} reports the numerical studies, and Section~\ref{sec:conclusion} concludes. Supplementary Sections~\ref{sec:supp_dp_intro}--\ref{supp:extended numerical studies} provide proofs, additional technical details, and extended experiments.

\section{Background and Motivation}\label{sec:background}

This section provides background and motivation for full-data private CP.

\subsection{Differential Privacy}

Differential Privacy (DP) \citep{dwork2006calibrating} is a framework that quantifies the privacy risk about releasing outputs from sensitive datasets. Intuitively, $M$ is DP if $M(D)$ is nearly indistinguishable from $M(D')$, where $D$ and $D'$ are adjacent datasets. In this paper, two datasets are adjacent if one can be obtained from the other by adding or deleting a single entry. DP is naturally formalized via hypothesis testing with a tradeoff function $T(P, Q): [0, 1] \to [0, 1]$, the minimum Type II error for testing $P$ vs. $Q$ at Type I error $\alpha$: $T(P,Q)(\alpha) = \inf_{\phi} \{ 1 - \mathbb{E}_{Q}[\phi] : \mathbb{E}_{P}[\phi] \le \alpha \}$. A function $f:[0,1]\to[0,1]$ is a tradeoff function if and only if it is convex, continuous, non-increasing, and satisfies $f(\alpha)\le 1-\alpha$.

\begin{definition}[$f$-DP \citep{dong2022gaussian}]
Let $f$ be a symmetric tradeoff function. A mechanism $M$ satisfies $f$-DP if, for all adjacent datasets $D$ and $D'$, $T(M(D),M(D')) \ge f$.
\end{definition}

Many DP notions are special cases of $f$-DP. In particular, $(\varepsilon,\delta)$-DP is equivalent to $f_{\varepsilon,\delta}$-DP, where for $\varepsilon\ge0$ and $\delta\in[0,1]$, $f_{\varepsilon,\delta}(\alpha)=\max\{0,1-\delta-e^\varepsilon\alpha,e^{-\varepsilon}(1-\delta-\alpha)\}$. Similarly, $\mu$-Gaussian DP ($\mu$-GDP; \cite{dong2022gaussian}) is $G_\mu$-DP, where $G_\mu=T(N(0,1),N(\mu,1))$. In both cases, smaller privacy parameters imply a stronger privacy guarantee.

For a statistic $S$, its $\ell_p$-sensitivity $\Delta_p(S) := \sup_{D,D'} \|S(D) - S(D')\|_p$ is the maximal change in the output by a single entry change. A common way to achieve DP is by adding noise scaled to the sensitivity. For example, $M(D)=S(D)+\xi$ is $\mu$-GDP when $\xi\sim N(0,(\Delta_2(S)/\mu)^2I_d)$.

DP has a few key properties: First, if $M$ is $f$-DP, $g \circ M$ remains $f$-DP for any data-independent $g$ (post-processing). 
Second, for $M_1,\ldots,M_k$ on dataset $D$, where each $M_i$ may depend on previous outputs and is $f_i$-DP conditional on them, the joint release $(M_1(D),\ldots,M_k(D))$ is $(f_1\otimes\cdots\otimes f_k)$-DP, where $f\otimes g:=T(P\times P',Q\times Q')$ for $f=T(P,Q)$ and $g=T(P',Q')$ (sequential composition). Finally, for a disjoint partition $D=\cup_{i=1}^k D_i$ and $f_i$-DP mechanisms $M_i$ on $D_i$, the joint release $(M_1(D_1),\ldots,M_k(D_k))$ is $\check f$-DP, where $\check f$ is the largest convex function bounded above by $\min_{i\in[k]} f_i$ (parallel composition).

Finally, other DP notions are also used in practice. For example, \texttt{Opacus} \citep{yousefpour2021opacus}, a library for DP stochastic gradient descent (DP-SGD), uses Rényi DP \citep{mironov2017renyi} for privacy accounting and reports $(\varepsilon,\delta)$-DP. Thus, we use $f$-DP and GDP for our theory in Sections~\ref{sec:proposed_framework}--\ref{sec:theory}, while Section~\ref{sec:numerical_study} uses \texttt{Opacus} for DP-SGD and reports $(\varepsilon,\delta)$-DP.

\subsection{Conformal Prediction}

Conformal prediction (CP) \citep{vovk2005algorithmic} is a distribution-free framework yielding a predictive set $\mathcal{C}(X_{n+1})$ that guarantees marginal coverage, $\mathbb{P}(Y_{n+1} \in \mathcal{C}(X_{n+1})) \ge 1 - \alpha$. Naturally, smaller prediction sets are more informative.

CP operates via a non-conformity score $s: \mathcal{X} \times \mathcal{Y} \times \Theta \to \mathbb{R}$ quantifying the discrepancy between a datum $(X,Y)$ and a model $\theta$ (e.g., $s(X, Y; \theta) = 1 - \hat{p}_Y(X; \theta)$ in classification). The set is constructed as $\mathcal{C}(X_{n+1}) = \{Y \in \mathcal{Y} : s(X_{n+1}, Y; \theta) \le \hat{q}\}$, where the threshold $\hat{q}$ ensures exact coverage provided the scores are exchangeable.

\begin{definition}[Exchangeability]
A sequence of random variables $S_1, \dots, S_k$ is exchangeable if their joint distribution is invariant to any permutation of indices.
\end{definition}

If a set of random variables $\{S_i\}_{i=1}^k$ is exchangeable, then the ranks are uniformly distributed. This property can be used to obtain a valid prediction set for any particular $S_i$:

\begin{proposition}[\citep{vovk2005algorithmic}]\label{prop: canonical coverage}
Let $\{S_i\}_{i=1}^{k}$ be exchangeable scores, and let $\hat{q}$ be the $\lceil(1-\alpha)k\rceil$-th order statistics. Then, for any $S_i$, we have $\mathbb{P}(S_i \le \hat{q}) \ge 1 - \alpha.$
\end{proposition}

Throughout, assume $(X_1,Y_1),\dots,(X_{n+1},Y_{n+1})$ are i.i.d.\ and the training mechanism is permutation invariant. Let $\theta_{n+1}$ be the oracle model trained on $D_{n+1}=D_n\cup\{(X_{n+1},Y_{n+1})\}$, which is assumed to be invariant to permutations on the data points. Then the scores $S_i^{(n+1)}=s\bigl((X_i,Y_i);\theta_{n+1}\bigr)$, $i=1,\dots,n+1$, are exchangeable. In practice, however, the test label $Y_{n+1}$ is unavailable, so the model $\theta_n$ is trained only on $D_n$. The first $n$ scores are then evaluated in-sample, while the test score is evaluated out-of-sample, violating exchangeability among $n+1$ scores. There are three broad ways to address this. 

{
\setlength{\leftmargini}{1.5em}
\begin{enumerate}
\item \textbf{Data splitting (Split-CP).} Split CP trains the model on a subset of $D_n$ and computes calibration scores on a disjoint subset. The calibration and test scores are both out-of-sample under the same model, which ensures exchangeability \citep{vovk2005algorithmic}. The cost is that fewer observations are available for both the training and calibration steps.

\item \textbf{Retraining (Full-CP and leave-one-out methods).} Full CP augments the data with each candidate response and refits the model, thereby treating the candidate test point symmetrically with the training observations \citep{vovk2005algorithmic}. Leave-one-out methods similarly fit multiple models so each residual is out-of-sample \citep{barber2021predictive}. These methods avoid a fixed calibration split but require repeated model fitting. 

\item \textbf{Quantifying distance to exchangeability.} Instead of directly obtaining exchangeable scores, this approach measures the distance to exchangeability, giving a bound on possible under-coverage. Existing work develops this idea under different forms of nonexchangeability, using the distance to exchangeability to control the resulting coverage loss \citep{tibshirani2019conformal,barber2023conformal,barber2026unifying}. Obtaining an explicit bound generally requires additional structure on the data-generating process.

\end{enumerate}
}

Existing private CP methods have largely followed the data-splitting route \citep{angelopoulos2022private,romanus2025differentially}. However, the cumulative accuracy cost of DP noise and sample splitting hurts the performance of this approach. Unfortunately, retraining does not seem to be a tractable option under DP. This is because repeated private training requires splitting the privacy budget (composition), and would require substantially more noise to ensure a meaningful privacy guarantee. DP-SCP follows the third route under DP. If the model is trained on $D_n$ by a DP mechanism, its distribution is necessarily similar to the distribution of the same mechanism applied to $D_{n+1}$, since they differ by a single entry. We leverage this property to quantify the coverage gap by DP-SCP.

\subsection{Intuition on full data versus split data CP}
\label{sec: full data motive}

The three approaches above differ in how they obtain validity. We now explain why using all the data for private training can improve prediction-set quality. While CP coverage is agnostic to model accuracy, set efficiency depends on it \citep{angelopoulos2021gentle}.

The following ``back-of-the-envelope'' calculation compares the cost of data splitting with the cost of privacy-budget splitting. For convex stochastic optimization under $(\varepsilon,\delta)$-DP, the optimal excess population risk scales as the sum of a statistical term proportional to $1/\sqrt{n}$ and a privacy term proportional to $1/(n\varepsilon)$ \citep{bassily2019private,gopi2022private}. For constants $c_1>0$ and $c_2>0$, consider the following two regimes:
{
\setlength{\leftmargini}{1.5em}
\begin{enumerate}
    \item \textbf{Split-Data Regime.} Halving the sample size to $n/2$ while retaining the full budget $\epsilon$ yields an error proportional to $\frac{\sqrt{2}c_1}{\sqrt{n}}+\frac{2c_2}{n\epsilon}$;
    \item \textbf{Full-Data Regime.} Using the full sample $n$ while splitting the budget, for example using $\epsilon/\sqrt{2}$ via GDP, yields an error proportional to $\frac{c_1}{\sqrt{n}}+\frac{\sqrt{2}c_2}{n\epsilon}$.
\end{enumerate}
}
\noindent
This comparison shows the tradeoff between using more data for training versus allocating more privacy budget to training. Both the statistical and the privacy terms have smaller constants in the full-data regime. Hence, when two procedures attain comparable coverage, full-data training is expected to yield more accurate predictors and thus smaller prediction sets. A complete comparison should also account for private quantile estimation and coverage corrections, but this calculation captures the advantage of using the full data under DP.

\section{DP-SCP Framework and Privacy Guarantee}
\label{sec:proposed_framework}

In this section, we present DP-SCP. We begin with its general structure, which consists of a training step and a calibration step. We then introduce our DP quantile estimation mechanism, which makes a key modification to an existing approach outside the DP CP literature to accommodate the conservative coverage guarantee. The two variants differ in this calibration step. Finally, we establish the privacy guarantee of DP-SCP.

\subsection{Overall Procedure}

Algorithm~\ref{alg: dp-scp} summarizes the general structure of DP-SCP. It trains a private model on the full data $D_n$, computes scores on the same data, and privately calibrates a threshold to construct the prediction set. DP-SCP-F and DP-SCP-A share the training step and differ only in calibration, using more and less conservative calibration targets, respectively.

\begin{algorithm}[!htpb]\spacingset{1.1}
\caption{DP-Stabilized Conformal Prediction (DP-SCP)}
\label{alg: dp-scp}
\begin{algorithmic}[1]
\State \textbf{Input:} Dataset $D_n=\{(X_i,Y_i)\}_{i=1}^n$, test point $X_{n+1}$, target miscoverage $\alpha$, DP training mechanism $M_{\mathrm{train}}$, variant $J\in\{F,A\}$, score range $[a,b]$, steps $N$, noise scale $\sigma$, failure level $\beta$, buffer $m_n$.
\State $\theta_n\gets M_{\mathrm{train}}(D_n)$ 
\State $\mathcal S\gets\{s(X_i,Y_i;\theta_n)\}_{i=1}^n$
\State $r\gets\lceil(1-\alpha)(n+1)\rceil$
\If{$J=F$}
    \State $\tau_n\gets\sigma\Phi^{-1}(1-\beta/N)-1$
    \State $R_J\gets r+m_n+\tau_n$
\Else
    \State $R_J\gets r$
\EndIf
\State Run Algorithm~\ref{alg:dp_binary_search} with $(\mathcal S,[a,b],R_J,N,\sigma)$ to obtain $\widehat q_J$
\State $\mathcal C_J(X_{n+1})\gets\{y\in\mathcal Y:s(X_{n+1},y;\theta_n)\le\widehat q_J\}$
\State \textbf{Output:} $\mathcal C_J(X_{n+1})$
\end{algorithmic}
\end{algorithm}

Algorithm~\ref{alg: dp-scp} allows for any training mechanism to be used in obtaining $\theta_n$. However, in Section~\ref{sec:theory}, we show that with some assumptions on the mechanism we can obtain a superior coverage guarantee. While many DP training mechanisms are available, DP-SGD \citep{abadi2016deep} is the standard scalable approach for training modern complex models. For this reason, it is our main motivating example.

We highlight that DP training alone does not make the entire DP-SCP satisfy DP. After training, DP-SCP accesses the raw data to compute the scores, so this step is not post-processing of the DP output. Quantile estimation must hence be privatized separately, and its privacy loss composes sequentially with DP training. Private Split CP is different as model training and calibration use disjoint data, so parallel composition avoids accumulation of the privacy losses. DP-SCP thus replaces data splitting with privacy-budget splitting.

The input $J\in\{F,A\}$ in Algorithm~\ref{alg: dp-scp} selects the variant according to the desired emphasis. The two variants differ in whether the finite-sample corrections are built into the calibration target. DP-SCP-F prioritizes finite-sample coverage through a more conservative calibration target, whereas DP-SCP-A uses a less conservative target to obtain more informative prediction sets. The finite-sample coverage bound of DP-SCP-A may therefore fall below the nominal level, while both variants recover the nominal level asymptotically under our conditions. The selected variant determines the comparison level $R_J$ passed to Algorithm~\ref{alg:dp_binary_search}. The next subsection specifies these calibration targets and details the DP quantile estimation.

\subsection{Right-Endpoint Calibration for DP Quantile Estimation}
\label{subsec:dp_search}

This subsection details the DP quantile estimation in DP-SCP. For $n+1$ exchangeable scores, Proposition~\ref{prop: canonical coverage} gives the target rank $r=\lceil(1-\alpha)(n+1)\rceil$. For the scores \(S_1,\ldots,S_n\), the corresponding threshold is the \(r\)th order statistic. Equivalently, letting \(C_n(t)=\sum_{i=1}^n\mathbbm 1\{S_i\le t\}\), we seek the smallest \(t\) satisfying \(C_n(t)\ge r\). Since $C_n(t)$ depends on the raw data, we cannot solve for $t$ exactly. We therefore use a noisy binary search, which replaces $C_n(t)$ with a noisy count $C_n(t)+Z$, where $Z\sim N(0,\sigma^2)$ \citep{huang2021instance,romanus2025differentially}. We treat \(\sigma\) as given and provide its accounting in Section~\ref{sec: privacy analysis}.

\begin{algorithm}[!htbp]
\caption{Private Right-Endpoint Search}
\label{alg:dp_binary_search}\spacingset{1.1}
\begin{algorithmic}[1]
\State \textbf{Input:} Scores \(\mathcal S\subseteq[a,b]\), comparison level \(R\), steps \(N\), noise scale \(\sigma\)
\State \(\texttt{left}\gets a\), \(\texttt{right}\gets b\)
\For{\(k=1\) \textbf{to} \(N\)}
    \State \(\texttt{mid}\gets(\texttt{left}+\texttt{right})/2\)
    \State \(\widetilde C_k\gets C_n(\texttt{mid})+Z_k\), where \(Z_k\sim N(0,\sigma^2)\) independently
    \If{\(\widetilde C_k\ge R\)}
        \State \(\texttt{right}\gets\texttt{mid}\)
    \Else
        \State \(\texttt{left}\gets\texttt{mid}\)
    \EndIf
\EndFor
\State \textbf{Output:} \(\widehat q\gets\texttt{right}\)
\end{algorithmic}
\end{algorithm}

Algorithm~\ref{alg:dp_binary_search} implements a noisy binary search. At each step, it compares the noisy count with the comparison level $R$ and updates the search interval accordingly. After $N$ steps, it outputs the right endpoint of the final interval. The comparison level \(R\) depends on the selected DP-SCP variant, as specified below.

The right-endpoint output is an important modification, which allows for the lower-side control needed for our coverage analysis. If a midpoint $t$ lies below the intended quantile, sufficiently large positive noise can make $C_n(t)+Z$ reach $R$ and incorrectly lower the right endpoint to $t$. When such false right updates do not occur, the right endpoint remains above the target throughout the search, whereas the midpoint need not. Hence we return the final right endpoint not the midpoint.

When analyzing quantile estimation by noisy binary search, we identify two sources of error. First, noise in the count queries can move the private search below its intended empirical rank. Second, the shift from \(\theta_{n+1}\) to \(\theta_n\) can change which scores fall below the target. We account for these errors through two corrective terms, \(\tau_n\) and \(m_n\), respectively.

For DP-SCP-F, we set \(R_F=r+m_n+\tau_n\). Here, \(\tau_n\) corrects for positive Gaussian noise, while \(m_n\) is the stability buffer for the model shift. Section~\ref{subsec:dpsgd_instantiation} gives the concrete form of \(m_n\) for DP-SGD. Thus, DP-SCP-F includes both corrections in the calibration target.

For DP-SCP-A, we instead set \(R_A=r\), excluding \(m_n\) and \(\tau_n\). Under the DP-SGD tuning in Corollary~\ref{cor:dpsgd_coverage}, \(m_n=O(\sqrt{n})\), while \(\tau_n=O(\log n)\). Since \(r=\Theta(n)\), both terms are \(o(n)\) on the rank scale. DP-SCP-A therefore uses a less conservative quantile. Its finite-sample coverage may fall below the nominal level, but it can yield more informative prediction sets.

\begin{proposition}\label{prop:private_search}
Let \(\widehat q\) be the output of Algorithm~\ref{alg:dp_binary_search}. For any \(\beta\in(0,1)\), we have

\textbf{(i) Target-rank bound.} For \(1 \leq k \leq n\), let \( \tau=\sigma\Phi^{-1}(1-\beta/N)-1.\)
If \(R=k+\tau\), then
\[
\mathbb P\!\left(\widehat q\ge S_{[k]}\right)\ge1-\beta.
\]

\textbf{(ii) Bound with rank loss.}
Let \(z=\sigma\Phi^{-1}(1-\beta/N).\) If \(1\le \lceil R-z\rceil\le n\), then
\[
\mathbb P\!\left(\widehat q\ge S_{[\lceil R-z\rceil]}\right)\ge1-\beta.
\]
\end{proposition}

Proposition~\ref{prop:private_search} quantifies the effect of the noisy search on the empirical rank of $\widehat q$. Part~\textnormal{(i)} is used for DP-SCP-F. Adding the noise correction $\tau$ to $R$ ensures that $\widehat q\ge S_{[k]}$ with probability at least $1-\beta$. Part~\textnormal{(ii)} is used for DP-SCP-A, where $R=r$. Without this correction, $\widehat q$ may fall below $S_{[r]}$, but it remains above $S_{[\lceil r-z\rceil]}$ with probability at least $1-\beta$. This rank loss enters the finite-sample coverage bound for DP-SCP-A.

The midpoint at each step depends on the previous noisy counts, so the search is adaptive. The proof controls the Gaussian noises jointly and uses induction to show that the right endpoint stays above the relevant empirical rank throughout the search. This is also why Algorithm~\ref{alg:dp_binary_search} outputs the right endpoint rather than the final midpoint.

\subsection{Privacy Analysis}\label{sec: privacy analysis}

We analyze the privacy guarantee of DP-SCP in terms of GDP. DP-SCP consists of two DP mechanisms. For model training, we use an off-the-shelf DP training algorithm and assume that it satisfies \(\mu_{\text{train}}\)-GDP. We therefore focus on Algorithm~\ref{alg:dp_binary_search}. Each of its \(N\) iterations issues a noisy count query, so the overall privacy accounting follows from sequential composition. A single record can change each underlying count by at most one, and hence each query has \(\ell_2\)-sensitivity \(\Delta_2=1\). This yields the following guarantee.

\begin{lemma}[Privacy of Buffered Binary Search]\label{lemma:search_privacy} For a target budget $\mu_{\text{calib}} > 0$, setting the noise scale $\sigma = \sqrt{N}/\mu_{\text{calib}}$ ensures Algorithm~\ref{alg:dp_binary_search} satisfies $\mu_{\text{calib}}$-GDP.
\end{lemma}

The overall privacy guarantee follows immediately from the composition theorem:

\begin{theorem}[Overall Privacy Guarantee]\label{thm:overall_privacy}Suppose $M_{\text{train}}$ is $\mu_{\text{train}}$-GDP and $M_{Q}$ is $\mu_{\text{calib}}$-GDP. Then DP-SCP satisfies $\mu_{\text{total}}$-GDP, where $\mu_{\text{total}} = \sqrt{\mu_{\text{train}}^2 + \mu_{\text{calib}}^2}.$
\end{theorem}

Note that if the mechanisms satisfy a different DP notion, the same type of composition result follows by composing the training and calibration guarantees in that framework.

\section{Coverage Analysis}\label{sec:theory}

This section develops the coverage of DP-SCP. We begin with, to our knowledge, the first characterization of the limit of conformal validity from DP alone. We then go beyond the generic DP guarantee and analyze the DP training mechanism, yielding a sharper coverage guarantee. Finally, we study DP-SGD as an explicit example.

\subsection{A Universal Coverage Guarantee from DP}

Since DP regulates output changes under adding or deleting a single datapoint, it is tempting to use DP as a generic stability tool for full-data CP. The next result characterizes what such a generic DP argument can guarantee, and its limitation without additional structure.

\begin{theorem}\label{thm:dp_gap_two_sides}
Fix a tradeoff function $f$, $\alpha\in(0,1)$, and $\gamma>0$.
Let $D_{n+1}=(X_i,Y_i)_{i=1}^{n+1}$ be a dataset of  i.i.d. data points from distribution $P$,  and denote $D_n=(X_i,Y_i)_{i=1}^n$. 

{
\setlength{\leftmargini}{1.5em} \begin{enumerate}
\item\textbf{(Universal Lower bound).}
For any $f$-DP mechanism $M$, any distribution $P$, and  any prediction-set map $C(\cdot,\cdot)$, if $\mathbb{P}\!\left(Y_{n+1}\in C(M(D_{n+1}),X_{n+1})\right)\ge 1-\alpha,$
then
\[
\mathbb{P}\!\left(Y_{n+1}\in C(M(D_n),X_{n+1})\right)\ge f(\alpha).
\]

\item\textbf{(An Upper Bound).}
 For every $n\in\mathbb{N}$ there exist a distribution $P$ on $(X,Y)$, an $f$-DP mechanism $M$, and a prediction-set map $C(\cdot,\cdot)$ such that
\[
\left\{
\begin{aligned}
&\mathbb{P}\!\left(Y_{n+1}\in C(M(D_{n+1}),X_{n+1})\right)\ge 1-\alpha,\\
&\mathbb{P}\!\left(Y_{n+1}\in C(M(D_n),X_{n+1})\right)\le f(\alpha)+\gamma.
\end{aligned}
\right.
\]
\end{enumerate}
}
\end{theorem}

Theorem~\ref{thm:dp_gap_two_sides} identifies the exact coverage guarantee that follows from DP alone. Part~\textnormal{(i)} transfers coverage \(1-\alpha\) under \(D_{n+1}\) to at least \(f(\alpha)\) under \(D_n\), while part~\textnormal{(ii)} shows that this lower bound is sharp. Recall that \(f(\alpha)=1-\alpha\) corresponds to perfect privacy, where the output distributions are identical under adjacent datasets. For any other tradeoff function, \(f(\alpha)<1-\alpha\) for every \(\alpha\in(0,1)\). Indeed, \(g(\alpha)=1-\alpha-f(\alpha)\) is nonnegative and concave. If \(g\) were zero at any interior point, concavity would force \(g\equiv0\), and hence \(f(\alpha)=1-\alpha\). Thus, DP alone leaves a strict gap from nominal coverage for any privacy guarantee other than perfect privacy, while the gap decreases as privacy becomes stronger.

\begin{remark}[Proof sketch and insight]
For part~\textnormal{(i)}, we view the miscoverage event under \(D_{n+1}\) as the rejection event of a test between the outputs under \(D_{n+1}\) and \(D_n\). Its Type-I error is at most \(\alpha\), so \(f\)-DP implies coverage at least \(f(\alpha)\) under \(D_n\). For part~\textnormal{(ii)}, we construct an \(f\)-DP hard instance using a noisy histogram. Under \(D_{n+1}\), the test pair contributes one count and attains coverage \(1-\alpha\), whereas under \(D_n\) it is rarely observed, bringing coverage arbitrarily close to \(f(\alpha)\). Thus, the generic DP lower bound is sharp.
\end{remark}

Theorem~\ref{thm:dp_gap_two_sides} requires no additional assumptions on the data distribution or the DP mechanism beyond the DP guarantee. Combined with Proposition~\ref{prop:private_search}(i), we have the following coverage guarantee for DP-SCP-F:

\begin{corollary}[Generic DP floor for DP-SCP-F]\label{cor:blackbox_dpscp}
Let \(\alpha,\beta\in(0,1)\), where \(\beta\) is the failure probability in Proposition~\ref{prop:private_search}(i), and suppose DP-SCP-F, including private training and calibration, satisfies \(f\)-DP. Under \(D_{n+1}\), its coverage is at least \((1-\beta)(1-\alpha)=1-\alpha_0\), where \(\alpha_0=\alpha+\beta-\alpha\beta\). Under \(D_n\), its coverage is therefore at least \(f(\alpha_0)\).
\end{corollary}

Here, \(f\) is the overall tradeoff function after composing private training and calibration. For \((\varepsilon,\delta)\)-DP, if \(\alpha_0\le(1-\delta)/(1+e^\varepsilon)\), then \(f_{\varepsilon,\delta}(\alpha_0)=1-\delta-e^\varepsilon\alpha_0\). Relative to the nominal level \(1-\alpha\), the resulting coverage gap is
\[
(1-\alpha)-f_{\varepsilon,\delta}(\alpha_0)
=
\delta+e^\varepsilon\alpha_0-\alpha
=
\delta+(e^\varepsilon-1)\alpha+e^\varepsilon\beta(1-\alpha)>0.
\]
This under-coverage motivates a refined analysis with additional properties of the training mechanism.

\subsection{Refined Coverage Guarantee with Further Assumptions}\label{subsec:coverage}

We now establish a sharper finite-sample guarantee under additional structure. We begin by introducing the notation and assumptions.

For a data point $Z_i = (X_i, Y_i)$ and a model $\theta$, we write the score as $s(Z_i; \theta)$. To distinguish between the training and ideal scenarios, we define:
{
\setlength{\leftmargini}{3em} \begin{enumerate}
\item[(a)] \textbf{Actual Scores (Non-exchangeable):} $S_i^{(n)} := s(Z_i; \theta_n)$ for $i=1, \dots, n+1$.
\item[(b)] \textbf{Ideal Scores (Exchangeable):} $S_i^{(n+1)} := s(Z_i; \theta_{n+1})$ for $i=1, \dots, n+1$.
\end{enumerate}
}
Here, $\theta_n$ is trained on $D_n$, while $\theta_{n+1}$ is trained on $D_{n+1}=D_n\cup\{Z_{n+1}\}$. For $1\leq k \leq n+1\}$, let $q_k$ be the $k$th order statistic of all $n+1$ ideal scores. For a miscoverage level $\alpha\in(0,1)$, set $r=\lceil(1-\alpha)(n+1)\rceil$, so that $q_r$ is the target order statistic under exchangeability. We consider the following assumptions.

\begin{assumption}[Model stability]\label{ass:stab}
There exist sequences $u_n\ge0$ and $\delta_n\in(0,1)$ such that $\mathbb{P}(\|\theta_{n+1}-\theta_n\|_2>u_n)\le \delta_n$.
\end{assumption}

\begin{assumption}[Score Lipschitz]\label{ass:lipschitz}
For all $Z$ and $\theta,\theta'\in\Theta$, $|s(Z;\theta)-s(Z;\theta')|\le L\,\|\theta-\theta'\|_2.$
\end{assumption}

\begin{assumption}[No ties]\label{ass:noties}
For all $i\neq j$, $\mathbb{P}(S_i^{(n+1)}=S_j^{(n+1)})=0$.
\end{assumption}

\begin{assumption}[Lower-side anti-concentration]\label{ass:anti}
There exist constants $\overline f<\infty$ and $\rho>0$ such that for any $0<\delta \leq \rho$, any $i\in\{1,\dots,n+1\}$, and any $k\in\{1,\dots,n+1\}$,
\[
\mathbb{P}\!\left(q_k-\delta\le S_i^{(n+1)}<q_k\right)
\le \overline f\,\delta.
\]
\end{assumption}
Assumption~\ref{ass:stab} controls how much the fitted model parameters change when one data point is added, while Assumption~\ref{ass:lipschitz} carries this change to the score scale. Assumptions~\ref{ass:noties} and~\ref{ass:anti} control the behavior of the ideal scores around their order statistics. Together, the first two assumptions control the discrepancy between the actual and ideal scores, while the latter two limit how much this discrepancy can affect the relevant order statistic.

\begin{remark}[Further discussion on Assumption~\ref{ass:stab}]\label{remark: dp mech stab}
Assumption~\ref{ass:stab} holds for several standard DP estimators. When the objective is strongly convex, output and objective perturbation change by \(O(n^{-1})\) when one data point is added. Analogous bounds hold for procedures that privatize sufficient statistics and then use post-processing. Supplementary Section~\ref{sec:assumption_validation} gives the calculations. We analyze DP-SGD separately in Section~\ref{subsec:dpsgd_instantiation}.
\end{remark}

The next result quantifies how the difference between \(\theta_n\) and \(\theta_{n+1}\) affects coverage. For DP-SCP-F, the calibration buffer \(m_n\) increases the coverage guarantee, while for DP-SCP-A the model shift and private-search error appear directly in the coverage bound.

\begin{theorem}\label{thm:coverage}
Suppose Assumptions~\ref{ass:stab}, \ref{ass:lipschitz}, \ref{ass:noties}, and~\ref{ass:anti} hold, and let \(\beta_n\in(0,1)\).\\
\textnormal{(i)} \textbf{DP-SCP-F.}
Let \(m_n\ge0\). If \(r+m_n\le n\) and \(2Lu_n\le\rho\), then
\[
\mathbb P\!\left(S_{n+1}^{(n)}\le\widehat q_F\right)
\ge
(1-\beta_n)
\left\{
1-\alpha
+\frac{m_n}{n+1}
-2\overline fLu_n
-\delta_n
\right\}.
\]

\textnormal{(ii)} \textbf{DP-SCP-A.}
Let \(z_n=\sigma\Phi^{-1}(1-\beta_n/N)\) and \(k_n=\lceil r-z_n\rceil\). If \(1\le k_n\le n\) and \(2Lu_n\le\rho\), then
\[
\mathbb P\!\left(S_{n+1}^{(n)}\le\widehat q_A\right)
\ge
(1-\beta_n)
\left\{
1-\alpha
-\frac{z_n}{n+1}
-2\overline fLu_n
-\delta_n
\right\}.
\]
\end{theorem}

Theorem~\ref{thm:coverage} gives finite-sample coverage guarantees for both DP-SCP variants by combining stability of the training mechanism with private calibration. Unlike the generic DP guarantee in the previous subsection, these bounds depend on how much the model changes between \(D_n\) and \(D_{n+1}\). Assumption~\ref{ass:stab} controls this change through \(u_n\) and \(\delta_n\), Assumption~\ref{ass:lipschitz} transfers it to the score scale, and Assumption~\ref{ass:anti} controls its effect near the relevant order statistics.

For DP-SCP-F, increasing the calibration rank by \(m_n\) contributes the positive term \(m_n/(n+1)\), while the model shift contributes \(2\overline fLu_n+\delta_n\). The calibration buffer can therefore offset the coverage loss from the model shift and private search, and can yield a conservative finite-sample guarantee when it is sufficiently large.

DP-SCP-A excludes this buffer. Instead, the private search reduces the retained rank by \(z_n\), contributing \(z_n/(n+1)\) to the coverage loss, while the model shift contributes the same \(2\overline fLu_n+\delta_n\). This less conservative calibration can yield more informative prediction sets. In the DP-SGD setting studied next, these terms vanish under our tuning, so both variants recover the nominal \(1-\alpha\) coverage asymptotically.

\subsection{Instantiation to DP-SGD}\label{subsec:dpsgd_instantiation}

We now analyze the stability of DP-SGD. DP-SGD forms a random minibatch at each iteration, clips individual gradients, and adds Gaussian noise to the aggregated gradient, with its update rule given in Algorithm~\ref{alg:dp-sgd} in Supplementary Section~\ref{sec:supp_dp_intro}. We consider DP-SGD trajectories on \(D_n\) and \(D_{n+1}\) from the same initial point and with the same underlying randomness, so that the common datapoints are sampled identically and the same Gaussian noise is added in both runs. The only difference is the additional datapoint in \(D_{n+1}\).

\subsubsection{Regime I: Universal Stability}

We first consider the stability that comes directly from the minibatch sampling step of DP-SGD. We highlight that this stability arises solely from this sampling step in DP-SGD and is independent of any geometric property of the loss function. DP-SGD commonly uses Poisson subsampling for privacy amplification \citep{bu2020deep}, where each datapoint is included independently with probability \(q\) at each iteration. Under the same underlying randomness, therefore, the two trajectories can differ only after the additional datapoint is selected. If it is never selected during training, the trajectories remain identical.

\begin{lemma}[Universal stability]\label{lem:dpsgd_stability}
Let \(\Theta\subset\mathbb R^d\) be nonempty, closed, convex, and bounded. For any datasets \(D_n\) and \(D_{n+1}=D_n\cup\{(X_{n+1},Y_{n+1})\}\), run Algorithm~\ref{alg:dp-sgd} from the same initial point and with the same underlying randomness for \(T\ge1\) iterations with Poisson subsampling rate \(q\in(0,1)\). Then
\[
\left\{
\begin{aligned}
&\mathbb P\!\left(
\max_{0\le t\le T}
\|\theta_n^{(t)}-\theta_{n+1}^{(t)}\|_2>0
\right)
\le
1-(1-q)^T,\\
&\mathbb E\!\left[
\max_{0\le t\le T}
\|\theta_n^{(t)}-\theta_{n+1}^{(t)}\|_2
\right]
\le
\operatorname{diam}(\Theta)\{1-(1-q)^T\}.
\end{aligned}
\right.
\]
\end{lemma}

Lemma~\ref{lem:dpsgd_stability} follows directly from Poisson subsampling. The additional datapoint is never selected over the \(T\) iterations with probability \((1-q)^T\), and the coupled trajectories then remain identical throughout training. Otherwise, they may separate, but boundedness of \(\Theta\) limits their discrepancy to \(\operatorname{diam}(\Theta)=\sup_{\theta,\theta'\in\Theta}\|\theta-\theta'\|_2\).

Lemma~\ref{lem:dpsgd_stability} verifies Assumption~\ref{ass:stab} with \(u_n=0\) and \(\delta_n=1-(1-q)^T\), so \(m_n=0\). Theorem~\ref{thm:coverage}(i) then gives \(\mathbb P\!\left(S_{n+1}^{(n)}\le \widehat q_F\right)\ge(1-\beta_n)\left(1-\alpha-\{1-(1-q)^T\}\right).\)
Lemma~\ref{lem:dpsgd_stability} can become loose as the additional datapoint becomes more likely to be selected over \(T\) iterations, since \(1-(1-q)^T\) approaches one as the cumulative sampling opportunity grows. After its first selection, the probability bound counts any separation, while the expectation bound uses only \(\operatorname{diam}(\Theta)\). The next regime uses smoothness to control the subsequent discrepancy.

\subsubsection{Regime II: Stability under Smoothness}

We next use smoothness to control a discrepancy after the trajectories bifurcate.

\begin{theorem}[Stability under smoothness]\label{thm:dpsgd_stability_smooth}
Suppose \(\Theta\subset\mathbb R^d\) is nonempty, closed, and convex, and \(\nabla\ell(\cdot;x,y)\) is \(L_\ell\)-Lipschitz on \(\Theta\) for every \((x,y)\). If \(\eta_t\ge0\) for \(t=0,\ldots,T-1\), then under the synchronized coupling,
\[
\mathbb E\!\left[
\max_{0\le t\le T}
\|\theta_n^{(t)}-\theta_{n+1}^{(t)}\|_2
\right]
\le
\frac{C}{nL_\ell}
\left[
\exp\left\{
\frac{nqL_\ell}{\bar b}
\sum_{t=0}^{T-1}\eta_t
\right\}-1
\right].
\]
\end{theorem}

Theorem~\ref{thm:dpsgd_stability_smooth} tracks the discrepancy after the two trajectories separate. When the additional datapoint is selected, it can create a gap between the two updates. Subsequent updates can propagate this gap, while smoothness bounds the increased discrepancy at each iteration. The product term accumulates these iteration-wise effects over the training trajectory.

If we additionally have that \(\bar b\asymp nq\), \(C\) and \(L_\ell\) are bounded, and \(\sum_{t=0}^{T-1}\eta_t=O(1)\), then the exponential factor is bounded and \(\mathbb E[\max_{0\le t\le T}\|\theta_n^{(t)}-\theta_{n+1}^{(t)}\|_2]=O(n^{-1})\). For example, this condition holds when \(\eta_t=O(n^{-1})\) and \(T=O(n)\). On the other hand, the Gaussian noise does not appear in the bound because the two trajectories share the same Gaussian noise at each iteration, which cancels when their updates are differenced.


\begin{corollary}[Coverage under DP-SGD]\label{cor:dpsgd_coverage}

Suppose the model is trained by DP-SGD and the stability bound in Theorem~\ref{thm:dpsgd_stability_smooth} satisfies \(E_n=O(n^{-1})\). Suppose Assumptions~\ref{ass:lipschitz}--\ref{ass:anti} hold.

\textnormal{(i)} \textbf{DP-SCP-F.}
Set \(u_n=\delta_n=E_n^{1/2}\), \(\beta_n=n^{-1}\), and \(m_n=\left\lceil (n+1)\{(2\overline fL+1)E_n^{1/2}+
\frac{1-\alpha}{n-1}
\right\rceil\). Then, for all sufficiently large \(n\),
\[
\mathbb P\!\left(S_{n+1}^{(n)}\le\widehat q_F\right)
\ge
1-\alpha.
\]

\textnormal{(ii)} \textbf{DP-SCP-A.}
Set \(u_n=\delta_n=E_n^{1/2}\), \(\beta_n=n^{-1}\), and \(N=\lceil\log_2 n\rceil\). Then, for all sufficiently large \(n\),
\[
\mathbb P\!\left(S_{n+1}^{(n)}\le\widehat q_A\right)
\ge
1-\alpha-O(n^{-1/2}).
\]

\end{corollary}

Corollary~\ref{cor:dpsgd_coverage} translates the \(O(n^{-1})\) DP-SGD stability into coverage guarantees for both variants. Both variants start from the same stability bound in Theorem~\ref{thm:dpsgd_stability_smooth}, but differ in how this stability is incorporated into calibration.

For DP-SCP-F, the expected stability bound is converted to a high-probability model discrepancy and incorporated into the calibration threshold through \(m_n\). Under the tuning in Corollary~\ref{cor:dpsgd_coverage}, \(E_n=O(n^{-1})\) implies \(m_n=O(\sqrt n)\), and hence \(m_n/n=O(n^{-1/2})\to0\). Thus, a vanishing relative calibration buffer is sufficient to guarantee coverage at least the nominal level \(1-\alpha\) for all sufficiently large \(n\). The buffer can also be increased within the same \(O(\sqrt n)\) order to obtain a strictly conservative guarantee. For example, replacing \(m_n\) by \(\left\lceil (n+1)\{(2\overline fL+1)E_n^{1/2}+(1-\alpha)/(n-1)+n^{-1/2}\}\right\rceil\) still gives \(m_n=O(\sqrt n)\), while the coverage lower bound becomes \(1-\alpha+(1-1/n)n^{-1/2}>1-\alpha\).

DP-SCP-A avoids this additional buffer, so the private-search and model-shift terms remain in its coverage bound, yielding \(1-\alpha-O(n^{-1/2})\) without a finite-sample over-coverage guarantee. Its less conservative calibration can yield more informative prediction sets, while the guarantees of both variants approach the nominal level \(1-\alpha\) asymptotically.

\section{Numerical Studies} \label{sec:numerical_study}

We evaluate the empirical performance of DP-SCP on classification and regression tasks. Fix $\alpha=0.1$.  Denote the test data by $\{(X_j, Y_j)\}_{j=1}^{n_{\text{test}}}$ and the corresponding prediction sets by $\{\mathcal{C}(X_j)\}_{j=1}^{n_{\text{test}}}$. Following \cite{romanus2025differentially}, we consider:
{
\setlength{\leftmargini}{1.5em} \begin{enumerate}
    \item $\text{Coverage} = \frac{1}{n_{\text{test}}}\sum_{j=1}^{n_{\text{test}}}\mathbbm{1}\{Y_j \in \mathcal{C}(X_j)\},
    $ the proportion of sets containing true response,
    \item$\text{Efficiency} = \frac{1}{n_{\text{test}}}\sum_{j=1}^{n_{\text{test}}} \text{Size}(\mathcal{C}(X_j)),$
    where $\text{Size}(C) = |C|$ for a discrete $C$ and $U - L$ for an interval $C=[L,U]$, the average size of the prediction sets, \item$\text{Informativeness} = \frac{1}{n_{\text{test}}}\sum_{j=1}^{n_{\text{test}}}\mathbbm{1}\{|{\mathcal{C}(X_j)}|=1\}$, the proportion of singleton prediction sets.
\end{enumerate}
}

We report informativeness only for classification tasks, where prediction sets are discrete. Alongside DP-SCP-F and DP-SCP-A, we implement the following:

{
\setlength{\leftmargini}{1.5em} \begin{enumerate}
\item\textbf{DP-Split (Private Baseline):} Existing private Split CP framework using disjoint training and calibration sets \citep{romanus2025differentially}, which does not guarantee the nominal \(1-\alpha\) coverage level in finite samples.
\item\textbf{Naive Full (Non-private Naive Full-CP):} Non-private full-data reuse using the exact conformal quantile, ignoring the violation of exchangeability. It serves to quantify the level of under-coverage induced by data reuse when stability corrections are absent.
\item\textbf{Split CP:} Standard non-private split conformal prediction with finite-sample coverage.
\end{enumerate}
}
Retraining based-CP is not considered due to its computational expense. Results are averaged over 30 random train/test splits with $\delta=10^{-5}$ and $\varepsilon \in \{0.5, 1.0, 2.0\}$ (accounting details in Section~\ref{supp:privacy_accounting}), and we report the standard deviation across these 30 trials. For DP-SCP, we sequentially allocate the budget, running DP-SGD at $(p\varepsilon,\delta)$ and calibrating Algorithm~\ref{alg:dp_binary_search} to meet the overall $(\varepsilon, \delta)$. We set $p=0.5$ for main results, deferring sensitivity over $p$ to Section~\ref{sec:supp_realworld}. Because DP-Split calibrates on a disjoint hold-out set, parallel composition permits utilizing the full budget for both stages.

\subsection{Biomedical Image Analysis: Classification}
\label{subsec:real_data_classification}

We evaluate DP-SCP on BloodMNIST \citep{yang2023medmnist}, an eight-class blood-cell image classification dataset. For each repetition, we recombine the original data splits and draw a fresh fitting pool of \(13{,}671\) observations and a test set of \(3{,}421\) observations, matching the original train-plus-validation and test sample sizes. We use a frozen ImageNet-pretrained ResNet-18 \citep{he2016deep} as a feature extractor and train a two-hidden-layer MLP with widths \(256\) and \(128\). The conformal score is \(s(x,y)=1-\hat p_y(x)\). 

\begin{table}[!ht]
\centering
\caption{BloodMNIST classification results. Values are mean (SD) over 30 repetitions. The 95\% confidence interval is for mean coverage. Boldface marks the best utility among DP methods whose 95\% confidence interval does not entirely lie below the level \(0.90\).}
\label{tab:bloodmnist_main}
\small\spacingset{1.1}
\setlength{\tabcolsep}{4pt}
\renewcommand{\arraystretch}{1.08}
\resizebox{\linewidth}{!}{%
\begin{tabular}{@{}llcccc@{}}
\toprule
\textbf{Privacy} & \textbf{Method} & \textbf{Coverage} & \textbf{95\% CI} & \textbf{Avg. set size} \(\downarrow\) & \textbf{Singleton rate} \(\uparrow\) \\
\midrule

\multirow{3}{*}{\(\varepsilon=0.5\)}
 & DP-SCP-F & 0.9633 (0.0428) & [0.9474, 0.9793] & 5.6967 (2.6800) & 0.1356 (0.1583) \\
 & DP-SCP-A & 0.9335 (0.0448) & [0.9167, 0.9502] & \textbf{4.2045 (2.5295)} & \textbf{0.2305 (0.1544)} \\
 & DP-Split & 0.8959 (0.0107) & [0.8919, 0.8999] & 2.6233 (0.1573) & 0.3153 (0.0253) \\
\midrule

\multirow{3}{*}{\(\varepsilon=1.0\)}
 & DP-SCP-F & 0.9039 (0.0066) & [0.9014, 0.9063] & 1.8166 (0.0457) & 0.5366 (0.0173) \\
 & DP-SCP-A & 0.8980 (0.0065) & [0.8956, 0.9005] & \textbf{1.7650 (0.0452)} & \textbf{0.5518 (0.0167)} \\
 & DP-Split & 0.9005 (0.0079) & [0.8976, 0.9035] & 1.8092 (0.0607) & 0.5371 (0.0204) \\
\midrule

\multirow{3}{*}{\(\varepsilon=2.0\)}
 & DP-SCP-F & 0.9035 (0.0070) & [0.9009, 0.9061] & 1.7275 (0.0320) & 0.5711 (0.0124) \\
 & DP-SCP-A & 0.9001 (0.0071) & [0.8975, 0.9028] & \textbf{1.6985 (0.0318)} & \textbf{0.5799 (0.0131)} \\
 & DP-Split & 0.9013 (0.0074) & [0.8985, 0.9040] & 1.7150 (0.0457) & 0.5761 (0.0150) \\
\midrule
\midrule

\multicolumn{2}{l}{\textit{Non-private Baselines}} & & & & \\
 & Naive Full & 0.8331 (0.0098) & [0.8294, 0.8368] & 0.9274 (0.0150) & 0.9274 (0.0150) \\
 & Split CP & 0.8995 (0.0056) & [0.8974, 0.9016] & 1.1491 (0.0157) & 0.8617 (0.0130) \\
\bottomrule
\end{tabular}%
}
\end{table}

Table~\ref{tab:bloodmnist_main} reports coverage, average set size, and singleton rate across the three privacy levels. We bold the best utility among private methods whose 95\% confidence interval for mean coverage does not lie entirely below \(0.90\). Under this criterion, DP-SCP-A gives the best utility at all three privacy levels, although its absolute set size increases sharply under the strongest privacy. Before examining the private methods more closely, the nonprivate baselines show why direct reuse requires correction. Naive Full has coverage \(0.8331\), compared with \(0.8995\) for Split CP, while its much smaller threshold, \(0.4109\) versus \(0.8340\), is consistent with the large train-test accuracy gap of the Full classifier, \(0.9326\) versus \(0.8660\), and the resulting undercoverage. Its smaller sets and higher singleton rate therefore reflect the undercoverage rather than a meaningful utility gain.

Turning to the private methods, the comparison is most direct at \(\varepsilon=1\) and \(2\), where all three methods have observed coverage close to \(0.90\). At \(\varepsilon=1\), DP-SCP-A reduces average set size relative to DP-Split by \(2.40\%\) in the paired comparison and also gives a higher singleton rate. The same ordering remains at \(\varepsilon=2\), although the paired set-size reduction decreases to \(0.93\%\), showing a more modest utility gain as the two methods become closer.

However, \(\varepsilon=0.5\) behaves differently. DP-SCP-A and DP-SCP-F produce average set sizes \(4.2045\) and \(5.6967\), respectively, compared with \(2.6233\) for DP-Split. Under the main allocation \(p=0.5\), DP-SCP trains on all observations with half of the overall privacy budget targeted to training, and the resulting upper calibration scores lie very close to the score bound \(1\). Since Algorithm~\ref{alg:dp_binary_search} returns the right endpoint after finitely many steps, it can return \(1\) in this setting, in which case every class is included in the prediction set. Algorithm~\ref{alg:dp_binary_search} returns the upper endpoint \(1\) in 9 of the 30 DP-SCP-A repetitions and 17 of the 30 DP-SCP-F repetitions, so some repetitions include all eight classes while others do not. This explains the substantially larger across-repetition coverage SDs at \(\varepsilon=0.5\). DP-SCP-F uses the larger comparison rank \(R_F=r+10+\tau_n\), which further pushes the threshold to the right, while the search does not return the upper score bound in the \(\varepsilon=1\) or \(2\) experiments. 

The comparison between DP-SCP-A and DP-SCP-F further presents the calibration step because the two variants use the same trained model. At \(\varepsilon=1\) and \(2\), DP-SCP-F raises coverage relative to DP-SCP-A but also produces larger sets and lower singleton rates, with average set size slightly exceeding DP-Split at both privacy levels. Thus, the additional calibration correction trades prediction-set informativeness for more conservative coverage. Moreover, the gap between the two DP-SCP methods  shrinks as \(\varepsilon\) increases, while DP-SCP-F is most conservative at \(\varepsilon=0.5\).

\begin{table}[!htbp]
\centering
\caption{Predictive quality of the DP training on BloodMNIST. DP-SCP-F and DP-SCP-A share the same trained model, denoted DP-SCP in the table. Values are test accuracy mean (SD) over 30 repetitions. The difference is computed within repetition as DP-SCP minus DP-Split and then summarized across repetitions. DP-SCP uses the full fitting pool with \(p=0.5\), whereas DP-Split uses half of the fitting pool with the full privacy budget.}
\label{tab:bloodmnist_training}
\small\spacingset{1.1}
\setlength{\tabcolsep}{6pt}
\renewcommand{\arraystretch}{1.08}
\begin{tabular}{lccc}
\toprule
\textbf{Privacy} & \textbf{DP-SCP} & \textbf{DP-Split} & \textbf{Difference} \\
\midrule
\(\varepsilon=0.5\) & 0.6391 (0.0115) & 0.6251 (0.0131) & 0.0140 (0.0053) \\
\(\varepsilon=1.0\) & 0.7344 (0.0090) & 0.7302 (0.0102) & 0.0042 (0.0045) \\
\(\varepsilon=2.0\) & 0.7438 (0.0099) & 0.7438 (0.0100) & 0.0000 (0.0044) \\
\bottomrule
\end{tabular}
\end{table}

Table~\ref{tab:bloodmnist_training} separates the effect of training from the calibration. At \(\varepsilon=0.5\), DP-SCP training yields higher test accuracy than DP-Split training despite the smaller training privacy budget, yet DP-SCP-A produces substantially larger prediction sets because the right-endpoint calibration dominates the final set size. Meanwhile, at \(\varepsilon=1\), DP-SCP training continues to yield higher test accuracy, and DP-SCP-A simultaneously produces smaller sets and more singleton predictions than DP-Split. At \(\varepsilon=2\), the two models have essentially identical test accuracy, and the set-size advantage of DP-SCP-A also becomes small.

Finally, Supplementary Section~\ref{sec:supp_realworld} examines \(p\in\{0.3,0.5,0.7,0.9\}\) and further clarifies the allocation tradeoff. At \(\varepsilon=0.5\), increasing \(p\) improves the classifier enough to sharply reduce the DP-SCP-A set size even though the calibration noise increases, showing that the large sets under strong privacy cannot be attributed to calibration noise alone. Rather, the results reflect the balance between improving the fitted model and retaining enough privacy for calibration. At \(\varepsilon=1\), most of the predictive and set-size improvement is realized by \(p=0.7\), whereas further allocation to training mainly increases the calibration cost for DP-SCP-F. Meanwhile, at \(\varepsilon=2\), the precommitted main choice \(p=0.5\) already gives the smallest average set sizes for both variants among the four allocations. Thus, the sensitivity analysis shows a genuine tradeoff between privacy spent on training and privacy spent on calibration rather than a uniformly preferable allocation.

\subsection{Housing Price Analysis: Regression}
\label{subsec:real_data_regression}

We evaluate DP-SCP on the California Housing dataset \citep{pace1997sparse}, which contains \(20{,}640\) observations and eight covariates. For each repetition, we randomly split the data into an 80\% fitting pool and a 20\% test set, and clip the covariates to fixed, data-independent bounds before training. We train a three-hidden-layer MLP with widths \(128\), \(128\), and \(64\). The conformal score is the absolute residual \(s(x,y)=|y-\hat f(x)|\), and we report RMSE and prediction-interval width on the original response scale.

\begin{table}[!ht]
\centering
\caption{California Housing regression results. Values are mean (SD) over 30 repetitions. The 95\% confidence interval is for mean coverage. Boldface marks the best utility among DP methods whose 95\% confidence interval does not lie entirely below the level \(0.90\).}
\label{tab:housing_main}
\small\spacingset{1.1}
\setlength{\tabcolsep}{5pt}
\renewcommand{\arraystretch}{1.08}
\begin{tabular}{llccc}
\toprule
\textbf{Privacy} & \textbf{Method} & \textbf{Coverage} & \textbf{95\% CI} & \textbf{Avg. width} \(\downarrow\) \\
\midrule

\multirow{3}{*}{\(\varepsilon=0.5\)}
 & DP-SCP-F & 0.9094 (0.0044) & [0.9078, 0.9111] & 2.1413 (0.0488) \\
 & DP-SCP-A & 0.9012 (0.0048) & [0.8994, 0.9030] & \textbf{2.0453 (0.0471)} \\
 & DP-Split & 0.8996 (0.0053) & [0.8976, 0.9015] & 2.1690 (0.0544) \\
\midrule

\multirow{3}{*}{\(\varepsilon=1.0\)}
 & DP-SCP-F & 0.9057 (0.0044) & [0.9041, 0.9073] & 2.0765 (0.0376) \\
 & DP-SCP-A & 0.9010 (0.0047) & [0.8993, 0.9028] & \textbf{2.0241 (0.0385)} \\
 & DP-Split & 0.9004 (0.0049) & [0.8986, 0.9023] & 2.1745 (0.0459) \\
\midrule

\multirow{3}{*}{\(\varepsilon=2.0\)}
 & DP-SCP-F & 0.9041 (0.0042) & [0.9025, 0.9057] & 2.0521 (0.0356) \\
 & DP-SCP-A & 0.9012 (0.0040) & [0.8997, 0.9027] & \textbf{2.0212 (0.0348)} \\
 & DP-Split & 0.9002 (0.0049) & [0.8983, 0.9020] & 2.1716 (0.0422) \\
\midrule
\midrule

\multicolumn{2}{l}{\textit{Non-private Baselines}} & & & \\
 & Naive Full & 0.8601 (0.0084) & [0.8570, 0.8633] & 1.2805 (0.0381) \\
 & Split CP & 0.8999 (0.0065) & [0.8975, 0.9023] & 1.6363 (0.0494) \\
\bottomrule
\end{tabular}
\end{table}

Table~\ref{tab:housing_main} reports coverage and average interval width across the three privacy levels. We bold the smallest width among DP methods whose 95\% confidence interval for mean coverage does not lie entirely below \(0.90\). Under this criterion, DP-SCP-A gives the narrowest intervals at all three privacy levels. Before examining the private methods more closely, the nonprivate baselines again show the cost of direct reuse. Naive Full attains coverage \(0.8601\), compared with \(0.8999\) for Split CP, while its residual threshold is only \(0.6402\), compared with \(0.8181\) for Split CP. The Full model also has training RMSE \(0.4149\) and test RMSE \(0.5089\), so the narrower Naive Full intervals reflect optimistic reused residuals and substantial undercoverage rather than a meaningful utility gain.

Turning to the private methods, DP-SCP-A produces the narrowest intervals among the DP methods across all three privacy levels. DP-SCP-F is more conservative than DP-SCP-A but still produces narrower intervals than DP-Split throughout. Relative to DP-Split, DP-SCP-A reduces average width by \(5.67\%\), \(6.90\%\), and \(6.92\%\) at \(\varepsilon=0.5\), \(1\), and \(2\), respectively.

The comparison between DP-SCP-A and DP-SCP-F highlights the calibration effect because the two variants use the same trained model. DP-SCP-F has higher coverage and wider intervals, with the coverage gap between the two DP-SCP methods shrinking from \(0.82\) to \(0.29\) percentage points as \(\varepsilon\) increases; nevertheless, it remains narrower than DP-Split, reducing width by \(1.24\%\), \(4.49\%\), and \(5.49\%\) across the three privacy levels. Thus, the full-data training gain remains visible even after the more conservative DP-SCP-F calibration enlarges the intervals.

\begin{table}[!htbp]
\centering
\caption{Predictive quality of the DP training. DP-SCP-F and DP-SCP-A share the same trained model, denoted DP-SCP in the table. Values are test RMSE mean (SD) over 30 repetitions. The difference is within repetition as DP-Split minus DP-SCP, and the relative reduction is within repetition relative to DP-Split. DP-SCP uses the full fitting pool with \(p=0.5\), whereas DP-Split uses half of the pool with the full privacy budget.}
\label{tab:housing_training}
\small\spacingset{1.1}
\setlength{\tabcolsep}{5pt}
\renewcommand{\arraystretch}{1.08}
\begin{tabular}{lcccc}
\toprule
\textbf{Privacy} & \textbf{DP-SCP} & \textbf{DP-Split} & \textbf{Difference} & \textbf{Relative reduction} \\
\midrule
\(\varepsilon=0.5\) & 0.6824 (0.0156) & 0.7258 (0.0193) & 0.0434 (0.0098) & 5.96\% (1.26\%) \\
\(\varepsilon=1.0\) & 0.6797 (0.0157) & 0.7257 (0.0180) & 0.0460 (0.0065) & 6.33\% (0.82\%) \\
\(\varepsilon=2.0\) & 0.6786 (0.0148) & 0.7254 (0.0176) & 0.0468 (0.0058) & 6.44\% (0.71\%) \\
\bottomrule
\end{tabular}
\end{table}

Table~\ref{tab:housing_training} compares the predictive quality under the two training schemes: DP-SCP uses the full fitting pool with a smaller training privacy budget, whereas DP-Split uses half of the observations with the full training budget. Despite the smaller budget, DP-SCP lowers test RMSE by about \(6\%\) at every privacy level, while DP-SCP-A reduces interval width by \(5.67\%\) to \(6.92\%\). Thus, using the full data for training improves prediction despite the smaller training privacy budget, while DP-SCP-A also yields sharper prediction intervals.

Within DP-SCP, increasing \(\varepsilon\) from \(0.5\) to \(2\) changes little: test RMSE decreases only from \(0.6824\) to \(0.6786\), and DP-SCP-A width decreases from \(2.0453\) to \(2.0212\). In contrast, at each privacy level, DP-SCP training lowers RMSE relative to DP-Split by about \(0.043\)--\(0.047\), while DP-SCP-A narrows the intervals by about \(0.124\)--\(0.150\). Thus, in this experiment, the gain from the full fitting pool is substantially larger than the change from increasing \(\varepsilon\).

Finally, additional results in Supplementary Section~\ref{sec:supp_realworld} show that most improvements in training RMSE and DP-SCP-A interval width occur from \(p=0.3\) to \(0.5\), with little further gain beyond \(p=0.5\) as less privacy remains for calibration. Meanwhile, DP-SCP-F becomes wider as the calibration correction grows. Thus, too little privacy for training weakens prediction, whereas further allocation eventually yields little predictive gain while increasing the calibration cost.

\section{Conclusion and Discussion}
\label{sec:conclusion}

In this work, we introduced DP-Stabilized Conformal Prediction (DP-SCP), a full-data conformal framework that combines private training with a mechanism-specific stability analysis. This analysis controls the discrepancy between in-sample and out-of-sample nonconformity scores while allowing the sensitive data to be reused for private calibration. Our theoretical coverage analysis separately characterizes the universal coverage guarantee that follows directly from DP and the stronger coverage guarantees obtained by additionally using stability of the training mechanism. DP-SCP-F uses a more conservative calibration to strengthen finite-sample coverage, whereas DP-SCP-A uses a less conservative calibration and can produce more informative prediction sets. In our numerical studies, both DP-SCP-F and DP-SCP-A generally produce more informative prediction sets than DP-Split, while DP-SCP-F remains more conservative and DP-SCP-A stays closer to the nominal coverage level, as predicted by our theory.

A broader implication is that DP provides bounds on how much an algorithm's output distribution can change when a single data point changes, and these bounds can be useful beyond privacy. In this paper, we use such bounds to justify reusing the full data for CP, while \citet{zhang2024winners} use a randomized weighting rule motivated by DP to stabilize discrete argmin inference. These examples suggest that there may be many other statistical problems where bounds provided by DP can play a useful role.


\section*{Data Availability Statement}
The BloodMNIST dataset is available through MedMNIST at \url{https://medmnist.com/}. The California Housing dataset is available at \url{https://scikit-learn.org/stable/modules/generated/fetch_california_housing.html}. 

\section*{Declaration of AI Use}

The authors used ChatGPT (OpenAI, GPT-5.6 Sol) to assist with manuscript writing and code development. However, all research ideas, methodology, theory, and results are the authors’ own, and the authors have carefully reviewed the manuscript and code and take full responsibility for the paper.


\bibliography{bibliography.bib}

\newpage
\appendix
\setcounter{section}{0}
\setcounter{table}{0}
\setcounter{figure}{0}
\setcounter{equation}{0}
\renewcommand{\thesection}{\Alph{section}}
\renewcommand{\thetable}{S\arabic{table}}
\renewcommand{\thefigure}{S\arabic{figure}}
\renewcommand{\theequation}{S\arabic{equation}}
\renewcommand{\theHequation}{S\arabic{equation}}

\spacingset{1.1}

\begin{center}
{\large\bf SUPPLEMENTARY MATERIAL of \\``Beyond Data Splitting: Full-Data Conformal
Prediction by Differential Privacy''}
\end{center}

\vspace{0.5cm}

The Supplement provides extended background, assumption checks, complete
proofs, privacy-accounting details, and additional experiments.  Section~\ref{sec:supp_dp_intro} expands the DP background. Section~\ref{sec:quantile_discussion}
discusses the right-endpoint private search. Section~\ref{sec:assumption_validation} verifies exchangeability, score regularity, and mechanism-specific stability conditions. Section~\ref{sec:supp-proofs} provides complete proofs of the theoretical results in the main text. Section~\ref{supp:extended numerical studies} gives the detailed privacy accounting, allocation sweeps, and synthetic experiments.

\section{Extended Introduction to Differential Privacy} \label{sec:supp_dp_intro}

Differential privacy (DP) has been characterised in many different ways. At its core, DP is about quantifying how similar the two output distributions should be when the underlying datasets differ in a single entry. Different choices of similarity measures naturally lead to different DP definitions, and many of these variants were developed to enable tighter privacy accounting under composition. 

DP can be naturally cast as a hypothesis testing \citep{dong2022gaussian}. Consider a randomized mechanism $M: \mathcal{D} \to \mathcal{Y}$ and two adjacent datasets $D$ and $D'$. An adversary observing the output $y = M(\cdot)$ seeks to distinguish between the hypotheses:
\begin{equation*}
    H_0: \text{The underlying dataset is } D \; \text{vs.} \; H_1: \text{The underlying dataset is } D'.
\end{equation*}
Let $P = M(D)$ and $Q = M(D')$ denote the probability distributions of the outputs under $H_0$ and $H_1$, respectively. The difficulty of this testing problem is fully characterised by the trade-off between Type I error ($\alpha$) and Type II error ($\beta$).

\begin{definition}[Trade-off Function \citep{dong2022gaussian}]
    For any two probability distributions $P$ and $Q$, the trade-off function $T(P, Q): [0, 1] \to [0, 1]$ is defined as the minimum achievable Type II error for a given Type I error $\alpha$:
    \begin{equation}
        T(P, Q)(\alpha) = \inf_{\phi} \left\{ 1 - \mathbb{E}_Q[\phi] : \mathbb{E}_P[\phi] \le \alpha \right\},
    \end{equation}
    where the infimum is taken over all measurable rejection rules (tests) $\phi: \mathcal{Y} \to [0, 1]$.
\end{definition}

A function $f: [0, 1] \to [0, 1]$ is a valid trade-off function if and only if it is convex, continuous, non-increasing, and satisfies $f(x) \le 1-x$ for all $x \in [0, 1]$. The condition $T(P, Q) \ge f$ (pointwise inequality) implies that distinguishing $P$ from $Q$ is at least as hard as the problem characterised by $f$.

The $f$-DP framework \citep{dong2022gaussian} generalises DP by parametrising the privacy guarantee directly via a trade-off function $f$.

\begin{definition}[$f$-Differential Privacy \citep{dong2022gaussian}]
    Let $f$ be a symmetric trade-off function. A mechanism $M$ is said to be $f$-differentially private ($f$-DP) if for all adjacent datasets $D, D'$, we have
    \begin{equation*}
        T(M(D), M(D'))(\alpha) \ge f(\alpha), \; \forall \alpha \in [0, 1].
    \end{equation*}
\end{definition}

As we have presented in Section~\ref{sec:background}, $f$-DP has two standard subclasses.

\begin{definition}[$(\varepsilon,\delta)$-Differential Privacy \citep{dwork2006calibrating}]
For $\varepsilon\ge 0$ and $\delta\in[0,1]$, a mechanism $M$ is $(\varepsilon,\delta)$-DP if for all adjacent datasets $D,D'$ and all measurable sets $A$, $\mathbb P(M(D)\in A) \le e^{\varepsilon}\,\mathbb P(M(D')\in A)+\delta.$
\end{definition}

The $(\varepsilon,\delta)$-DP definition is one of the earliest formulations of DP and, arguably, remains the most widely used in practice. Under the $f$-DP framework, $(\varepsilon,\delta)$-DP is equivalent to $f_{\varepsilon,\delta}$-DP, where $f_{\varepsilon,\delta}(\alpha)=\max\{0,\ 1-\delta-e^{\varepsilon}\alpha,\ e^{-\varepsilon}(1-\delta-\alpha)\}$ for $\alpha\in[0,1]$. 
\begin{definition}[Gaussian Differential Privacy \citep{dong2022gaussian}] For $\mu>0$, let $G_{\mu}(\alpha)
=
T\!\left(\mathcal N(0,1),\,\mathcal N(\mu,1)\right)(\alpha),$ for $\alpha\in[0,1].$
A mechanism $M$ is $\mu$-GDP if it is $G_{\mu}$-DP.
\end{definition}

The tradeoff function for GDP has the explicit form $G_{\mu}(\alpha)=\Phi\!\Bigl(\Phi^{-1}(1-\alpha)-\mu\Bigr),$
where $\Phi$ is the standard normal distribution function.

\begin{proposition}[Conversion between $\mu$-GDP and $(\varepsilon,\delta)$-DP \citep{balle2018improving}] If $M$ is $\mu$-GDP, then for every $\varepsilon\ge 0$, $M$ is $(\varepsilon,\delta(\varepsilon;\mu))$-DP with
\begin{equation*}
\delta(\varepsilon;\mu)
=
\Phi\!\Bigl(-\frac{\varepsilon}{\mu}+\frac{\mu}{2}\Bigr)
-
e^{\varepsilon}\,
\Phi\!\Bigl(-\frac{\varepsilon}{\mu}-\frac{\mu}{2}\Bigr).
\end{equation*}
\end{proposition}

While our theoretical development is stated in the $f$-DP and GDP, our numerical studies follow the privacy accounting via R\'enyi DP (RDP), which is convenient for composition and subsampling.

\begin{definition}[Rényi Differential Privacy (RDP) \citep{mironov2017renyi}]
Let $\alpha>1$. A mechanism $M$ satisfies $(\alpha,\varepsilon)$-RDP if for all adjacent datasets $D,D'$,
\begin{equation*}
D_\alpha\!\big(M(D)\,\|\,M(D')\big)\le \varepsilon,
\end{equation*}
where $D_{\alpha}$ is the R\'enyi divergence of order $\alpha$.
\end{definition}

Note that if a mechanism is $\mu$-GDP, then for every order $\alpha>1$,
\[
D_\alpha\!\bigl(M(D)\,\|\,M(D')\bigr)
\le
D_\alpha\!\bigl(N(0,1)\,\|\,N(\mu,1)\bigr)
=
\frac{\alpha\mu^2}{2}.
\]
Hence, a $\mu$-GDP mechanism implies $(\alpha,\alpha\mu^2/2)$-RDP for every $\alpha>1$. 

RDP composes additively across adaptive sequential compositions, and it can be converted to an $(\varepsilon,\delta)$-DP. In our experiments, we use \texttt{Opacus} to track the RDP profile of subsampled Gaussian mechanisms during DP-SGD training and then convert to $(\varepsilon,\delta)$-DP. We defer a more detailed discussion of why we use RDP rather than GDP to the next subsection on DP-SGD, where privacy amplification by subsampling is discussed explicitly.

\subsection{Private Model Training via DP-SGD}

There is an extensive literature on differentially private empirical risk minimization, including the exponential mechanism \citep{mcsherry2007mechanism,gopi2022private}, objective perturbation \citep{chaudhuri2011differentially,cho2024privacy}, and output perturbation \citep{chaudhuri2011differentially}. We use Differentially Private Stochastic Gradient Descent (DP-SGD) \citep{abadi2016deep} as the default scalable training procedure. Write $\mathbb B_2(C)=\{v\in\mathbb R^d:\|v\|_2\le C\}$ for the clipping ball.

\begin{algorithm}[t]
\caption{DP--SGD}
\label{alg:dp-sgd}
\begin{algorithmic}[1]
\Require dataset $D_n=\{z_i\}_{i=1}^n$, loss $\ell(\theta;z)$, initial parameter $\theta_0\in\mathbb{R}^d$, steps $T$, learning rates $\{\eta_t\}_{t=0}^{T-1}$, Poisson subsampling rate $q\in(0,1)$, clipping norm $C>0$, noise multiplier $\sigma_{\mathrm{sgd}}>0$, deterministic normalizer $\bar b>0$
\For{$t=0,1,\dots,T-1$}
  \State sample $B_t\subseteq\{1,\dots,n\}$ by independent Bernoulli$(q)$ inclusion, allowing $B_t=\varnothing$
  \State compute $g_i\gets\nabla_\theta \ell(\theta_t;z_i)$ for each $i\in B_t$
  \State set $\bar g_i\gets\Pi_{\mathbb B_2(C)}(g_i)$ for each $i\in B_t$
  \State draw $G_t\sim\mathcal{N}(0,I_d)$
  \State set $\widetilde g_t
  \gets
  \frac{1}{\bar b}
  \left(
  \sum_{i\in B_t}\bar g_i
  +
  \sigma_{\mathrm{sgd}}C G_t
  \right)
  $
  \State update $\theta_{t+1}\gets\theta_t-\eta_t\widetilde g_t$
\EndFor
\State \Return $\theta_T$
\end{algorithmic}
\end{algorithm}

Algorithm~\ref{alg:dp-sgd} uses Poisson subsampling, under which the realised minibatch size is random and may equal zero. The clipped and noised gradient sum is therefore normalised by a positive deterministic quantity $\bar b$, fixed before training, rather than by the realised size $|B_t|$. In the experiments, this convention is implemented through the \texttt{expected\_batch\_size} used by the mean-reduction DP optimizer in \texttt{Opacus}. The normaliser is held fixed across iterations even though $|B_t|$ varies.

This convention makes every Poisson draw well-defined. We interpret the sum over an empty minibatch as zero. Hence, if $B_t=\varnothing$, then
\(
\widetilde g_t
=
\frac{\sigma_{\mathrm{sgd}}C}{\bar b}G_t.
\)
The empty minibatch is neither discarded nor resampled, and the Gaussian perturbation is still applied. This is the empty-batch and normalisation convention used throughout the DP-SGD analysis below.\\

\noindent\textbf{Gradient Clipping.} Gradient clipping plays a dichotomous role, serving as a practical necessity while simultaneously introducing significant theoretical hurdles.

From a practical standpoint, clipping is indispensable for training modern complex models. In these settings, the global sensitivity of the gradient is often intractable or theoretically unbounded. Clipping enforces a deterministic upper bound on the influence of any single individual. This creates a bounded sensitivity, allowing for the injection of calibrated noise to guarantee privacy without relying on worst-case assumptions about the data distribution.

However, clipping makes the sample gradient biased for the unclipped population gradient, so analyses based directly on an unbiased-gradient recursion do not apply.\\

\noindent\textbf{Privacy Amplification by Subsampling and Privacy Accounting}. Intuitively, if a datapoint is not included in the batch, such datapoint enjoys the full privacy during the update as it does not contribute to the model update at all.

If a base mechanism satisfies $f$-DP and each record is retained independently with probability $q$, the amplified tradeoff curve is obtained by convexifying and symmetrizing the curve generated by $qf+(1-q)\operatorname{Id}$ \citep{bu2020deep}.

While $f$-DP is closed under subsampling, the exact amplified tradeoff curve of a subsampled Gaussian mechanism does not generally remain inside the one-parameter GDP family. A valid, possibly non-tight GDP lower guarantee may still be obtained, but exact practical accounting is more naturally carried out with the full amplified tradeoff or with RDP. 

Practical libraries such as Opacus \citep{yousefpour2021opacus} typically utilise Rényi DP (RDP) for accounting. RDP provides tight composition bounds for subsampled Gaussian mechanisms and allows for the exact tracking of the privacy budget across iterations. In our experiments, we adopt this standard approach, implemented via \texttt{Opacus}.

\subsection{Canonical Noise Distribution}
In $f$-DP framework, a practical question is: for a given target $f$, what is the optimal noise to add? This motivates the concept of a \textit{Canonical Noise Distribution} (CND), introduced by \cite{awan2023canonical}. In the context of additive mechanisms, where we release a statistic $\theta(D)$ by adding independent noise $Z$, a CND is a distribution to match the $f$ exactly. 

\begin{definition}[Canonical Noise Distribution \citep{awan2023canonical}]
    Let $f$ be a symmetric nontrivial trade-off function. A continuous cumulative distribution function $F$ is a \textbf{Canonical Noise Distribution (CND)} for $f$ if:{
\setlength{\leftmargini}{1.5em} \begin{enumerate}
        \item For every scalar statistic $S$ with sensitivity at most $\Delta>0$ and noise $N\sim F$, the additive mechanism $S(D)+\Delta N$ satisfies $f$-DP.
        \item The privacy guarantee is tight, i.e., $T(F(\cdot), F(\cdot - 1)) = f$.
        \item The trade-off function satisfies $T(F(\cdot), F(\cdot - 1))(\alpha) = F(F^{-1}(1-\alpha) - 1)$ for all $\alpha \in (0, 1)$.
        \item $F$ corresponds to a random variable symmetric about zero, i.e., $F(x) = 1 - F(-x)$ for all $x \in \mathbb{R}$.
    \end{enumerate}
    }
\end{definition}

If a mechanism adds noise distributed according to a CND (scaled by the sensitivity of the query), it satisfies $f$-DP in a \textit{lossless} manner.

\section{Additional Discussion on Quantile Estimation}\label{sec:quantile_discussion}

In this section, we revisit DP quantile estimators based on a \emph{noisy midpoint search}, as Algorithm~1 in \cite{romanus2025differentially}. At a high level, these methods maintain a bracket $[\,\texttt{left},\texttt{right}\,]$, repeatedly query a noisy count at the midpoint, update both endpoints depending on the noisy inequality, and finally return the midpoint as the DP quantile estimate. We show that this design admits structural failure modes under DP noise, both in the presence of large tie jumps and in completely tie-free settings. This motivates our use of a \emph{buffered right-endpoint} rule with one-sided updates in the main algorithm.

\begin{algorithm}[ht]
\caption{Noisy Midpoint DP Quantile (schematic version of Algorithm 1 in \cite{romanus2025differentially})}
\label{alg:noisy-midpoint}
\begin{algorithmic}[1]
\Require Calibration scores $\mathcal{S}$, significance level $\alpha$, range $[a,b]$, precision $\delta$, privacy parameter $\rho$
\Ensure DP quantile $q^{\mathrm{DP}}$
\State $r \leftarrow \big\lceil (1-\alpha)(n_{\mathrm{cal}}+1)\big\rceil$
\State $N \leftarrow \big\lceil \log_2\!\big((b-a)/\delta\big) \big\rceil$
\State $\texttt{left} \leftarrow a$, \quad $\texttt{right} \leftarrow b$, \quad $i \leftarrow 0$
\While{$i \le N$}
    \State $\texttt{mid} \leftarrow (\texttt{left} + \texttt{right})/2$
    \State $\tilde c \leftarrow \text{NoisyRC}([a,\texttt{mid}], \mathcal{S})$ \Comment{noisy count over $[a,\texttt{mid}]$}
    \If{$\tilde c < r$}
        \State $\texttt{left} \leftarrow \texttt{mid} + \delta$
    \Else
        \State $\texttt{right} \leftarrow \texttt{mid}$
    \EndIf
    \State $i \leftarrow i+1$
\EndWhile
\State \Return $q^{\mathrm{DP}} \leftarrow (\texttt{left} + \texttt{right})/2$
\end{algorithmic}
\end{algorithm}

This midpoint rule is inherently vulnerable to noisy misclassification: a single \emph{false positive} (i.e., $\tilde c \ge r$ when the true count is strictly below $r$) at a point $t$ that lies below the target quantile can force the right boundary $\texttt{right}$ below the desired level.  Thereafter, the search can only move within a bracket that never crosses the true quantile, and the final midpoint necessarily underestimates it. We illustrate this phenomenon first under large tie jumps and then under strictly increasing (no-tie) scores.

\subsubsection*{Example 1: Large tie jump (catastrophic under-coverage)}

Let $n=14$ and consider the calibration scores
\[
\mathcal S
= \underbrace{[\,0,0,0,0,0\,]}_{\text{5 zeros}}
\cup 
\underbrace{[\,10,10,10,10,10,10,10,10\,]}_{\text{8 tens}}
\cup
\underbrace{[\,11\,]}_{\text{one outlier}},
\]
so the sorted scores are $[\,0,0,0,0,0,10,\dots,10,11\,]$. For $\alpha=0.2$ we have
\[
r \;=\; \Big\lceil (1-\alpha)(n+1)\Big\rceil = \big\lceil 0.8\cdot 15\big\rceil = 12,
\]
and the $r$-th order statistic is $s_{[r]} = s_{[12]} = 10$. The empirical count function 
$C(t) := \#\{i:S_i\le t\}$ satisfies
\[
C(t) =
\begin{cases}
0, & t<0,\\[2pt]
5, & 0\le t<10,\\[2pt]
13, & 10\le t<11,\\[2pt]
14, & t\ge 11.
\end{cases}
\]
Thus $r=12$ lies strictly inside the tie jump at $t^\star=10$, where $C$ jumps from $5$ to $13$.

Run Algorithm~\ref{alg:noisy-midpoint} on $[a,b]=[0,11]$, and model the noisy count as 
$\tilde C(t) = C(t) + Z$ with $Z\sim\mathcal N(0,\sigma^2)$, independent across queries.
With positive probability (uniformly bounded away from zero in $\sigma$) the following event occurs:
at some iteration the algorithm queries a point $t_1<10$ (e.g., $t_1=9.5$), where $C(t_1)=5<r$, but the noise
realization $Z_1$ is sufficiently large and positive so that
\[
\tilde C(t_1) = C(t_1) + Z_1 \;\ge\; r.
\]
The algorithm then takes the ``$\tilde c \ge r$'' branch and shrinks the bracket from 
$[\,\texttt{left},\texttt{right}\,]$ down to $[\,\texttt{left}, t_1\,]$ with 
$\texttt{right} \leftarrow t_1 < 10 = s_{[r]}$.

From that point onward, all subsequent midpoints satisfy $\texttt{mid} \in [\texttt{left},\texttt{right}] \subset [a,10)$, 
hence $C(\texttt{mid}) = 5 < r$. While additional noisy comparisons may occasionally further reduce $\texttt{right}$,
they can never move it back above $10$. The bracket therefore collapses entirely inside $[a,10)$, and the returned
midpoint
\[
\hat q^{\mathrm{DP}} = \frac{\texttt{left}+\texttt{right}}{2}
\]
necessarily satisfies $\hat q^{\mathrm{DP}} < 10 = s_{[r]}$, so $C(\hat q^{\mathrm{DP}})=5<r$ and the induced prediction
set $\{y: s(x_{\rm test},y)\le \hat q^{\mathrm{DP}}\}$ is too small, causing under-coverage. This example shows that a single
false positive strictly below a large tie jump suffices for the noisy midpoint rule to output $\hat q^{\mathrm{DP}}<s_{[r]}$.

\subsubsection*{Example 2: No-tie scores (failure without ties)}

The failure of the noisy midpoint rule is not specific to ties. It persists even when all calibration scores are distinct.

Let $n=10$ and take strictly increasing scores
\[
\mathcal S = [\,1,2,3,4,5,6,7,8,9,10\,].
\]
For $\alpha=0.2$,
\[
r = \Big\lceil (1-\alpha)(n+1)\Big\rceil = \big\lceil 0.8\cdot 11\big\rceil = 9,
\qquad
s_{[r]} = s_{[9]} = 9.
\]
The empirical count $C(t) = \#\{i:S_i\le t\}$ increases by $1$ at each integer and has no ties.
We again run Algorithm~\ref{alg:noisy-midpoint} on $[a,b]=[1,10]$ with 
$\tilde C(t) = C(t) + Z$, $Z\sim\mathcal N(0,\sigma^2)$.

Consider the following event, which has strictly positive probability for any fixed $\sigma>0$:
at some iteration the algorithm queries $t_1 = 8.5$. Then $C(8.5) = 8 < r$, but the noise
realization $Z_1$ happens to satisfy $Z_1 \ge 1$, so that
\[
\tilde C(8.5) = 8 + Z_1 \;\ge\; 9 = r.
\]
The midpoint rule then takes the ``$\tilde c \ge r$'' branch and shrinks the right boundary to
\[
\texttt{right} \leftarrow 8.5 < 9 = s_{[r]}.
\]

From this point onward, all subsequent midpoints lie in $[\,\texttt{left},\texttt{right}\,] \subset [1,8.5)$,
and therefore satisfy $C(\texttt{mid}) \le 8 < r$. Additional positive noises can only move $\texttt{right}$
further below $8.5$, while negative or small noises move $\texttt{left}$ upward but never beyond $\texttt{right}$.
Consequently, the interval collapses entirely inside $[1,8.5)$ and the final returned midpoint satisfies
\[
\hat q^{\mathrm{DP}}
= \frac{\texttt{left}+\texttt{right}}{2}
\;<\; 8.5
\;<\; 9 = s_{[r]},
\]
so $C(\hat q^{\mathrm{DP}})\le 8<r$ and the induced prediction set again under-covers. 
This example demonstrates that the midpoint-return rule can under-estimate the target quantile purely due to a single noisy false positive below $s_{[r]}$, even in the absence of ties.\\

\noindent\textbf{Implications for Algorithm Design} Our proposed Algorithm~\ref{alg:dp_binary_search} structurally addresses these vulnerabilities. First, by employing a one-sided noise correction $\tau$, we ensure that the condition $\tilde{C}_n(t) \ge r'$ implies the true count condition $C_n(t) \ge r$ with high probability, effectively blocking the ``false positive'' failure path. Second, by returning the \texttt{right} endpoint instead of the midpoint, we maintain the invariant that the returned threshold is always an upper bound on the valid region throughout the search process (conditioned on the good event). This conservative design deliberately trades off a small amount of efficiency (slightly larger sets) to strictly guarantee coverage, a necessity in safety-critical applications.

\section{Validation of Assumptions}\label{sec:assumption_validation}

This section expands on the discussions in Example~\ref{ex:reg} and Remark~\ref{remark: dp mech stab}. We first discuss the exchangeability of the ideal scores. We then revisit the regression example in detail. 

\subsection{Exchangeability of the ideal scores}

The ideal scores are computed from the model $\theta_{n+1}$ trained on the full $D_{n+1}$. Permutation invariance is a sufficient structural condition for the i.i.d. sample symmetry to pass through the training mechanism and score map. Without it, exchangeability can fail, as the following example shows.

\begin{example}[Failure without permutation invariance]
\label{ex:permutation_invariance_failure}
Let $Z_1,\dots,Z_{n+1}$ be i.i.d. from a nondegenerate distribution. Consider
the order-dependent rule $M(Z_1,\dots,Z_{n+1})=Z_1$ and the score
$s(z;\theta)=|z-\theta|$. Then
\begin{align*}
S_1
&=|Z_1-Z_1|=0,\\
S_i
&=|Z_i-Z_1|,
\; i\ne1.
\end{align*}
The first score is identically zero, whereas every $S_i$ with $i>1$ is
nondegenerate. The scores therefore do not even have identical marginal laws
and hence are not exchangeable.
\end{example}

Many private mechanisms of interest are permutation invariant, including
DP-SGD with Poisson subsampling or random shuffling, output perturbation, and
objective perturbation. For these mechanisms the dataset is treated as a
multiset, so the symmetry of the i.i.d. sample is preserved. The example does
not claim that every non-invariant mechanism must destroy exchangeability in
every special construction; it shows why permutation invariance is imposed as
a clean sufficient condition.

\subsection{Verification for residual scores in regression}
\label{subsec:regression_score_regularity}

\begin{example}\label{ex:reg}
Consider the additive-noise model $Y=f_0(X)+\varepsilon$. To supply the public range required by Algorithm~\ref{alg:dp_binary_search}, suppose domain knowledge specifies an admissible response interval $[y_{\min},y_{\max}]$ independently of the sample, and use the range-constrained predictor $\bar f_\theta(x)=\Pi_{[y_{\min},y_{\max}]}\{f_\theta(x)\}$. The residual score $s((X,Y);\theta)=|Y-\bar f_\theta(X)|$ then lies in the public interval $[0,y_{\max}-y_{\min}]$. Projection onto an interval is non-expansive, so any uniform parameter-Lipschitz or response-influence bound for $f_\theta$ is preserved for $\bar f_\theta$. Other application-specific public bounds may be used analogously.

Assumption~\ref{ass:lipschitz} therefore follows for linear regression with bounded covariates, generalized linear model-type predictors with a Lipschitz link, and linear heads on bounded fixed representations. For score regularity, fix $i$, vary only $\varepsilon_i$ from $u$ to $v$, and let $\widehat f_j^{(i,u)}$ be the resulting range-constrained fitted value at $X_j$ under the same auxiliary training randomness. Suppose that the conditional density of $\varepsilon_i$ is uniformly bounded by a deterministic constant $M_\varepsilon$, and that deterministic constants $\kappa_{\mathrm{self}},\kappa_{\mathrm{cross}}\ge0$ bound the response influence at the perturbed and competing observations, uniformly in $i,u,v$, with $\kappa_{\mathrm{self}}+\kappa_{\mathrm{cross}}<1$. Proposition~\ref{prop:regression_score_regularity} then gives pairwise no ties and two-sided anti-concentration around every ideal order statistic, with constant $2M_\varepsilon/(1-\kappa_{\mathrm{self}}-\kappa_{\mathrm{cross}})$. The public bounded-range condition and this score-regularity condition play separate roles. Linear smoothers and smooth strongly convex estimators satisfy the latter under explicit leverage or curvature bounds; details are given in Supplementary Section~\ref{subsec:regression_score_regularity}.

Consider the additive-noise regression model $Y=f_0(X)+\varepsilon$. To match the public search-range convention in Example~\ref{ex:reg}, suppose domain knowledge supplies $[y_{\min},y_{\max}]$, use the projected predictor $\bar f_\theta=\Pi_{[y_{\min},y_{\max}]}\circ f_\theta$, and write $f_\theta$ below for this range-constrained predictor. Then $s((X,Y);\theta)=|Y-f_\theta(X)|$ lies in $[0,y_{\max}-y_{\min}]$. Because projection is non-expansive, it cannot increase any predictor difference used in the Lipschitz or response-influence bounds. We first verify Assumption~\ref{ass:lipschitz}. By the reverse triangle inequality,
\begin{equation*}
\left|
|Y-f_\theta(X)|-|Y-f_{\theta'}(X)|
\right|
\le
|f_\theta(X)-f_{\theta'}(X)|.
\end{equation*}
Thus the score inherits any uniform parameter-Lipschitz bound for the predictor. In particular, the corresponding constants are $L=B$ for $f_\theta(x)=\theta^\top x$ with $\|x\|_2\le B$, $L=\operatorname{Lip}(g)B$ for $f_\theta(x)=g(\theta^\top x)$ with a Lipschitz link $g$, and $L=B_\phi$ for $f_\theta(x)=\theta^\top\phi(x)$ with $\|\phi(x)\|_2\le B_\phi$. The last class includes a frozen deep or language-model representation with a trainable linear head \citep{yosinski2014transferable,devlin2019bert,chen2020simple}.

We next verify Assumptions~\ref{ass:noties} and \ref{ass:anti}. Write $m=n+1$ for the size of the ideal sample and let $\omega$ collect all auxiliary randomness of the training mechanism. For each $i\in\{1,\dots,m\}$, define
\begin{equation*}
\mathcal G_i
:=
\sigma\!\left(X_1,\dots,X_m,\{\varepsilon_j:j\ne i\},\omega\right).
\end{equation*}
For $u\in\mathbb R$, replace only the $i$th response by $f_0(X_i)+u$, run the training mechanism with the same realization of $\omega$, and denote the resulting fitted value at $X_j$ by $\widehat f_j^{(i,u)}$. This one-coordinate perturbation changes both the $i$th residual and the fitted values entering the other residual scores. The next proposition controls these two effects jointly.
\end{example}

\begin{proposition}[Regression score regularity]
\label{prop:regression_score_regularity}
Suppose that there is a deterministic constant $M_\varepsilon<\infty$ such that, almost surely, the conditional law of $\varepsilon_i$ given $\mathcal G_i$ has a density bounded by $M_\varepsilon$, uniformly in $i$. Suppose further that deterministic constants $\kappa_{\mathrm{self}},\kappa_{\mathrm{cross}}\ge0$ satisfy $|\widehat f_i^{(i,u)}-\widehat f_i^{(i,v)}|\le\kappa_{\mathrm{self}}|u-v|$ and $|\widehat f_j^{(i,u)}-\widehat f_j^{(i,v)}|\le\kappa_{\mathrm{cross}}|u-v|$ almost surely for every $i$, every $u,v\in\mathbb R$, and every $j\ne i$. Let $c=1-\kappa_{\mathrm{self}}-\kappa_{\mathrm{cross}}>0$, and let $q_k$ be the $k$th order statistic of the $m$ ideal residual scores. Then, for every distinct $i,j\in\{1,\dots,m\}$, every $k\in\{1,\dots,m\}$, and every $\delta>0$,
\begin{equation}\label{eq:regression_uniform_anti}
\begin{aligned}
\mathbb P\!\left(S_i^{(m)}=S_j^{(m)}\right)
&=0,\\
\mathbb P\!\left(q_k<S_i^{(m)}\le q_k+\delta\right)
&\le
\frac{2M_\varepsilon}{c}\,\delta,\\
\mathbb P\!\left(q_k-\delta\le S_i^{(m)}<q_k\right)
&\le
\frac{2M_\varepsilon}{c}\,\delta.
\end{aligned}
\end{equation}
In particular, Assumptions~\ref{ass:noties} and \ref{ass:anti} hold, with \(\overline f=2M_\varepsilon/c\).
\end{proposition}

\begin{proof}
Fix $i$ and condition throughout on $\mathcal G_i$. The covariates, the errors $\{\varepsilon_j:j\ne i\}$, and the auxiliary training randomness are then fixed. Write $p_i(\cdot\mid\mathcal G_i)$ for the remaining conditional density of $\varepsilon_i$, which is bounded by $M_\varepsilon$ almost surely.

\noindent\textbf{Step 1. Own and competing score geometry.}
For $u\in\mathbb R$, define the signed own residual and the residual-score maps by
\begin{equation*}
\begin{aligned}
g_i(u)
&:=
 f_0(X_i)+u-\widehat f_i^{(i,u)},
&
 h_i(u)
&:=
 |g_i(u)|,\\
h_j(u)
&:=
 |Y_j-\widehat f_j^{(i,u)}|,
&& j\ne i.
\end{aligned}
\end{equation*}
Under the conditioning, replacing $\varepsilon_i$ by $u$ produces the score vector $(h_1(u),\dots,h_m(u))$. In particular, $S_j^{(m)}=h_j(\varepsilon_i)$ for every $j$.

For $u>v$, the self-influence condition gives the complete increment bound
\begin{equation}\label{eq:own_signed_increment}
\begin{aligned}
g_i(u)-g_i(v)
&=
\left\{f_0(X_i)+u-\widehat f_i^{(i,u)}\right\}
-
\left\{f_0(X_i)+v-\widehat f_i^{(i,v)}\right\}\\
&=
(u-v)
-
\left\{
\widehat f_i^{(i,u)}-\widehat f_i^{(i,v)}
\right\}\\
&\ge
(u-v)
-
\left|
\widehat f_i^{(i,u)}-\widehat f_i^{(i,v)}
\right|\\
&\ge
(1-\kappa_{\mathrm{self}})(u-v).
\end{aligned}
\end{equation}
The first two equalities expand the definition and cancel the common terms. The first inequality uses $-x\ge-|x|$, and the last uses the self-influence bound. The same bound also makes $u\mapsto\widehat f_i^{(i,u)}$ Lipschitz, so $g_i$ is continuous. Since $1-\kappa_{\mathrm{self}}>0$, Equation~\eqref{eq:own_signed_increment} shows that $g_i$ is strictly increasing and hence has at most one zero. The map $h_i=|g_i|$ therefore has at most two monotone branches. If $u>v$ lie on its branch with $g_i\ge0$, then
\begin{equation*}
\begin{aligned}
h_i(u)-h_i(v)
&=
g_i(u)-g_i(v)\\
&\ge
(1-\kappa_{\mathrm{self}})(u-v),
\end{aligned}
\end{equation*}
whereas, if they lie on its branch with $g_i\le0$, then
\begin{equation*}
\begin{aligned}
h_i(v)-h_i(u)
&=
-g_i(v)+g_i(u)\\
&=
g_i(u)-g_i(v)\\
&\ge
(1-\kappa_{\mathrm{self}})(u-v).
\end{aligned}
\end{equation*}
Thus $h_i$ is inverse-Lipschitz with modulus $1-\kappa_{\mathrm{self}}$ on each branch.

For every $j\ne i$, the reverse triangle inequality and the cross-influence condition give
\begin{equation}\label{eq:other_score_lipschitz}
\begin{aligned}
|h_j(u)-h_j(v)|
&=
\left|
|Y_j-\widehat f_j^{(i,u)}|
-
|Y_j-\widehat f_j^{(i,v)}|
\right|\\
&\le
\left|
\widehat f_j^{(i,u)}-\widehat f_j^{(i,v)}
\right|\\
&\le
\kappa_{\mathrm{cross}}|u-v|.
\end{aligned}
\end{equation}

\noindent\textbf{Step 2. Competing order statistics and gap maps.}
For $k\in\{1,\dots,m-1\}$, let $Q_{-i,k}(u)$ be the $k$th order statistic of $\{h_j(u):j\ne i\}$. Equation~\eqref{eq:other_score_lipschitz} implies, coordinatewise for every $j\ne i$,
\begin{equation*}
\begin{aligned}
h_j(u)
&\le
h_j(v)+\kappa_{\mathrm{cross}}|u-v|,\\
h_j(v)
&\le
h_j(u)+\kappa_{\mathrm{cross}}|u-v|.
\end{aligned}
\end{equation*}
Taking the $k$th order statistic preserves each coordinatewise inequality. Consequently,
\begin{equation}\label{eq:competing_order_lipschitz}
\begin{aligned}
Q_{-i,k}(u)
&\le
Q_{-i,k}(v)+\kappa_{\mathrm{cross}}|u-v|,\\
Q_{-i,k}(v)
&\le
Q_{-i,k}(u)+\kappa_{\mathrm{cross}}|u-v|,\\
|Q_{-i,k}(u)-Q_{-i,k}(v)|
&\le
\kappa_{\mathrm{cross}}|u-v|.
\end{aligned}
\end{equation}

Define $d_{i,k}(u)=h_i(u)-Q_{-i,k}(u)$. If $u>v$ lie on the increasing branch of $h_i$, then
\begin{equation*}
\begin{aligned}
d_{i,k}(u)-d_{i,k}(v)
&=
\{h_i(u)-h_i(v)\}
-
\{Q_{-i,k}(u)-Q_{-i,k}(v)\}\\
&\ge
(1-\kappa_{\mathrm{self}})(u-v)
-
|Q_{-i,k}(u)-Q_{-i,k}(v)|\\
&\ge
(1-\kappa_{\mathrm{self}}-\kappa_{\mathrm{cross}})(u-v)\\
&=
c(u-v).
\end{aligned}
\end{equation*}
If $u>v$ lie on the decreasing branch, then
\begin{equation*}
\begin{aligned}
d_{i,k}(v)-d_{i,k}(u)
&=
\{h_i(v)-h_i(u)\}
-
\{Q_{-i,k}(v)-Q_{-i,k}(u)\}\\
&\ge
(1-\kappa_{\mathrm{self}})(u-v)
-
|Q_{-i,k}(v)-Q_{-i,k}(u)|\\
&\ge
(1-\kappa_{\mathrm{self}}-\kappa_{\mathrm{cross}})(u-v)\\
&=
c(u-v).
\end{aligned}
\end{equation*}
The first inequality in each block uses the branchwise bound from Step~1 and subtracts at most the absolute competing-order-statistic increment; the second uses Equation~\eqref{eq:competing_order_lipschitz}. Hence $d_{i,k}$ is strictly monotone and inverse-Lipschitz with modulus $c$ on each branch.

Let $I\subset\mathbb R$ be any interval of length $\delta$. If $u_1$ and $u_2$ belong to the same branch and both lie in $d_{i,k}^{-1}(I)$, then
\begin{equation*}
\begin{aligned}
c|u_2-u_1|
&\le
|d_{i,k}(u_2)-d_{i,k}(u_1)|\\
&\le
\delta.
\end{aligned}
\end{equation*}
Thus the inverse image of $I$ has Lebesgue measure at most $\delta/c$ on each branch and at most $2\delta/c$ across the two branches.

\noindent\textbf{Step 3. Absence of ties.}
Fix $j\ne i$. Replacing $Q_{-i,k}$ in Step~2 by the competing score $h_j$ and using Equation~\eqref{eq:other_score_lipschitz} shows that $u\mapsto h_i(u)-h_j(u)$ is strictly monotone with inverse-Lipschitz modulus $c$ on each branch of $h_i$. Hence its zero set $\mathcal Z_{ij}$ contains at most one point on each branch and therefore at most two points in total. Conditional on $\mathcal G_i$,
\begin{equation*}
\begin{aligned}
\mathbb P\!\left(
S_i^{(m)}=S_j^{(m)}
\,\middle|\,
\mathcal G_i
\right)
&=
\mathbb P\!\left(
h_i(\varepsilon_i)=h_j(\varepsilon_i)
\,\middle|\,
\mathcal G_i
\right)\\
&=
\mathbb P\!\left(
\varepsilon_i\in\mathcal Z_{ij}
\,\middle|\,
\mathcal G_i
\right)\\
&=
\int_{\mathcal Z_{ij}}
p_i(u\mid\mathcal G_i)\,du=
0.
\end{aligned}
\end{equation*}
The final equality holds because $\mathcal Z_{ij}$ is finite and hence has Lebesgue measure zero. Taking expectations over $\mathcal G_i$ proves the first line of \eqref{eq:regression_uniform_anti}; a finite union over pairs also gives simultaneous absence of ties almost surely.

\noindent\textbf{Step 4. Anti-concentration around the random order statistic.}
Fix $k\in\{1,\dots,m-1\}$. On the event $q_k<S_i^{(m)}$, the $i$th score lies strictly above the first $k$ scores, so removing it leaves the $k$th order statistic unchanged. Therefore $q_k=Q_{-i,k}(\varepsilon_i)$ on this event, and
\begin{equation*}
\begin{aligned}
\{q_k<S_i^{(m)}\le q_k+\delta\}
&=
\left\{
Q_{-i,k}(\varepsilon_i)
<
h_i(\varepsilon_i)
\le
Q_{-i,k}(\varepsilon_i)+\delta
\right\}\\
&=
\left\{
0<d_{i,k}(\varepsilon_i)\le\delta
\right\}.
\end{aligned}
\end{equation*}
By Step~2, the set of $u$ in the last event has total Lebesgue measure at most $2\delta/c$. The conditional density bound therefore yields
\begin{equation*}
\begin{aligned}
\mathbb P\!\left(
q_k<S_i^{(m)}\le q_k+\delta
\,\middle|\,
\mathcal G_i
\right)
&\;=
\int_{\{u:\,0<d_{i,k}(u)\le\delta\}}
p_i(u\mid\mathcal G_i)\,du\\
&\;\le
M_\varepsilon\,
\operatorname{Leb}
\{u:\,0<d_{i,k}(u)\le\delta\}\\
&\;\le
\frac{2M_\varepsilon}{c}\,\delta.
\end{aligned}
\end{equation*}
For $k=m$, the upper-side event is empty because $q_m$ is the maximum score.

For the lower side, fix $k\in\{2,\dots,m\}$. On the event $S_i^{(m)}<q_k$, removing the $i$th score shifts the relevant competing rank down by one, so $q_k=Q_{-i,k-1}(\varepsilon_i)$. With $d_{i,k-1}=h_i-Q_{-i,k-1}$,
\begin{equation*}
\begin{aligned}
\{q_k-\delta\le S_i^{(m)}<q_k\}
&=
\left\{
Q_{-i,k-1}(\varepsilon_i)-\delta
\le
h_i(\varepsilon_i)
<
Q_{-i,k-1}(\varepsilon_i)
\right\}\\
&=
\left\{
-\delta\le d_{i,k-1}(\varepsilon_i)<0
\right\}.
\end{aligned}
\end{equation*}
The same branchwise preimage calculation gives
\begin{equation*}
\begin{aligned}
&\mathbb P\!\left(
q_k-\delta\le S_i^{(m)}<q_k
\,\middle|\,
\mathcal G_i
\right)\\
&\;\le
M_\varepsilon\,
\operatorname{Leb}
\{u:\,-\delta\le d_{i,k-1}(u)<0\}\\
&\;\le
\frac{2M_\varepsilon}{c}\,\delta.
\end{aligned}
\end{equation*}
For $k=1$, the lower-side event is empty. Taking expectations over $\mathcal G_i$ proves the remaining two lines of \eqref{eq:regression_uniform_anti}.
\end{proof}

Proposition~\ref{prop:regression_score_regularity} treats the dependence between the fitted score and its random order-statistic cutoff directly. It neither replaces \(q_k\) by a fixed threshold nor infers pairwise no ties from marginal continuity. Its conclusion is uniform over the order-statistic rank, so one response-influence condition verifies Assumption~\ref{ass:anti}.

\noindent\textbf{Linear smoothers.}
Suppose that, conditionally on the covariates and auxiliary randomness, the fitted values have the affine form $\widehat Y=HY+\zeta$, where $H$ and $\zeta$ do not depend on $Y$. Varying only $Y_i$ gives $\widehat f_j^{(i,u)}-\widehat f_j^{(i,v)}=H_{ji}(u-v)$. Hence Proposition~\ref{prop:regression_score_regularity} applies with $\kappa_{\mathrm{self}}=\max_i|H_{ii}|$ and $\kappa_{\mathrm{cross}}=\max_{i\ne j}|H_{ji}|$ whenever their sum is smaller than one.

For ridge regression minimizing $m^{-1}\|Y-X\theta\|_2^2+\lambda\|\theta\|_2^2$, the first-order condition gives $\widehat\theta=(X^\top X+m\lambda I_d)^{-1}X^\top Y$ and therefore $H=X(X^\top X+m\lambda I_d)^{-1}X^\top$. If every design row satisfies $\|x_j\|_2\le B$, then
\begin{equation*}
\begin{aligned}
|H_{ji}|
&=
\left|
x_j^\top
(X^\top X+m\lambda I_d)^{-1}
x_i
\right|\\
&\le
\|x_j\|_2
\left\|
(X^\top X+m\lambda I_d)^{-1}
\right\|_{\mathrm{op}}
\|x_i\|_2\le
B\cdot\frac{1}{m\lambda}\cdot B=
\frac{B^2}{m\lambda}.
\end{aligned}
\end{equation*}
The second inequality uses $X^\top X+m\lambda I_d\succeq m\lambda I_d$. Thus $m\lambda>2B^2$ is a simple sufficient condition for $\kappa_{\mathrm{self}}+\kappa_{\mathrm{cross}}<1$, and one may take the common anti-concentration constant $2M_\varepsilon/\{1-2B^2/(m\lambda)\}$. The same argument applies to ridge regression on a bounded frozen feature representation.

\noindent\textbf{Smooth strongly convex estimators.}
For the one-coordinate response perturbation above, let $\widehat\theta(u)$ be the unique fitted parameter and let $\Psi_i(\theta,u)=0$ be its first-order condition. Any auxiliary random perturbation is held fixed as part of $\omega$. Suppose deterministic constants $\mu,B_\Psi,B_f>0$ satisfy, uniformly in the relevant arguments, $\langle\Psi_i(\theta,u)-\Psi_i(\theta',u),\theta-\theta'\rangle\ge\mu\|\theta-\theta'\|_2^2$, $\|\Psi_i(\theta,u)-\Psi_i(\theta,v)\|_2\le B_\Psi|u-v|/m$, and $|f_\theta(X_j)-f_{\theta'}(X_j)|\le B_f\|\theta-\theta'\|_2$ for any $j$.

The two first-order conditions and strong monotonicity give
\begin{equation*}
\begin{aligned}
\mu
\|\widehat\theta(u)-\widehat\theta(v)\|_2^2
&\le
\Big\langle
\Psi_i(\widehat\theta(u),u)
-
\Psi_i(\widehat\theta(v),u),
\widehat\theta(u)-\widehat\theta(v)
\Big\rangle\\
&=
\Big\langle
\Psi_i(\widehat\theta(v),v)
-
\Psi_i(\widehat\theta(v),u),
\widehat\theta(u)-\widehat\theta(v)
\Big\rangle\\
&\le
\left\|
\Psi_i(\widehat\theta(v),v)
-
\Psi_i(\widehat\theta(v),u)
\right\|_2
\|\widehat\theta(u)-\widehat\theta(v)\|_2\\
&\le
\frac{B_\Psi}{m}|u-v|
\|\widehat\theta(u)-\widehat\theta(v)\|_2.
\end{aligned}
\end{equation*}
The first inequality is strong monotonicity. The equality uses $\Psi_i(\widehat\theta(u),u)=0$ and $\Psi_i(\widehat\theta(v),v)=0$, the next inequality is Cauchy--Schwarz, and the last uses response sensitivity of the first-order map. If the two minimizers differ, canceling one norm gives the displayed bound below; if they coincide, the same bound is immediate. Consequently,
\begin{equation*}
\begin{aligned}
\|\widehat\theta(u)-\widehat\theta(v)\|_2
&\le
\frac{B_\Psi}{\mu m}|u-v|,\\
\left|
f_{\widehat\theta(u)}(X_j)
-
f_{\widehat\theta(v)}(X_j)
\right|
&\le
B_f\|\widehat\theta(u)-\widehat\theta(v)\|_2\le
\frac{B_fB_\Psi}{\mu m}|u-v|.
\end{aligned}
\end{equation*}
Writing $K=B_fB_\Psi/\mu$, the proposition therefore applies with $\kappa_{\mathrm{self}}=\kappa_{\mathrm{cross}}=K/m$. When $m>2K$, one may take the common anti-concentration constant $2M_\varepsilon/(1-2K/m)$. This covers regularized linear regression, smooth strongly convex generalized linear models, and smooth trainable heads on fixed representations whenever the stated strong-monotonicity, response-sensitivity, and prediction-Lipschitz bounds hold.

The proposition supplies sufficient conditions for score regularity; it does not assert that privacy alone implies no ties or anti-concentration. Outside the regression classes verified above, Assumptions~\ref{ass:noties} and \ref{ass:anti} remain primitive conditions of the generic results.

\subsection{Detailed discussion of mechanisms satisfying Assumption~\ref{ass:stab}}

We now verify Assumption~\ref{ass:stab} for the alternative private estimators mentioned in Remark~\ref{remark: dp mech stab}. The couplings below use the same base perturbation across adjacent sample sizes while preserving the correct marginal law of each release. They are analytical couplings, not a requirement that two separately released estimators literally reuse the same realized noise.

\noindent\textbf{Regularized ERM.}
For $m\ge1$, define $F_m(\theta)=m^{-1}\sum_{i=1}^m\ell(\theta;Z_i)+\lambda\|\theta\|_2^2/2$, where $\lambda>0$ is held fixed across sample sizes, and let $\widehat\theta_m$ be its unique minimizer over $\mathbb R^d$. Suppose that each $\ell(\cdot;z)$ is convex and that every subgradient $g\in\partial\ell(\theta;z)$ satisfies $\|g\|_2\le G$.

Choose $g_i\in\partial\ell(\widehat\theta_n;Z_i)$ so the first-order condition for the $n$-sample objective holds. Then
\begin{equation*}
\begin{aligned}
0
&=
\frac1n\sum_{i=1}^n g_i
+\lambda\widehat\theta_n,\\
\lambda\|\widehat\theta_n\|_2
&=
\left\|
\frac1n\sum_{i=1}^n g_i
\right\|_2\le
\frac1n\sum_{i=1}^n\|g_i\|_2\le
G.
\end{aligned}
\end{equation*}
The first equality is the convex optimality condition. The second follows by rearranging it, the first inequality is the triangle inequality, and the last uses the uniform subgradient bound.

Choose any $g_{n+1}\in\partial\ell(\widehat\theta_n;Z_{n+1})$. A subgradient of $F_{n+1}$ at the $n$-sample minimizer is
\begin{equation*}
\begin{aligned}
\zeta_{n+1}
&:=
\frac{1}{n+1}
\left(
\sum_{i=1}^n g_i+g_{n+1}
\right)
+\lambda\widehat\theta_n
=
\frac{1}{n+1}
\left(
-n\lambda\widehat\theta_n+g_{n+1}
\right)
+\lambda\widehat\theta_n\\
&=
\frac{\lambda\widehat\theta_n+g_{n+1}}{n+1}
\in
\partial F_{n+1}(\widehat\theta_n),\\
\|\zeta_{n+1}\|_2
&\le
\frac{
\lambda\|\widehat\theta_n\|_2+\|g_{n+1}\|_2
}{n+1}\le
\frac{2G}{n+1}.
\end{aligned}
\end{equation*}
The second equality substitutes $\sum_{i=1}^n g_i=-n\lambda\widehat\theta_n$, and the final inequality uses the preceding parameter-norm bound and $\|g_{n+1}\|_2\le G$.

The function $F_{n+1}$ is $\lambda$-strongly convex, so its subdifferential is $\lambda$-strongly monotone. Since $0\in\partial F_{n+1}(\widehat\theta_{n+1})$, we obtain
\begin{equation*}
\begin{aligned}
\lambda
\|\widehat\theta_n-\widehat\theta_{n+1}\|_2^2
&\le
\left\langle
\zeta_{n+1}-0,
\widehat\theta_n-\widehat\theta_{n+1}
\right\rangle\\
&\le
\|\zeta_{n+1}\|_2
\|\widehat\theta_n-\widehat\theta_{n+1}\|_2\\
&\le
\frac{2G}{n+1}
\|\widehat\theta_n-\widehat\theta_{n+1}\|_2.
\end{aligned}
\end{equation*}
The first inequality is strong monotonicity and the second is Cauchy--Schwarz. If the two minimizers differ, canceling their distance gives $\|\widehat\theta_{n+1}-\widehat\theta_n\|_2\le2G/\{\lambda(n+1)\}$; if they coincide, the same bound is immediate. Thus regularized ERM has an explicit $O(n^{-1})$ add-one stability bound without comparing the two objectives through an identity involving different sample normalizations.

\noindent\textbf{Output perturbation.}
Consider releases of the form $\theta_m^{\mathrm{out}}=\widehat\theta_m+\xi/m$ for $m\in\{n,n+1\}$, coupled through the same base noise vector $\xi$. The regularized-ERM bound gives
\begin{equation*}
\begin{aligned}
\|\theta_{n+1}^{\mathrm{out}}-\theta_n^{\mathrm{out}}\|_2
&=
\left\|
\widehat\theta_{n+1}-\widehat\theta_n
+
\left(
\frac1{n+1}-\frac1n
\right)\xi
\right\|_2\\
&\le
\|\widehat\theta_{n+1}-\widehat\theta_n\|_2
+
\left|
\frac1{n+1}-\frac1n
\right|
\|\xi\|_2\\
&\le
\frac{2G}{\lambda(n+1)}
+
\frac{\|\xi\|_2}{n(n+1)}.
\end{aligned}
\end{equation*}
The first inequality is the triangle inequality, and the last uses the preceding add-one bound and $|1/(n+1)-1/n|=1/\{n(n+1)\}$. Hence, for any deterministic $t_n>0$, Assumption~\ref{ass:stab} holds with
\begin{equation*}
\begin{aligned}
u_n
&=
\frac{2G}{\lambda(n+1)}
+
\frac{t_n}{n(n+1)},
&
\delta_n
&=
\mathbb P(\|\xi\|_2>t_n).
\end{aligned}
\end{equation*}

\noindent\textbf{Objective perturbation.}
Using a common base perturbation $b$, let
\begin{equation*}
\widehat\theta_m^b
\in
\arg\min_{\theta\in\mathbb R^d}
\left\{
F_m(\theta)
+
\left\langle
\frac{b}{m},\theta
\right\rangle
\right\},
\;
m\in\{n,n+1\}.
\end{equation*}
The linear perturbation does not change the strong-convexity parameter. At $\widehat\theta_n^b$, choose $g_i\in\partial\ell(\widehat\theta_n^b;Z_i)$ so that the perturbed $n$-sample optimality condition holds. Then
\begin{equation*}
\begin{aligned}
0
&=
\frac1n\sum_{i=1}^n g_i
+\lambda\widehat\theta_n^b
+\frac bn,\\
\lambda\|\widehat\theta_n^b\|_2
&=
\left\|
-\frac1n\sum_{i=1}^n g_i
-\frac bn
\right\|_2\\
&\le
\frac1n\sum_{i=1}^n\|g_i\|_2
+\frac{\|b\|_2}{n}\\
&\le
G+\frac{\|b\|_2}{n}.
\end{aligned}
\end{equation*}
Choose $g_{n+1}\in\partial\ell(\widehat\theta_n^b;Z_{n+1})$. At $\widehat\theta_n^b$, a subgradient of the perturbed $(n+1)$-sample objective is
\begin{equation*}
\begin{aligned}
\zeta_{n+1}^b
&:=
\frac{1}{n+1}
\left(
\sum_{i=1}^n g_i+g_{n+1}
\right)
+\lambda\widehat\theta_n^b
+\frac{b}{n+1}\\
&=
\frac{1}{n+1}
\left(
-n\lambda\widehat\theta_n^b-b+g_{n+1}
\right)
+\lambda\widehat\theta_n^b
+\frac{b}{n+1}\\
&=
\frac{\lambda\widehat\theta_n^b+g_{n+1}}{n+1},\\
\|\zeta_{n+1}^b\|_2
&\le
\frac{
\lambda\|\widehat\theta_n^b\|_2+\|g_{n+1}\|_2
}{n+1}\\
&\le
\frac{2G+\|b\|_2/n}{n+1}.
\end{aligned}
\end{equation*}
The second equality substitutes $\sum_{i=1}^n g_i=-n\lambda\widehat\theta_n^b-b$. The common perturbation cancels from the residual itself, while its norm remains in the bound on $\widehat\theta_n^b$.

Strong monotonicity of the perturbed $(n+1)$-sample objective and optimality of $\widehat\theta_{n+1}^b$ now give
\begin{equation*}
\begin{aligned}
\lambda
\|\widehat\theta_n^b-\widehat\theta_{n+1}^b\|_2^2
&\le
\left\langle
\zeta_{n+1}^b,
\widehat\theta_n^b-\widehat\theta_{n+1}^b
\right\rangle\\
&\le
\|\zeta_{n+1}^b\|_2
\|\widehat\theta_n^b-\widehat\theta_{n+1}^b\|_2\\
&\le
\frac{2G+\|b\|_2/n}{n+1}
\|\widehat\theta_n^b-\widehat\theta_{n+1}^b\|_2.
\end{aligned}
\end{equation*}
Canceling the distance when it is nonzero, and treating equality as immediate otherwise, yields
\begin{equation*}
\|\widehat\theta_{n+1}^b-\widehat\theta_n^b\|_2
\le
\frac{2G}{\lambda(n+1)}
+
\frac{\|b\|_2}{\lambda n(n+1)}.
\end{equation*}
Consequently, for any deterministic $t_n>0$, Assumption~\ref{ass:stab} holds with
\begin{equation*}
\begin{aligned}
u_n
&=
\frac{2G}{\lambda(n+1)}
+
\frac{t_n}{\lambda n(n+1)},
&
\delta_n
&=
\mathbb P(\|b\|_2>t_n).
\end{aligned}
\end{equation*}
This direct comparison removes the need for an unspecified fixed-offset add-one term.

\noindent\textbf{Privatized summaries.}
A related pattern arises when one first releases a privatized low-dimensional summary and then constructs the final estimator by deterministic post-processing. Let $T_n=T(D_n)\in\mathbb R^m$, and couple the releases by $\widetilde T_n=T_n+\xi/n$ and $\widetilde T_{n+1}=T_{n+1}+\xi/(n+1)$. Suppose $\theta_n^{\mathrm{sum}}=\Gamma(\widetilde T_n)$ and $\theta_{n+1}^{\mathrm{sum}}=\Gamma(\widetilde T_{n+1})$, where $\Gamma$ is $L_\Gamma$-Lipschitz. Then
\begin{equation*}
\begin{aligned}
\|\theta_{n+1}^{\mathrm{sum}}-\theta_n^{\mathrm{sum}}\|_2
&\le
L_\Gamma\|\widetilde T_{n+1}-\widetilde T_n\|_2\\
&\le
L_\Gamma\|T_{n+1}-T_n\|_2
+
L_\Gamma
\left|
\frac1{n+1}-\frac1n
\right|
\|\xi\|_2\\
&=
L_\Gamma\|T_{n+1}-T_n\|_2
+
\frac{L_\Gamma}{n(n+1)}\|\xi\|_2.
\end{aligned}
\end{equation*}

For the empirical-average summary $T_n=n^{-1}\sum_{i=1}^n\psi(Z_i)$ with $\|\psi(Z)\|_2\le B$ almost surely,
\begin{equation*}
\begin{aligned}
T_{n+1}-T_n
&=
\frac{1}{n+1}
\left
\{
\psi(Z_{n+1})-T_n
\right\},\\
\|T_{n+1}-T_n\|_2
&\le
\frac{
\|\psi(Z_{n+1})\|_2+\|T_n\|_2
}{n+1}\\
&\le
\frac{2B}{n+1}.
\end{aligned}
\end{equation*}
The last inequality uses $\|T_n\|_2\le n^{-1}\sum_{i=1}^n\|\psi(Z_i)\|_2\le B$. It follows that
\begin{equation*}
\|\theta_{n+1}^{\mathrm{sum}}-\theta_n^{\mathrm{sum}}\|_2
\le
\frac{2L_\Gamma B}{n+1}
+
\frac{L_\Gamma}{n(n+1)}\|\xi\|_2.
\end{equation*}
Thus, for any deterministic $t_n>0$, Assumption~\ref{ass:stab} holds with
\begin{equation*}
\begin{aligned}
u_n
&=
\frac{2L_\Gamma B}{n+1}
+
\frac{L_\Gamma t_n}{n(n+1)},
&
\delta_n
&=
\mathbb P(\|\xi\|_2>t_n).
\end{aligned}
\end{equation*}
This covers procedures that privatize a low-dimensional summary, such as a sufficient statistic, and then construct the final estimator by Lipschitz post-processing.

The three examples share the same structure: a deterministic $O(n^{-1})$ add-one effect plus a smaller term induced by the change from scale $1/n$ to $1/(n+1)$ in the shared-base perturbation. DP-SGD follows the same stability principle but requires a trajectory-level analysis because noise is injected throughout optimization; it is therefore treated separately in Section~\ref{subsec:dpsgd_instantiation}.

\section{ Proofs}\label{sec:supp-proofs}

In this section, we provide the complete formal proofs for the theoretical results presented in the main text. 

\subsection{Proof of Proposition~\ref{prop:private_search}}

\begin{proof}
Condition throughout on the score vector. Let $L_j$ and $U_j$ denote the left
and right endpoints after iteration $j$, with $L_0=a$ and $U_0=b$, and let
$t_j=(L_{j-1}+U_{j-1})/2$ be the midpoint queried at iteration $j$. Thus
$\widehat q=U_N$. Under the conditioning, the empirical count
$C_n(t)=\sum_{i=1}^n\mathbbm 1\{S_i\le t\}$ is deterministic, whereas the
Gaussian search noises retain their stated laws. Although each $t_j$ depends on
the preceding search history, the events below are defined directly through the
entire Gaussian noise vector and therefore remain valid for the adaptive search.

\noindent\textbf{Step 1. Target-rank bound.}
Set $u=\sigma\Phi^{-1}(1-\beta/N)=\tau+1$ and
$G_+=\{\max_{1\le j\le N}Z_j<u\}$. Conditional on the score vector,
\begin{equation*}
\begin{aligned}
\mathbb P\!\left(G_+^c\mid S_1,\dots,S_n\right)
&=
\mathbb P\!\left(
\bigcup_{j=1}^N\{Z_j\ge u\}
\,\middle|\,
S_1,\dots,S_n
\right)\\
&\le
\sum_{j=1}^N
\mathbb P\!\left(
Z_j\ge u
\,\middle|\,
S_1,\dots,S_n
\right)\\
&=
\sum_{j=1}^N
\left\{1-\Phi\!\left(\frac{u}{\sigma}\right)\right\}\\
&=
N\left[
1-\Phi\!\left\{
\Phi^{-1}\!\left(1-\frac{\beta}{N}\right)
\right\}
\right]\\
&=
\beta.
\end{aligned}
\end{equation*}
The first equality expands the complement of $G_+$, the inequality is the
union bound, and the second equality uses the independence of the search noises
from the score vector together with their common $N(0,\sigma^2)$ law. The final
two equalities use the definition of $u$.

Suppose that the algorithm lowers the right endpoint to $t_j$. On $G_+$, the
update condition gives
\begin{equation*}
\begin{aligned}
C_n(t_j)+Z_j\ge R
&=
k+\tau\\
&\Longrightarrow
C_n(t_j)
\ge
k+\tau-Z_j\\
&>
k+\tau-(\tau+1)\\
&=
k-1\\
&\Longrightarrow
C_n(t_j)\ge k.
\end{aligned}
\end{equation*}
The first implication rearranges the update inequality, the strict inequality
uses $Z_j<u=\tau+1$ on $G_+$, and the final implication uses that the empirical
count is integer-valued. Thus the subtraction of one in $\tau$ converts the
strict real-valued inequality into the desired integer rank.

At initialization, $C_n(U_0)=C_n(b)=n\ge k$. If
$C_n(U_{j-1})\ge k$, then the endpoint update and the preceding implication give
\begin{equation*}
\begin{aligned}
C_n(U_j)
&=
\begin{cases}
C_n(U_{j-1}),
& C_n(t_j)+Z_j<R,\\[2mm]
C_n(t_j),
& C_n(t_j)+Z_j\ge R
\end{cases}\\
&\ge
\begin{cases}
k,& C_n(t_j)+Z_j<R,\\[1mm]
k,& C_n(t_j)+Z_j\ge R
\end{cases}\\
&=
k
\;\text{on }G_+.
\end{aligned}
\end{equation*}
The first branch uses the induction hypothesis because the right endpoint is
unchanged; the second uses the rank implication above because the new right
endpoint is $t_j$. Induction over $j=1,\dots,N$ therefore yields
\begin{equation*}
\begin{aligned}
C_n(\widehat q)
&=
C_n(U_N)\\
&\ge
k\\
&\Longleftrightarrow
\#\{i:S_i\le\widehat q\}\ge k\\
&\Longleftrightarrow
\widehat q\ge S_{[k]}.
\end{aligned}
\end{equation*}
Thus $G_+\subseteq\{\widehat q\ge S_{[k]}\}$. Combining this inclusion with
the conditional bound for $G_+$ proves part~(i).

\noindent\textbf{Step 2. Bound with rank loss.}
Set $G_z=\{\max_{1\le j\le N}Z_j\le z\}$. Its conditional failure
probability satisfies
\begin{equation*}
\begin{aligned}
\mathbb P\!\left(G_z^c\mid S_1,\dots,S_n\right)
&=
\mathbb P\!\left(
\bigcup_{j=1}^N\{Z_j>z\}
\,\middle|\,
S_1,\dots,S_n
\right)\\
&\le
\sum_{j=1}^N
\mathbb P\!\left(
Z_j>z
\,\middle|\,
S_1,\dots,S_n
\right)\\
&=
\sum_{j=1}^N
\left\{
1-\Phi\!\left(\frac{z}{\sigma}\right)
\right\}\\
&=
N\left[
1-\Phi\!\left\{
\Phi^{-1}\!\left(1-\frac{\beta}{N}\right)
\right\}
\right]\\
&=
\beta.
\end{aligned}
\end{equation*}
The first equality expands the complement of the good event, the inequality is
the union bound, and the following equalities use the conditional Gaussian law
and the definition of $z$.

On $G_z$, every update of the right endpoint satisfies
\begin{equation*}
\begin{aligned}
C_n(t_j)+Z_j\ge R
&\Longrightarrow
C_n(t_j)\ge R-Z_j\\
&\ge
R-z\\
&\Longrightarrow
C_n(t_j)\ge\lceil R-z\rceil.
\end{aligned}
\end{equation*}
The second inequality uses $Z_j\le z$, and the last implication uses the
integrality of $C_n(t_j)$. Since $C_n(U_0)=n\ge\lceil R-z\rceil$, the endpoint
recursion is
\begin{equation*}
\begin{aligned}
C_n(U_j)
&=
\begin{cases}
C_n(U_{j-1}),
& C_n(t_j)+Z_j<R,\\[2mm]
C_n(t_j),
& C_n(t_j)+Z_j\ge R
\end{cases}\\
&\ge
\lceil R-z\rceil
\;\text{on }G_z.
\end{aligned}
\end{equation*}
The first branch preserves the induction hypothesis and the second follows from
the preceding update calculation. Consequently,
\begin{equation*}
\begin{aligned}
C_n(\widehat q)
&=
C_n(U_N)\\
&\ge
\lceil R-z\rceil\\
&\Longrightarrow
\widehat q\ge S_{[\lceil R-z\rceil]}.
\end{aligned}
\end{equation*}
Thus $G_z\subseteq\{\widehat q\ge S_{[\lceil R-z\rceil]}\}$. Together with the
conditional probability bound for $G_z$, this proves part~(ii).
\end{proof}

\subsection{Proof of Lemma~\ref{lemma:search_privacy}}

\begin{proof}
The proof relies on the composition property of GDP. Algorithm~\ref{alg:dp_binary_search} constitutes a sequential composition of $N$ adaptive queries, denoted as $M_1, \dots, M_N$. At each step $k$, the mechanism queries the empirical count $C_n(t_k) = \sum_{i=1}^n \mathbbm{1}(S_i \le t_k)$.
Since the addition or removal of a single data point changes the count by at most 1, the $\ell_2$-sensitivity of the query function is $\Delta_2 = \sup_{D \sim D'} |C_n(D) - C_n(D')| = 1$.

Each mechanism $M_k$ releases a noisy count by adding independent Gaussian noise $Z_k \sim \mathcal{N}(0, \sigma^2)$. Since the Gaussian mechanism with sensitivity $\Delta_2=1$ and noise scale $\sigma$ exactly satisfies $\mu_k$-GDP, where $\mu_k = \Delta_2/\sigma = 1/\sigma$.

Moreover, by sequential composition, we have 
\begin{equation*}
\mu_{\text{total}} = \sqrt{\sum_{k=1}^N \mu_k^2} = \sqrt{\sum_{k=1}^N \left(\frac{1}{\sigma}\right)^2} = \frac{\sqrt{N}}{\sigma}.
\end{equation*}

To ensure the entire procedure satisfies the target budget $\mu_{\text{calib}}$, we solve for the required noise scale $\sigma$:
\[
\mu_{\text{calib}} = \frac{\sqrt{N}}{\sigma} \implies \sigma = \frac{\sqrt{N}}{\mu_{\text{calib}}}.
\]
This concludes the proof.
\end{proof}

\subsection{Proof of Theorem~\ref{thm:overall_privacy}}

\begin{proof}
By assumption, \(M_{\mathrm{train}}\) satisfies \(\mu_{\mathrm{train}}\)-GDP and \(M_Q\) satisfies \(\mu_{\mathrm{calib}}\)-GDP conditional on the preceding training output. The sequential composition property of GDP therefore gives
\[
\mu_{\mathrm{total}}
=
\sqrt{\mu_{\mathrm{train}}^2+\mu_{\mathrm{calib}}^2}.
\]
The final prediction set is a post-processing of the trained model and private calibration output, so DP-SCP satisfies \(\mu_{\mathrm{total}}\)-GDP.
\end{proof}

\subsection{Proof of Theorem~\ref{thm:dp_gap_two_sides}}\label{app:proof_dp_gap_two_sides}
\begin{proof}
We prove the universal lower bound first and then construct a finite problem that approaches it.

\noindent\textbf{Step 1. Pointwise transfer through the tradeoff function.}
Fix a realization $d_{n+1}=(z_1,\dots,z_{n+1})$, write $d_n=(z_1,\dots,z_n)$, and let $z_{n+1}=(x_{n+1},y_{n+1})$. Define the test $\phi_{d_{n+1}}(w)=\mathbbm 1\{y_{n+1}\notin C(w,x_{n+1})\}$. Its Type-I error under the augmented dataset and its Type-II error under the reduced dataset are, respectively, $p(d_{n+1})=\mathbb E[\phi_{d_{n+1}}\{M(d_{n+1})\}]$ and $c(d_{n+1})=1-\mathbb E[\phi_{d_{n+1}}\{M(d_n)\}]$, where these expectations are over the mechanism randomness. Since $d_{n+1}$ and $d_n$ are add-or-delete adjacent,
\begin{equation*}
\begin{aligned}
c(d_{n+1})
&=
1-
\mathbb E\!\left[
\phi_{d_{n+1}}\{M(d_n)\}
\right]\\
&\ge
T\!\left(
\mathcal L\{M(d_{n+1})\},
\mathcal L\{M(d_n)\}
\right)
\!\left(p(d_{n+1})\right)\\
&\ge
f\!\left(p(d_{n+1})\right).
\end{aligned}
\end{equation*}
The first inequality follows because $\phi_{d_{n+1}}$ is a feasible test at Type-I error $p(d_{n+1})$, so its Type-II error cannot be smaller than the defining infimum of the tradeoff function. The second inequality is the $f$-DP guarantee for the adjacent datasets.

\noindent\textbf{Step 2. Averaging over the data.}
Taking expectations of the pointwise inequality from Step~1 and using the tower property gives
\begin{equation*}
\begin{aligned}
&\mathbb P\!\left\{
Y_{n+1}\in C(M(D_n),X_{n+1})
\right\}\\
&\;=
\mathbb E\!\left[c(D_{n+1})\right]\\
&\;\ge
\mathbb E\!\left[
f\{p(D_{n+1})\}
\right]\\
&\;\ge
f\!\left(
\mathbb E\!\left[p(D_{n+1})\right]
\right)\\
&\;=
f\!\left(
\mathbb P\!\left\{
Y_{n+1}\notin C(M(D_{n+1}),X_{n+1})
\right\}
\right)\\
&\;\ge
f(\alpha).
\end{aligned}
\end{equation*}
The second inequality is Jensen's inequality, using convexity of a tradeoff function. The following equality integrates the conditional augmented-sample miscoverage. The final inequality uses the assumed augmented-sample coverage and the fact that $f$ is non-increasing. This proves part~(i).

\noindent\textbf{Step 3. Finite hard instance and privacy of the noisy histogram.}
Choose an integer $K$ such that $n/K\le\gamma$, let $\mathcal X=\mathcal Y=[K]$, and define $P$ by $P\{(X,Y)=(j,j)\}=1/K$ for $j\in[K]$. Let $F$ be a continuous canonical noise distribution for the symmetric nontrivial tradeoff function $f$ \citep{awan2023canonical}, and let $N_{x,y}$, $(x,y)\in[K]^2$, be independent random variables with CDF $F$. For an arbitrary dataset $D$ on $[K]^2$, define
\begin{equation*}
\begin{aligned}
H_{x,y}(D)
&:=
\sum_{z\in D}
\mathbbm 1\{z=(x,y)\},
\; (x,y)\in[K]^2,\\
M(D)
&:=
\bigl(H_{x,y}(D)+N_{x,y}\bigr)_{(x,y)\in[K]^2}.
\end{aligned}
\end{equation*}
We verify the privacy of the complete vector release rather than invoking coordinatewise composition. Let $D$ and $D'$ be adjacent, and suppose first that $H_{x_0,y_0}(D')=H_{x_0,y_0}(D)+1$ for one coordinate $(x_0,y_0)$, while every other count is unchanged. Let $P_0$ and $P_1$ be the laws of the $(x_0,y_0)$th noisy coordinate under $D$ and $D'$, and let $R$ be the common joint law of all remaining noisy coordinates. Independence of the noises gives
\begin{equation*}
\mathcal L\{M(D)\}=P_0\otimes R,
\;
\mathcal L\{M(D')\}=P_1\otimes R.
\end{equation*}
For any measurable joint test $\phi(u,v)\in[0,1]$, define the averaged coordinate test $\overline\phi(u)=\int\phi(u,v)\,dR(v)$. Fubini's theorem gives
\begin{equation*}
\begin{aligned}
\mathbb E_{P_0}\!\left[\overline\phi(U)\right]
&=
\mathbb E_{P_0\otimes R}\!\left[\phi(U,V)\right],\\
1-
\mathbb E_{P_1}\!\left[\overline\phi(U)\right]
&=
1-
\mathbb E_{P_1\otimes R}\!\left[\phi(U,V)\right].
\end{aligned}
\end{equation*}
Thus every test of the joint release has a test of the changed coordinate with exactly the same Type-I and Type-II errors. Conversely, any coordinate test $\psi(u)$ lifts to the joint test $\phi(u,v)=\psi(u)$ with unchanged errors. Consequently, for every testing level $a\in[0,1]$,
\begin{equation*}
\begin{aligned}
T(P_0\otimes R,P_1\otimes R)(a)
&\ge
T(P_0,P_1)(a),\\
T(P_0\otimes R,P_1\otimes R)(a)
&\le
T(P_0,P_1)(a).
\end{aligned}
\end{equation*}
Hence $T(P_0\otimes R,P_1\otimes R)=T(P_0,P_1)$ pointwise on $[0,1]$.
The two scalar laws differ only by a unit translation. Translation invariance of the testing problem and the defining CND identity therefore yield
\begin{equation*}
\begin{aligned}
T\!\left(
\mathcal L\{M(D)\},
\mathcal L\{M(D')\}
\right)
&=
T(P_0\otimes R,P_1\otimes R)\\
&=
T(P_0,P_1)\\
&=
T\{F(\cdot),F(\cdot-1)\}\\
&=
f.
\end{aligned}
\end{equation*}
For the reverse ordering, the standard duality identity for tradeoff functions gives $T(P_1,P_0)=T(P_0,P_1)^{-1}=f^{-1}=f$, where the final equality uses symmetry of $f$. Applying the same common-factor identity therefore gives the identical tradeoff for the reversed vector experiment. The argument also covers deletion adjacency, and hence $M$ is $f$-DP.

Let $q_\alpha:=F^{-1}(1-\alpha)$ and set $t_\alpha:=1-q_\alpha$. Define the prediction-set map by $C_\alpha(w,x)=\{y\in[K]:w_{x,y}\ge t_\alpha\}$.

\noindent\textbf{Step 4. Coverage under the augmented dataset.}

Write $J=(X_{n+1},Y_{n+1})$. Since the augmented dataset contains its test pair, $H_J(D_{n+1})\ge1$. Conditional on $D_{n+1}$, the selected noise coordinate $N_J$ still has CDF $F$, and therefore
\begin{equation*}
\begin{aligned}
&\mathbb P\!\left\{
Y_{n+1}\in C_\alpha(M(D_{n+1}),X_{n+1})
\,\middle|\,
D_{n+1}
\right\}\\
&\;=
\mathbb P\!\left\{
H_J(D_{n+1})+N_J\ge t_\alpha
\,\middle|\,
D_{n+1}
\right\}\\
&\;\ge
\mathbb P\!\left\{
1+N_J\ge1-q_\alpha
\right\}\\
&\;=
\mathbb P\!\left\{
N_J\ge-q_\alpha
\right\}\\
&\;=
1-F(-q_\alpha)\\
&\;=
F(q_\alpha)\\
&\;=
1-\alpha.
\end{aligned}
\end{equation*}
The first inequality uses $H_J(D_{n+1})\ge1$. The penultimate equality uses symmetry of the CND, and the final equality uses continuity of $F$ and the definition of $q_\alpha$. Averaging over $D_{n+1}$ proves the first inequality in part~(ii).

\noindent\textbf{Step 5. Coverage under the reduced dataset.}
Let $B$ be the event that the test pair does not appear in $D_n$. Conditional on $J$, each training observation equals $J$ with probability $1/K$, independently across observations. Hence
\begin{equation*}
\begin{aligned}
\mathbb P(B^c\mid J)
&=
1-
\left(1-\frac1K\right)^n\\
&\le
\frac{n}{K}\\
&\le
\gamma,
\end{aligned}
\end{equation*}
where the first inequality is the union bound. On $B$, one has $H_J(D_n)=0$, while $N_J$ remains independent of the data and has CDF $F$. The conditional coverage is therefore
\begin{equation*}
\begin{aligned}
\mathbb P\!\left\{
Y_{n+1}\in C_\alpha(M(D_n),X_{n+1})
\,\middle|\,
B
\right\}
&=
\mathbb P\!\left\{
N_J\ge1-q_\alpha
\right\}\\
&=
1-F(1-q_\alpha)\\
&=
F(q_\alpha-1)\\
&=
f(\alpha).
\end{aligned}
\end{equation*}
The third equality again uses symmetry of $F$, and the final equality is the exact CND shift identity $f(\alpha)=F\{F^{-1}(1-\alpha)-1\}$. Decomposing according to $B$ and using that a coverage probability is at most one on $B^c$ gives
\begin{equation*}
\begin{aligned}
&\mathbb P\!\left\{
Y_{n+1}\in C_\alpha(M(D_n),X_{n+1})
\right\}\\
&\;=
\mathbb P(B)
\mathbb P\!\left\{
Y_{n+1}\in C_\alpha(M(D_n),X_{n+1})
\,\middle|\,
B
\right\}\\
&\;\;+
\mathbb P(B^c)
\mathbb P\!\left\{
Y_{n+1}\in C_\alpha(M(D_n),X_{n+1})
\,\middle|\,
B^c
\right\}\\
&\;\le
\mathbb P(B)f(\alpha)+\mathbb P(B^c)\\
&\;=
f(\alpha)+\{1-f(\alpha)\}\mathbb P(B^c)\\
&\;\le
f(\alpha)+\gamma.
\end{aligned}
\end{equation*}
This proves the second inequality in part~(ii) and completes the proof.
\end{proof}

\subsection{Proof of Corollary~\ref{cor:blackbox_dpscp}}

\begin{proof}
Under \(D_{n+1}\), the ideal scores are exchangeable. Since DP-SCP-F uses \(R=r+m_n+\tau\ge r+\tau\), the argument of Proposition~\ref{prop:private_search}(i) implies, conditional on the ideal score vector, that the private threshold is at least the \(r\)th order statistic of the first \(n\) ideal scores with probability at least \(1-\beta\). The exchangeable rank argument in Proposition~\ref{prop: canonical coverage} then gives
\[
\mathbb P\!\left(
Y_{n+1}\in C(M(D_{n+1}),X_{n+1})
\right)
\ge
(1-\beta)(1-\alpha)
=
1-\alpha_0.
\]
Since DP-SCP-F satisfies \(f\)-DP, applying Theorem~\ref{thm:dp_gap_two_sides}(i) with miscoverage level \(\alpha_0\) yields
\[
\mathbb P\!\left(
Y_{n+1}\in C(M(D_n),X_{n+1})
\right)
\ge
f(\alpha_0).
\]
\end{proof}

\subsection{Proof of Theorem~\ref{thm:coverage}}
\label{sub:proof_coverage}

\begin{proof}

Let \(V=S_{n+1}^{(n)}\) and \(W=S_{n+1}^{(n+1)}\). Let \(\mathcal K_n\) be the sigma-field generated by \(D_{n+1}\), the training randomness used to construct \(\theta_n\) and \(\theta_{n+1}\), and all other randomness fixed before the Gaussian noises for Algorithm~\ref{alg:dp_binary_search} are drawn. The actual and ideal scores and their order statistics are therefore \(\mathcal K_n\)-measurable, while the Gaussian noises are independent of \(\mathcal K_n\). 

Additionally, for \(k\in\{1,\dots,n\}\), let \(S_{[k]}^{(n)}\) and \(T_{[k]}^{(n+1)}\) denote the \(k\)th order statistics of the first \(n\) actual and ideal scores, respectively.

\noindent\textbf{Step 1. Compare actual and ideal scores.} We first show that, on the model-stability event, the actual test score and the order statistics of the actual training scores remain close to their ideal counterparts. We then use these inequalities for both DP-SCP-F and DP-SCP-A.

Let
\[
H=\left\{\|\theta_{n+1}-\theta_n\|_2\le u_n\right\}.
\]
On \(H\), Assumption~\ref{ass:lipschitz} gives, for every \(i\in\{1,\dots,n+1\}\),
\[
\left|S_i^{(n)}-S_i^{(n+1)}\right|
\le
L\|\theta_n-\theta_{n+1}\|_2
\le
Lu_n.
\]
Hence, for every \(k\in\{1,\dots,n\}\), \(V\le W+Lu_n\) and
\[
\begin{aligned}
S_{[k]}^{(n)}
&\ge
\left(S_1^{(n+1)}-Lu_n,\dots,S_n^{(n+1)}-Lu_n\right)_{[k]}\\
&=
T_{[k]}^{(n+1)}-Lu_n.
\end{aligned}
\]
For \(k\in\{1,\dots,n\}\), define
\[
B_k
=
\left\{
W\le T_{[k]}^{(n+1)}-2Lu_n
\right\}.
\]
The preceding inequalities imply
\[
H\cap B_k
\subseteq
\left\{
V
\le
W+Lu_n
\le
T_{[k]}^{(n+1)}-Lu_n
\le
S_{[k]}^{(n)}
\right\}.
\]

By exchangeability and Assumption~\ref{ass:noties}, the augmented rank of \(W\) is uniform on \(\{1,\dots,n+1\}\), so
\[
\mathbb P\!\left(
W\le T_{[k]}^{(n+1)}
\right)
=
\frac{k}{n+1}.
\]
Let
\[
E_k
=
\left\{
T_{[k]}^{(n+1)}-2Lu_n
<
W
\le
T_{[k]}^{(n+1)}
\right\}.
\]
On \(E_k\), no ties imply \(W<T_{[k]}^{(n+1)}\), so the augmented rank of \(W\) is at most \(k\). Removing \(W\) therefore shifts the \(k\)th order statistic of the first \(n\) ideal scores to augmented rank \(k+1\), giving
\[
T_{[k]}^{(n+1)}
=
q_{k+1}.
\]
Consequently, we have
\[
E_k
\subseteq
\left\{
q_{k+1}-2Lu_n
<
W
<
q_{k+1}
\right\}.
\]
If \(u_n=0\), then \(E_k\) is empty. Otherwise, Assumption~\ref{ass:anti} and \(2Lu_n\le\rho\) give
\[
\mathbb P(E_k)
\le
2\overline fLu_n.
\]
Therefore,
\[
\begin{aligned}
\mathbb P(B_k)
&=
\mathbb P\!\left(
W\le T_{[k]}^{(n+1)}
\right)
-
\mathbb P(E_k)\\
&\ge
\frac{k}{n+1}
-
2\overline fLu_n.
\end{aligned}
\]

\noindent\textbf{Step 2. Derive the DP-SCP-F coverage bound.}
We apply the inequalities from Step~1 at the buffered rank \(r+m_n\) and combine them with Proposition~\ref{prop:private_search}(i). This shows how the calibration buffer \(m_n\) enters positively in the finite-sample coverage bound.

Let \(k_F=r+m_n\) and
\[
G_F
=
\left\{
\widehat q_F\ge S_{[k_F]}^{(n)}
\right\}.
\]
Since DP-SCP-F uses comparison level \(k_F+\tau_n\), Proposition~\ref{prop:private_search}(i) gives
\[
\mathbb P(G_F\mid\mathcal K_n)
\ge
1-\beta_n
\]
almost surely. By Step~1,
\[
H\cap B_{k_F}
\subseteq
\left\{
V\le S_{[k_F]}^{(n)}
\right\},
\]
and hence
\[
\begin{aligned}
\mathbb P\!\left(V\le\widehat q_F\right)
&\ge
\mathbb P\!\left(G_F\cap H\cap B_{k_F}\right)\\
&=
\mathbb E\!\left[
\mathbbm 1\{H\cap B_{k_F}\}
\mathbb P(G_F\mid\mathcal K_n)
\right]\\
&\ge
(1-\beta_n)
\mathbb P(H\cap B_{k_F})\\
&\ge
(1-\beta_n)
\left\{
\mathbb P(B_{k_F})-\mathbb P(H^c)
\right\}\\
&\ge
(1-\beta_n)
\left\{
\frac{r+m_n}{n+1}
-
2\overline fLu_n
-
\delta_n
\right\}\\
&\ge
(1-\beta_n)
\left\{
1-\alpha
+
\frac{m_n}{n+1}
-
2\overline fLu_n
-
\delta_n
\right\},
\end{aligned}
\]
where the last inequality uses
\[
r
=
\left\lceil(1-\alpha)(n+1)\right\rceil
\ge
(1-\alpha)(n+1).
\]
This proves part~\textnormal{(i)}.

\noindent\textbf{Step 3. Derive the DP-SCP-A coverage bound.}
We apply the same inequalities at the retained rank \(k_n=\lceil r-z_n\rceil\) and combine them with Proposition~\ref{prop:private_search}(ii). The private-search rank loss \(z_n\) then appears directly in the coverage bound.

Let
\[
G_A
=
\left\{
\widehat q_A\ge S_{[k_n]}^{(n)}
\right\}.
\]
Since DP-SCP-A uses \(R=r\), Proposition~\ref{prop:private_search}(ii), with
\[
z_n
=
\sigma\Phi^{-1}\!\left(1-\frac{\beta_n}{N}\right),
\qquad
k_n
=
\left\lceil r-z_n\right\rceil,
\]
gives
\[
\mathbb P(G_A\mid\mathcal K_n)
\ge
1-\beta_n
\]
almost surely. Applying Step~1 with \(k=k_n\) gives
\[
H\cap B_{k_n}
\subseteq
\left\{
V\le S_{[k_n]}^{(n)}
\right\},
\]
so
\[
\begin{aligned}
\mathbb P\!\left(V\le\widehat q_A\right)
&\ge
\mathbb P\!\left(G_A\cap H\cap B_{k_n}\right)\\
&=
\mathbb E\!\left[
\mathbbm 1\{H\cap B_{k_n}\}
\mathbb P(G_A\mid\mathcal K_n)
\right]\\
&\ge
(1-\beta_n)
\mathbb P(H\cap B_{k_n})\\
&\ge
(1-\beta_n)
\left\{
\frac{k_n}{n+1}
-
2\overline fLu_n
-
\delta_n
\right\}.
\end{aligned}
\]
Finally,
\[
k_n
=
\left\lceil r-z_n\right\rceil
\ge
r-z_n
\ge
(1-\alpha)(n+1)-z_n,
\]
and therefore
\[
\begin{aligned}
\mathbb P\!\left(V\le\widehat q_A\right)
&\ge
(1-\beta_n)
\left\{
1-\alpha
-
\frac{z_n}{n+1}
-
2\overline fLu_n
-
\delta_n
\right\}.
\end{aligned}
\]
This proves part~\textnormal{(ii)}.

\end{proof}

\subsection{Proof of Lemma~\ref{lem:dpsgd_stability}}
\label{sub:proof_dpsgd_stability}

\begin{proof}
Fix the adjacent datasets and condition on the common initialization if it is
random. All probabilities below are therefore taken over the coupled Poisson
subsampling and Gaussian perturbations. We construct the two trajectories on a
common probability space and show that they coincide at every iteration until
the additional observation is first selected.

\noindent\textbf{Step 1. Coupling construction and its marginal laws.}
Write $[m]=\{1,\dots,m\}$. For every iteration $t=0,\dots,T-1$ and every
$i\in[n+1]$, let $b_{i,t}$ be independent Bernoulli$(q)$ variables, and let
$G_0,\dots,G_{T-1}$ be independent $\mathcal N(0,I_d)$ vectors, independent of all the
inclusion indicators. Define $B_t^{(n)}=\{i\in[n]:b_{i,t}=1\}$ and
$B_t^{(n+1)}=\{i\in[n+1]:b_{i,t}=1\}$. For every $A\subseteq[n]$ and
$A'\subseteq[n+1]$, the two batch marginals satisfy
\begin{equation*}
\begin{aligned}
\mathbb P\!\left(B_t^{(n)}=A\right)
&=
\mathbb P\!\left(
 b_{i,t}=1\ \text{for }i\in A,
\ b_{i,t}=0\ \text{for }i\in[n]\setminus A
\right)\\
&=
\prod_{i\in A}\mathbb P(b_{i,t}=1)
\prod_{i\in[n]\setminus A}\mathbb P(b_{i,t}=0)\\
&=
q^{|A|}(1-q)^{n-|A|},\\[2mm]
\mathbb P\!\left(B_t^{(n+1)}=A'\right)
&=
\mathbb P\!\left(
 b_{i,t}=1\ \text{for }i\in A',
\ b_{i,t}=0\ \text{for }i\in[n+1]\setminus A'
\right)\\
&=
\prod_{i\in A'}\mathbb P(b_{i,t}=1)
\prod_{i\in[n+1]\setminus A'}\mathbb P(b_{i,t}=0)\\
&=
q^{|A'|}(1-q)^{n+1-|A'|}.
\end{aligned}
\end{equation*}
For each marginal, the first equality rewrites the batch event in terms of its
inclusion indicators, the second uses their coordinatewise independence, and
the final equality uses the Bernoulli$(q)$ law. The indicators are also
independent across iterations, so the complete minibatch sequences have the
joint laws of ordinary Poisson subsampling on $D_n$ and $D_{n+1}$,
respectively. The calculation includes $A=\varnothing$ and
$A'=\varnothing$, and hence retains empty minibatches. The shared sequence
$(G_t)_{t=0}^{T-1}$ likewise has the correct Gaussian marginal law for each run.

\noindent\textbf{Step 2. Exact coincidence before selection of the additional observation.}
For $t=0,\dots,T$, define
$\mathcal N_t=\bigcap_{s=0}^{t-1}\{b_{n+1,s}=0\}$, with
$\mathcal N_0=\Omega$. We prove by induction that
\begin{equation*}
\mathcal N_t
\subseteq
\bigcap_{s=0}^{t}
\{\theta_n^{(s)}=\theta_{n+1}^{(s)}\}
\;
\text{for every }t=0,\dots,T.
\end{equation*}
The claim holds at $t=0$ because the runs have the same initialization. Suppose
it holds for some $t<T$. Since
$\mathcal N_{t+1}=\mathcal N_t\cap\{b_{n+1,t}=0\}$, on
$\mathcal N_{t+1}$ the induction hypothesis gives
$\theta_n^{(t)}=\theta_{n+1}^{(t)}$, while the coupling gives
\begin{equation*}
\begin{aligned}
B_t^{(n+1)}
&=
\{i\in[n+1]:b_{i,t}=1\}\\
&=
\{i\in[n]:b_{i,t}=1\}\\
&=
B_t^{(n)}.
\end{aligned}
\end{equation*}
Let $\bar g(\theta;z)$ denote the clipped gradient. The two projected updates
then obey the complete equality chain
\begin{equation*}
\begin{aligned}
\theta_{n+1}^{(t+1)}
&=
\Pi_\Theta\!\left[
\theta_{n+1}^{(t)}
-
\frac{\eta_t}{\bar b}
\left\{
\sum_{i\in B_t^{(n+1)}}
\bar g(\theta_{n+1}^{(t)};Z_i)
+
\sigma_{\mathrm{sgd}}C G_t
\right\}
\right]\\
&=
\Pi_\Theta\!\left[
\theta_n^{(t)}
-
\frac{\eta_t}{\bar b}
\left\{
\sum_{i\in B_t^{(n)}}
\bar g(\theta_n^{(t)};Z_i)
+
\sigma_{\mathrm{sgd}}C G_t
\right\}
\right]\\
&=
\theta_n^{(t+1)}.
\end{aligned}
\end{equation*}
The first equality is the update of the augmented run. The second uses the
induction hypothesis, the batch equality just established, the shared Gaussian
vector, and the common learning rate, clipping norm, noise multiplier, and
deterministic normalizer. The final equality is the update of the $n$-sample
run. If the common minibatch is empty, both clipped-gradient sums are exactly
zero, while the shared Gaussian term remains well defined because $\bar b>0$.
This proves the induction step and therefore
\begin{equation*}
\mathcal N_T
\subseteq
\bigcap_{t=0}^{T}
\{\theta_n^{(t)}=\theta_{n+1}^{(t)}\}.
\end{equation*}

\noindent\textbf{Step 3. Trajectory probability and expectation bounds.}
The pathwise conclusion in Step~2 gives the event inclusion and the resulting
probability calculation
\begin{equation*}
\begin{aligned}
\mathbb P\!\left(
\max_{0\le t\le T}
\|\theta_n^{(t)}-\theta_{n+1}^{(t)}\|_2>0
\right)
&\le
\mathbb P(\mathcal N_T^c)\\
&=
1-\mathbb P(\mathcal N_T)\\
&=
1-
\mathbb P\!\left(
\bigcap_{t=0}^{T-1}\{b_{n+1,t}=0\}
\right)\\
&=
1-
\prod_{t=0}^{T-1}
\mathbb P(b_{n+1,t}=0)\\
&=
1-(1-q)^T.
\end{aligned}
\end{equation*}
The first inequality uses the inclusion from Step~2, the next two equalities use
the definition of $\mathcal N_T$, the product equality uses independence across
iterations, and the last equality uses
$\mathbb P(b_{n+1,t}=0)=1-q$.

Every iterate belongs to $\Theta$ because the common initialization lies in
$\Theta$ and every update is projected onto $\Theta$. Hence the trajectory gap
is zero on $\mathcal N_T$ and is at most $\operatorname{diam}(\Theta)$ on its
complement. Pointwise,
\begin{equation*}
\max_{0\le t\le T}
\|\theta_n^{(t)}-\theta_{n+1}^{(t)}\|_2
\le
\operatorname{diam}(\Theta)\,
\mathbbm 1\{\mathcal N_T^c\}.
\end{equation*}
Taking expectations and substituting the probability already calculated gives
\begin{equation*}
\begin{aligned}
\mathbb E\!\left[
\max_{0\le t\le T}
\|\theta_n^{(t)}-\theta_{n+1}^{(t)}\|_2
\right]
&\le
\operatorname{diam}(\Theta)\,
\mathbb E\!\left[\mathbbm 1\{\mathcal N_T^c\}\right]\\
&=
\operatorname{diam}(\Theta)\,
\mathbb P(\mathcal N_T^c)\\
&=
\operatorname{diam}(\Theta)\{1-(1-q)^T\}.
\end{aligned}
\end{equation*}
The bounds hold conditionally on every fixed adjacent dataset pair and common
initialization. Taking expectations over any randomness in these objects gives
the unconditional statements in the lemma. This proves both conclusions.
\end{proof}

\subsection{Proof of Theorem~\ref{thm:dpsgd_stability_smooth}}
\label{sub:proof_smooth_stability}

\begin{proof}
Fix the adjacent datasets and condition on the common initialization if it is
random. We use the synchronized coupling constructed in the proof of
Lemma~\ref{lem:dpsgd_stability}; all expectations below are therefore over the
coupled Poisson subsampling and Gaussian perturbations.

\noindent\textbf{Step 1. Regularity of the clipped gradient.}
For $\mathbb{B}_2(C)=\{v\in\mathbb R^d:\|v\|_2\le C\}$, write the clipped
gradient as $\bar g(\theta;z)=\Pi_{\mathbb{B}_2(C)}\{\nabla\ell(\theta;z)\}$.
For every $\theta,\theta'\in\Theta$ and every $z$,
\begin{equation*}
\begin{aligned}
\|\bar g(\theta;z)\|_2
&=
\left\|
\Pi_{\mathbb{B}_2(C)}\{\nabla\ell(\theta;z)\}
\right\|_2
\le C,\\
\|\bar g(\theta;z)-\bar g(\theta';z)\|_2
&=
\left\|
\Pi_{\mathbb{B}_2(C)}\{\nabla\ell(\theta;z)\}
-
\Pi_{\mathbb{B}_2(C)}\{\nabla\ell(\theta';z)\}
\right\|_2\\
&\le
\|\nabla\ell(\theta;z)-\nabla\ell(\theta';z)\|_2\\
&\le
L_\ell\|\theta-\theta'\|_2.
\end{aligned}
\end{equation*}
The first inequality holds because the projection lies in the radius-$C$
ball. The second uses non-expansiveness of Euclidean projection onto a closed
convex set, and the last uses the assumed Lipschitz continuity of the loss
gradient. Thus clipping preserves the Lipschitz constant needed below while
uniformly bounding every individual gradient contribution by $C$.

\noindent\textbf{Step 2. One-step discrepancy.}
Let $\Delta_t=\theta_n^{(t)}-\theta_{n+1}^{(t)}$,
$K_t=\sum_{i=1}^n b_{i,t}$, and $A_t=b_{n+1,t}$, where the inclusion
indicators are those in the coupling from Lemma~\ref{lem:dpsgd_stability}.
The two projected updates are
\begin{equation*}
\begin{aligned}
\theta_n^{(t+1)}
&=
\Pi_\Theta\!\left[
\theta_n^{(t)}
-
\frac{\eta_t}{\bar b}
\left\{
\sum_{i=1}^n b_{i,t}\bar g(\theta_n^{(t)};Z_i)
+
\sigma_{\mathrm{sgd}}C G_t
\right\}
\right],\\
\theta_{n+1}^{(t+1)}
&=
\Pi_\Theta\!\left[
\theta_{n+1}^{(t)}
-
\frac{\eta_t}{\bar b}
\left\{
\sum_{i=1}^n b_{i,t}\bar g(\theta_{n+1}^{(t)};Z_i)
+
A_t\bar g(\theta_{n+1}^{(t)};Z_{n+1})
+
\sigma_{\mathrm{sgd}}C G_t
\right\}
\right].
\end{aligned}
\end{equation*}
Non-expansiveness of the parameter projection, followed by subtraction of the
two pre-projection arguments, gives the complete one-step bound
\begin{equation*}
\begin{aligned}
\|\Delta_{t+1}\|_2
&\le
\left\|
\begin{aligned}
&\Delta_t
-
\frac{\eta_t}{\bar b}
\sum_{i=1}^n b_{i,t}
\left\{
\bar g(\theta_n^{(t)};Z_i)
-
\bar g(\theta_{n+1}^{(t)};Z_i)
\right\}\\
&\;+
\frac{\eta_tA_t}{\bar b}
\bar g(\theta_{n+1}^{(t)};Z_{n+1})
\end{aligned}
\right\|_2\\
&\le
\|\Delta_t\|_2
+
\frac{\eta_t}{\bar b}
\sum_{i=1}^n b_{i,t}
\left\|
\bar g(\theta_n^{(t)};Z_i)
-
\bar g(\theta_{n+1}^{(t)};Z_i)
\right\|_2\\
&\;+
\frac{\eta_tA_t}{\bar b}
\left\|
\bar g(\theta_{n+1}^{(t)};Z_{n+1})
\right\|_2\\
&\le
\|\Delta_t\|_2
+
\frac{\eta_tL_\ell}{\bar b}
\sum_{i=1}^n b_{i,t}\|\Delta_t\|_2
+
\frac{\eta_tC}{\bar b}A_t\\
&=
\left(
1+\frac{\eta_tL_\ell K_t}{\bar b}
\right)
\|\Delta_t\|_2
+
\frac{\eta_tC}{\bar b}A_t.
\end{aligned}
\end{equation*}
The first inequality uses non-expansiveness of $\Pi_\Theta$; the shared
Gaussian perturbations cancel exactly in the difference of the two
pre-projection arguments. The second inequality is the triangle inequality,
the third applies both bounds from Step~1, and the final equality uses the
definition of $K_t$. Empty minibatches are already covered: $K_t=0$ when no
common observation is selected, and $A_t=0$ as well when the augmented
minibatch is empty.

\noindent\textbf{Step 3. A pathwise envelope for the complete trajectory.}
Define the nonnegative scalar process $\Gamma_0=0$ and
$\Gamma_{t+1}=(1+\eta_tL_\ell K_t/\bar b)\Gamma_t+
\eta_tCA_t/\bar b$. We first show by induction that
$\|\Delta_t\|_2\le\Gamma_t$ for every $t$. The claim holds at $t=0$ because
the coupled runs have the same initialization. If it holds at iteration $t$,
then Step~2 and the nonnegativity of the multiplier give
\begin{equation*}
\begin{aligned}
\|\Delta_{t+1}\|_2
&\le
\left(
1+\frac{\eta_tL_\ell K_t}{\bar b}
\right)
\|\Delta_t\|_2
+
\frac{\eta_tC}{\bar b}A_t\\
&\le
\left(
1+\frac{\eta_tL_\ell K_t}{\bar b}
\right)
\Gamma_t
+
\frac{\eta_tC}{\bar b}A_t\\
&=
\Gamma_{t+1}.
\end{aligned}
\end{equation*}
Moreover, the envelope is nondecreasing, because
\begin{equation*}
\begin{aligned}
\Gamma_{t+1}-\Gamma_t
&=
\frac{\eta_tL_\ell K_t}{\bar b}\Gamma_t
+
\frac{\eta_tC}{\bar b}A_t\\
&\ge 0.
\end{aligned}
\end{equation*}
Here $\eta_t$, $K_t$, $A_t$, and $\Gamma_t$ are nonnegative, while
$L_\ell>0$, $C>0$, and $\bar b>0$.
The induction and monotonicity therefore yield the pathwise domination
\begin{equation*}
\begin{aligned}
\max_{0\le t\le T}
\|\theta_n^{(t)}-\theta_{n+1}^{(t)}\|_2
&=
\max_{0\le t\le T}\|\Delta_t\|_2\\
&\le
\max_{0\le t\le T}\Gamma_t\\
&=
\Gamma_T.
\end{aligned}
\end{equation*}

\noindent\textbf{Step 4. Expected envelope and closed-form bounds.}
Let $\mathcal F_t$ contain the fixed datasets and common initialization, as
well as all inclusion indicators and Gaussian perturbations used before
iteration $t$. Then $\Gamma_t$ is $\mathcal F_t$-measurable, while the current
indicators are independent of $\mathcal F_t$ and satisfy
$\mathbb E(K_t\mid\mathcal F_t)=nq$ and
$\mathbb E(A_t\mid\mathcal F_t)=q$. Consequently,
\begin{equation*}
\begin{aligned}
\mathbb E(\Gamma_{t+1}\mid\mathcal F_t)
&=
\mathbb E\!\left[
\left(
1+\frac{\eta_tL_\ell K_t}{\bar b}
\right)\Gamma_t
+
\frac{\eta_tC}{\bar b}A_t
\,\middle|\,
\mathcal F_t
\right]\\
&=
\left(
1+
\frac{\eta_tL_\ell}{\bar b}
\mathbb E(K_t\mid\mathcal F_t)
\right)\Gamma_t
+
\frac{\eta_tC}{\bar b}
\mathbb E(A_t\mid\mathcal F_t)\\
&=
\left(
1+\frac{nqL_\ell\eta_t}{\bar b}
\right)\Gamma_t
+
\frac{qC\eta_t}{\bar b}.
\end{aligned}
\end{equation*}
Let $g_t=\mathbb E\Gamma_t$, $a_t=nqL_\ell\eta_t/\bar b$, and
$c=C/(nL_\ell)$. Taking expectations in the preceding display and using
$c a_t=qC\eta_t/\bar b$ gives the exact shifted recurrence
\begin{equation*}
\begin{aligned}
g_{t+1}
&=
(1+a_t)g_t+c a_t,\\
g_{t+1}+c
&=
(1+a_t)g_t+c a_t+c\\
&=
(1+a_t)(g_t+c).
\end{aligned}
\end{equation*}
Iterating this identity and using $g_0=\mathbb E\Gamma_0=0$ yields
\begin{equation*}
\begin{aligned}
g_T+c
&=
(1+a_{T-1})(g_{T-1}+c)\\
&=
(1+a_{T-1})(1+a_{T-2})(g_{T-2}+c)\\
&=
(g_0+c)\prod_{t=0}^{T-1}(1+a_t)\\
&=
c\prod_{t=0}^{T-1}(1+a_t).
\end{aligned}
\end{equation*}
Combining this identity with the pathwise envelope from Step~3 gives
\begin{equation*}
\begin{aligned}
\mathbb E\!\left[
\max_{0\le t\le T}
\|\theta_n^{(t)}-\theta_{n+1}^{(t)}\|_2
\right]
&\le
\mathbb E\Gamma_T\\
&=
g_T\\
&=
\frac{C}{nL_\ell}
\left[
\prod_{t=0}^{T-1}
\left(
1+\frac{nqL_\ell\eta_t}{\bar b}
\right)
-1
\right].
\end{aligned}
\end{equation*}
Finally, every factor is nonnegative and $1+x\le e^x$ for $x\ge0$, so
\begin{equation*}
\begin{aligned}
\prod_{t=0}^{T-1}
\left(
1+\frac{nqL_\ell\eta_t}{\bar b}
\right)
&\le
\prod_{t=0}^{T-1}
\exp\left\{
\frac{nqL_\ell\eta_t}{\bar b}
\right\}\\
&=
\exp\left\{
\frac{nqL_\ell}{\bar b}
\sum_{t=0}^{T-1}\eta_t
\right\}.
\end{aligned}
\end{equation*}
Substituting this bound into the preceding display proves both inequalities in Theorem~\ref{thm:dpsgd_stability_smooth}.
\end{proof}

\subsection{Proof of Corollary~\ref{cor:dpsgd_coverage}}

\begin{proof}
Let \(E_n\) denote the upper bound in Theorem~\ref{thm:dpsgd_stability_smooth}, so that
\[
\mathbb E\!\left[\max_{0\le t\le T}\|\theta_n^{(t)}-\theta_{n+1}^{(t)}\|_2\right]
\le E_n.
\]

\noindent\textbf{Step 1. Convert expected stability to high-probability stability.}
We first convert the expected DP-SGD trajectory bound into the high-probability model-stability condition required by Theorem~\ref{thm:coverage}. The choice \(u_n=\delta_n=E_n^{1/2}\) balances the size of the model discrepancy and its failure probability at the same \(O(n^{-1/2})\) order.

For both variants, set \(u_n=\delta_n=E_n^{1/2}\). By Markov's inequality,
\[
\mathbb P\!\left(
\max_{0\le t\le T}\|\theta_n^{(t)}-\theta_{n+1}^{(t)}\|_2>u_n
\right)
\le
\frac{E_n}{u_n}
=
E_n^{1/2}
=
\delta_n.
\]
Thus, Assumption~\ref{ass:stab} holds with these choices. Since \(E_n=O(n^{-1})\), we have \(u_n=\delta_n=O(n^{-1/2})\), and hence \(2Lu_n\le\rho\) for all sufficiently large \(n\).

\noindent\textbf{Step 2. Derive the DP-SCP-F coverage guarantee.}
We next choose the calibration buffer \(m_n\) to offset the coverage loss from the model discrepancy and the failure probability in private calibration. With this choice, Theorem~\ref{thm:coverage}(i) yields the nominal coverage level directly.

For DP-SCP-F, \(\beta_n=n^{-1}\) and
\[
m_n=
\left\lceil
(n+1)\left\{
(2\overline fL+1)E_n^{1/2}
+\frac{1-\alpha}{n-1}
\right\}
\right\rceil.
\]
Since \(E_n=O(n^{-1})\), we have \(m_n=O(\sqrt n)=o(n)\). Since \(r=\lceil(1-\alpha)(n+1)\rceil\), it follows that \(r+m_n\le n\) for all sufficiently large \(n\). Theorem~\ref{thm:coverage}(i) therefore gives
\[
\begin{aligned}
\mathbb P\!\left(S_{n+1}^{(n)}\le\widehat q_F\right)
&\ge
\left(1-\frac1n\right)
\left\{
1-\alpha+\frac{m_n}{n+1}
-2\overline fLE_n^{1/2}-E_n^{1/2}
\right\}\\
&\ge
\left(1-\frac1n\right)
\left\{
1-\alpha+\frac{1-\alpha}{n-1}
\right\}\\
&=
1-\alpha.
\end{aligned}
\]

\noindent\textbf{Step 3. Derive the DP-SCP-A coverage guarantee.}
Finally, DP-SCP-A uses no calibration buffer, so we bound the rank loss from private search and combine it with the same model-stability error from Step~1. The private-search term is smaller than \(n^{-1/2}\), leaving the model discrepancy as the leading term in the coverage bound.

For DP-SCP-A, \(\beta_n=n^{-1}\) and \(N=\lceil\log_2 n\rceil\). Recall that \(z_n=\sigma\Phi^{-1}(1-\beta_n/N)\), where \(\sigma=\sqrt N/\mu_{\mathrm{calib}}\). Since \(\mu_{\mathrm{calib}}\) is fixed, \(\sigma=O(\sqrt{\log n})\). Moreover, \(\beta_n/N\asymp1/(n\log n)\), so
\[
\Phi^{-1}(1-\beta_n/N)=O(\sqrt{\log n}),
\]
and therefore \(z_n=O(\log n)\). Hence
\[
\frac{z_n}{n+1}
=
O\!\left(\frac{\log n}{n}\right)
=
o(n^{-1/2}).
\]
Since \(r\asymp n\) and \(z_n=O(\log n)\), \(k_n=\lceil r-z_n\rceil\) satisfies \(1\le k_n\le n\) for all sufficiently large \(n\). Applying Theorem~\ref{thm:coverage}(ii),
\[
\begin{aligned}
\mathbb P\!\left(S_{n+1}^{(n)}\le\widehat q_A\right)
&\ge
\left(1-\frac1n\right)
\left\{
1-\alpha-\frac{z_n}{n+1}
-2\overline fLE_n^{1/2}-E_n^{1/2}
\right\}\\
&=
1-\alpha-O(n^{-1/2}),
\end{aligned}
\]
which proves the result.

\end{proof}

\section{Extended Numerical Studies}\label{supp:extended numerical studies}

This section gives the detailed privacy accounting, full allocation sweeps for the two real-data analyses, and additional synthetic experiments. 

\subsection{Details on Privacy Accounting}\label{supp:privacy_accounting}

This section provides a detailed description of the privacy accounting used in our experiments.
Throughout, the target privacy level is $(\varepsilon,\delta)$ with $\delta=10^{-5}$, and $\varepsilon\in\{0.5,1.0,2.0\}$.
The experimental pipeline consists of two private stages: (i) DP-SGD training (implemented in Opacus) and (ii) a private quantile routine based on $K$ noisy count queries (Algorithm~2), where $K=\texttt{QUANTILE\_STEPS}$.\\

\noindent\textbf{Stage 1: DP-SGD training and its RDP profile.}
Model training is performed using DP-SGD with gradient clipping and additive Gaussian noise, as implemented by Opacus.
In each run, Opacus maintains an \emph{RDP accountant} that records the sampling rate and noise multiplier used during training.
Concretely, Opacus stores an internal \emph{history} of tuples of the form
\[
(\sigma_{\mathrm{sgd}},\, q,\, T),
\]
where $\sigma_{\mathrm{sgd}}$ is the noise multiplier, $q$ is the sampling rate, and $T$ is the number of steps accumulated for that configuration.
Given a fixed set of RDP orders $\mathcal{A}$ (we use Opacus' default order grid), Opacus' analysis routines provide the corresponding order-wise RDP values for DP-SGD.
Denoting by $\mathrm{RDP}_{\mathrm{train}}(\alpha)$ the total training RDP at order $\alpha\in\mathcal{A}$, the reconstruction used in our code is
\[
\mathrm{RDP}_{\mathrm{train}}(\alpha)
=
\sum_{(\sigma_{\mathrm{sgd}},q,T)\in\mathrm{history}}
\mathrm{RDP}_{\mathrm{sgd}}(\alpha;\sigma_{\mathrm{sgd}},q,T),
\; \alpha\in\mathcal{A},
\]
where $\mathrm{RDP}_{\mathrm{sgd}}(\cdot)$ is computed by Opacus' internal function \texttt{compute\_rdp} for the corresponding subsampled Gaussian mechanism analysis.\\

\noindent\textbf{Stage 2: Private quantile via noisy counts.}
Let $S_1,\dots,S_n$ denote the conformity scores on the calibration set for the current run.
The private quantile routine performs $K$ adaptive binary-search steps, and at each step evaluates a (thresholded) count query $C(t) \;=\; \sum_{i=1}^n \mathbbm{1}\{S_i \le t\},$ releasing (internally) a noisy count $\tilde C(t)=C(t)+Z$ with $Z\sim\mathcal{N}(0,\sigma_q^2)$.
The $\ell_2$-sensitivity of $C(t)$ is $\Delta=1$, so a single noisy count query induces RDP
\[
\mathrm{RDP}_{\mathrm{Gauss}}(\alpha;\sigma_q)
=
\frac{\alpha \Delta^2}{2\sigma_q^2},
\; \alpha>1,
\]
and basic composition over $K$ queries yields
\[
\mathrm{RDP}_{\mathrm{qt}}(\alpha;\sigma_q)
=
K\cdot \frac{\alpha \Delta^2}{2\sigma_q^2},
\; \alpha\in\mathcal{A}.
\]
The variant-specific comparison levels $R_F=r+m_n+\tau_n$ and $R_A=r$ are fixed before the noisy queries. They do not change the sequence of private primitives and therefore do not alter the calibration accounting beyond the dependence on $\sigma_q$.\\

\noindent\textbf{RDP-to-$(\varepsilon,\delta)$ conversion (Opacus implementation).}
Given an order-wise RDP profile $\{\mathrm{RDP}(\alpha):\alpha\in\mathcal{A}\}$, we convert it to an $(\varepsilon,\delta)$ guarantee by optimizing over orders using Opacus' routine \texttt{get\_privacy\_spent}.
We denote the resulting value by
\[
\varepsilon(\delta)
=
\texttt{get\_privacy\_spent}\bigl(\{\mathrm{RDP}(\alpha)\}_{\alpha\in\mathcal{A}},\,\delta\bigr),
\]
which matches Opacus' own \texttt{get\_epsilon} outputs when applied to the same accountant state.\\

\noindent\textbf{Choosing the calibration noise level $\sigma_q^\star$.}
We select the per-query noise $\sigma_q$ to enforce the \emph{global} privacy target $(\varepsilon_\star,\delta_\star)$.
For DP-SCP (full reuse), the two stages compose sequentially in RDP: $\mathrm{RDP}_{\mathrm{tot}}(\alpha;\sigma_q)=\mathrm{RDP}_{\mathrm{train}}(\alpha)+\mathrm{RDP}_{\mathrm{qt}}(\alpha;\sigma_q),$ for $\alpha\in\mathcal{A}$.

We then compute $\varepsilon_{\mathrm{tot}}(\delta_\star;\sigma_q)
=
\texttt{get\_privacy\_spent}\bigl(\{\mathrm{RDP}_{\mathrm{tot}}(\alpha;\sigma_q)\}_{\alpha\in\mathcal{A}},\,\delta_\star\bigr),
$
and define $\sigma_q^\star$ as an approximately smallest (up to bisection tolerance) value such that $\varepsilon_{\mathrm{tot}}(\delta_\star;\sigma_q^\star)\le \varepsilon_\star$.
Since $\mathrm{RDP}_{\mathrm{qt}}(\alpha;\sigma_q)$ decreases pointwise in $\sigma_q$ for every $\alpha$, the map $\sigma_q\mapsto \varepsilon_{\mathrm{tot}}(\delta_\star;\sigma_q)$ is nonincreasing.

Accordingly, we compute $\sigma_q^\star$ numerically by (i) bracketing (geometrically increasing $\sigma_q$ until the inequality holds) and (ii) bisection on the bracketed interval for a fixed number of iterations.
If training alone already exceeds the target, i.e., $\varepsilon_{\mathrm{train}}(\delta):=\texttt{get\_privacy\_spent}\bigl(\{\mathrm{RDP}_{\mathrm{train}}(\alpha)\}_{\alpha\in\mathcal{A}},\,\delta\bigr)>\varepsilon,$
then no feasible $\sigma_q^\star$ exists and the run is declared infeasible.

For DP-Split (disjoint split), the calibration stage is accounted for in isolation by setting $\mathrm{RDP}_{\mathrm{train}}\equiv 0$ in the display above and choosing $\sigma_q^\star$ so that the quantile routine alone satisfies $(\varepsilon,\delta)$, while DP-SGD training separately targets $(\varepsilon,\delta)$ on the training split.
Under parallel composition over disjoint individual sets, the overall mechanism therefore satisfies $(\varepsilon,\delta)$.

\noindent\textbf{Budget allocation via $p$.}
For DP-SCP, we also study a privacy allocation parameter $p\in(0,1)$.
The DP-SGD training stage is targeted at $(\varepsilon_{\mathrm{train}},\delta)$ with $\varepsilon_{\mathrm{train}}=p\,\varepsilon$, and $\sigma_q^\star$ is then chosen (by the same procedure above) so the composed mechanism meets the global target $(\varepsilon,\delta)$. This produces an explicit empirical tradeoff between privacy spent in training and the calibration noise required at test time to certify the same overall privacy level.

\noindent\textbf{Implementation notes.}
Within each trial, input features and targets are standardised using statistics computed from the training split, and the same transformation is applied to calibration and test.
The reported privacy accounting covers the randomized DP-SGD training stage and the noisy-count primitives used in the private quantile routine, with all $(\varepsilon,\delta)$ conversions performed by Opacus' RDP conversion routines. In deployment, any data-dependent preprocessing statistics must be public or separately privacy-accounted.

\subsection{Additional Results on Real Data Analysis}
\label{sec:supp_realworld}

This section provides additional results for the two real-data analyses in Section~\ref{sec:numerical_study}. In addition to the main allocation \(p=0.5\), we report results for \(p\in\{0.3,0.5,0.7,0.9\}\), together with additional training and calibration results across the 30 repetitions. The main allocation \(p=0.5\) was fixed before the final experiments and is not selected from this sensitivity analysis.

\subsubsection{BloodMNIST}
\label{app:bloodmnist}

\begin{table}[!htbp]
\centering
\caption{Experimental configuration for the BloodMNIST analysis.}
\label{tab:supp_blood_setup}
\small\spacingset{1.1}
\renewcommand{\arraystretch}{1.08}
\begin{tabular}{p{0.24\linewidth}p{0.68\linewidth}}
\toprule
\textbf{Setting} & \textbf{Specification} \\
\midrule
Data split
& Original train, validation, and test samples recombined; fresh unstratified \(13{,}671/3{,}421\) fitting-pool/test split in each repetition \\

DP-Split sample
& Nested random half of the fitting pool for training \\

Representation
& Frozen ImageNet-pretrained ResNet-18, \(512\)-dimensional features \\

Feature preprocessing
& No data-dependent scaling \\

Model
& \(512\)-\(256\)-\(128\)-\(8\) MLP with ReLU activations \\

Nonprivate training
& AdamW, learning rate \(0.001\), batch size \(256\), \(1500\) steps, weight decay \(10^{-4}\) \\

Private training
& Opacus SGD, learning rate \(0.1\), batch size \(256\), \(1500\) steps, clipping norm \(1\), Poisson sampling \\

Conformal score
& \(s(x,y)=1-\hat p_y(x)\in[0,1]\) \\

Private search
& \(N=10\) noisy count queries with \(\beta=1/n_{\mathrm{scores}}\) \\

DP-SCP-A
& \(R_A=r\); final right endpoint \\

DP-SCP-F
& \(R_F=r+10+\tau_n\); final right endpoint \\

DP-Split
& Final midpoint \\

Privacy
& \(\varepsilon\in\{0.5,1,2\}\), \(\delta=10^{-5}\) \\

Allocation
& Main \(p=0.5\); sensitivity \(p\in\{0.3,0.5,0.7,0.9\}\) \\

Repetitions
& \(30\) \\
\bottomrule
\end{tabular}
\end{table}

Table~\ref{tab:supp_blood_setup} summarizes the experimental configuration for BloodMNIST. DP-SCP-A and DP-SCP-F use the same trained model and binary-search noise within each repetition, differing only in the comparison rank used during calibration.

\begin{table}[!ht]
\centering
\caption{BloodMNIST results across privacy allocations at \(\varepsilon=0.5\). Values are mean (SD) over 30 repetitions. The precommitted main allocation \(p=0.5\) is indicated in bold.}
\label{tab:supp_blood_eps05}
\small\spacingset{1.1}
\setlength{\tabcolsep}{5pt}
\renewcommand{\arraystretch}{1.06}
\begin{tabular}{llccc}
\toprule
\textbf{Method} & \(p\) & \textbf{Coverage} & \textbf{Avg. set size} \(\downarrow\) & \textbf{Singleton rate} \(\uparrow\) \\
\midrule
\multirow{4}{*}{DP-SCP-F}
& 0.3 & 1.0000 (0.0000) & 8.0000 (0.0000) & 0.0000 (0.0000) \\
& \textbf{0.5} & 0.9633 (0.0428) & 5.6967 (2.6800) & 0.1356 (0.1583) \\
& 0.7 & 0.9095 (0.0047) & 2.1257 (0.0628) & 0.4416 (0.0194) \\
& 0.9 & 0.9173 (0.0060) & 2.0093 (0.0647) & 0.4811 (0.0203) \\
\addlinespace
\multirow{4}{*}{DP-SCP-A}
& 0.3 & 1.0000 (0.0000) & 8.0000 (0.0000) & 0.0000 (0.0000) \\
& \textbf{0.5} & 0.9335 (0.0448) & 4.2045 (2.5295) & 0.2305 (0.1544) \\
& 0.7 & 0.8983 (0.0063) & 2.0039 (0.0585) & 0.4714 (0.0204) \\
& 0.9 & 0.8985 (0.0067) & 1.8254 (0.0596) & 0.5301 (0.0190) \\
\addlinespace
DP-Split & -- & 0.8959 (0.0107) & 2.6233 (0.1573) & 0.3153 (0.0253) \\
\bottomrule
\end{tabular}
\end{table}

\begin{table}[!ht]
\centering
\caption{BloodMNIST results across privacy allocations at \(\varepsilon=1\). Values are mean (SD) over 30 repetitions. The precommitted main allocation \(p=0.5\) is indicated in bold.}
\label{tab:supp_blood_eps1}
\small\spacingset{1.1}
\setlength{\tabcolsep}{5pt}
\renewcommand{\arraystretch}{1.06}
\begin{tabular}{llccc}
\toprule
\textbf{Method} & \(p\) & \textbf{Coverage} & \textbf{Avg. set size} \(\downarrow\) & \textbf{Singleton rate} \(\uparrow\) \\
\midrule
\multirow{4}{*}{DP-SCP-F}
& 0.3 & 0.9046 (0.0059) & 2.2791 (0.0792) & 0.3969 (0.0206) \\
& \textbf{0.5} & 0.9039 (0.0066) & 1.8166 (0.0457) & 0.5366 (0.0173) \\
& 0.7 & 0.9058 (0.0055) & 1.7620 (0.0346) & 0.5586 (0.0138) \\
& 0.9 & 0.9103 (0.0067) & 1.7907 (0.0349) & 0.5529 (0.0127) \\
\addlinespace
\multirow{4}{*}{DP-SCP-A}
& 0.3 & 0.8993 (0.0062) & 2.2182 (0.0849) & 0.4105 (0.0215) \\
& \textbf{0.5} & 0.8980 (0.0065) & 1.7650 (0.0452) & 0.5518 (0.0167) \\
& 0.7 & 0.8984 (0.0060) & 1.7002 (0.0328) & 0.5773 (0.0134) \\
& 0.9 & 0.8999 (0.0074) & 1.7003 (0.0382) & 0.5794 (0.0136) \\
\addlinespace
DP-Split & -- & 0.9005 (0.0079) & 1.8092 (0.0607) & 0.5371 (0.0204) \\
\bottomrule
\end{tabular}
\end{table}

\begin{table}[!ht]
\centering
\caption{BloodMNIST results across privacy allocations at \(\varepsilon=2\). Values are mean (SD) over 30 repetitions. The precommitted main allocation \(p=0.5\) is indicated in bold.}
\label{tab:supp_blood_eps2}
\small\spacingset{1.1}
\setlength{\tabcolsep}{5pt}
\renewcommand{\arraystretch}{1.06}
\begin{tabular}{llccc}
\toprule
\textbf{Method} & \(p\) & \textbf{Coverage} & \textbf{Avg. set size} \(\downarrow\) & \textbf{Singleton rate} \(\uparrow\) \\
\midrule
\multirow{4}{*}{DP-SCP-F}
& 0.3 & 0.9013 (0.0065) & 1.7445 (0.0362) & 0.5616 (0.0151) \\
& \textbf{0.5} & 0.9035 (0.0070) & 1.7275 (0.0320) & 0.5711 (0.0124) \\
& 0.7 & 0.9040 (0.0073) & 1.7404 (0.0322) & 0.5678 (0.0132) \\
& 0.9 & 0.9065 (0.0068) & 1.7665 (0.0307) & 0.5599 (0.0124) \\
\addlinespace
\multirow{4}{*}{DP-SCP-A}
& 0.3 & 0.8984 (0.0064) & 1.7196 (0.0362) & 0.5690 (0.0150) \\
& \textbf{0.5} & 0.9001 (0.0071) & 1.6985 (0.0318) & 0.5799 (0.0131) \\
& 0.7 & 0.9002 (0.0073) & 1.7087 (0.0295) & 0.5777 (0.0124) \\
& 0.9 & 0.9002 (0.0074) & 1.7133 (0.0312) & 0.5760 (0.0124) \\
\addlinespace
DP-Split & -- & 0.9013 (0.0074) & 1.7150 (0.0457) & 0.5761 (0.0150) \\
\bottomrule
\end{tabular}
\end{table}

Tables~\ref{tab:supp_blood_eps05}--\ref{tab:supp_blood_eps2} report the full allocation sensitivity for coverage, average set size, and singleton rate. At \(\varepsilon=0.5\), both variants return the full eight-class set at \(p=0.3\), while their average set sizes decrease sharply as more privacy is allocated to training. At \(\varepsilon=1\), most of the set-size improvement is realized by \(p=0.7\), while further allocation provides little additional gain for DP-SCP-A and makes DP-SCP-F more conservative. At \(\varepsilon=2\), the differences across \(p\) are small, and the main allocation \(p=0.5\) gives the smallest mean set size for both variants among the four allocations.

\begin{table}[!ht]
\centering
\caption{Predictive quality of DP training on BloodMNIST across privacy allocations. Values are test accuracy mean (SD) over 30 repetitions. DP-SCP uses the full fitting pool with the indicated allocation \(p\), whereas DP-Split uses half of the fitting pool with the full training privacy budget.}
\label{tab:supp_blood_training}
\small\spacingset{1.1}
\setlength{\tabcolsep}{5pt}
\renewcommand{\arraystretch}{1.06}
\begin{tabular}{lccccc}
\toprule
\textbf{Privacy} & \textbf{DP-Split} & \(p=0.3\) & \(\mathbf{p=0.5}\) & \(p=0.7\) & \(p=0.9\) \\
\midrule
\(\varepsilon=0.5\) & 0.6251 (0.0131) & 0.4715 (0.0243) & 0.6391 (0.0115) & 0.7016 (0.0089) & 0.7271 (0.0097) \\
\(\varepsilon=1\) & 0.7302 (0.0102) & 0.6757 (0.0105) & 0.7344 (0.0090) & 0.7431 (0.0098) & 0.7441 (0.0096) \\
\(\varepsilon=2\) & 0.7438 (0.0100) & 0.7405 (0.0098) & 0.7438 (0.0099) & 0.7413 (0.0100) & 0.7397 (0.0100) \\
\bottomrule
\end{tabular}
\end{table}

Table~\ref{tab:supp_blood_training} reports test accuracy across privacy allocations. At \(\varepsilon=0.5\), increasing \(p\) yields a pronounced improvement in predictive accuracy, while at \(\varepsilon=1\) most of the improvement is realized by \(p=0.7\). At \(\varepsilon=2\), test accuracy changes little across the four allocations and is highest at the main allocation \(p=0.5\) up to the reported precision relative to \(p=0.3\).

\begin{table}[!ht]
\centering
\caption{Calibration and endpoint diagnostics for BloodMNIST. Thresholds are mean (SD) over 30 repetitions. The calibration noise scale is fixed within each \((\varepsilon,p)\) configuration. The last two columns report the number of repetitions in which the final threshold equals the score upper bound \(1\).}
\label{tab:supp_blood_calibration}
\small\spacingset{1.1}
\setlength{\tabcolsep}{4pt}
\renewcommand{\arraystretch}{1.06}
\begin{tabular}{llccccc}
\toprule
\textbf{Privacy} & \(p\) & \(\boldsymbol{\sigma_q}\) & \textbf{A threshold} & \textbf{F threshold} & \textbf{A at 1} & \textbf{F at 1} \\
\midrule
\multirow{4}{*}{\(\varepsilon=0.5\)}
& 0.3 & 25.22 & 1.0000 (0.0000) & 1.0000 (0.0000) & 30/30 & 30/30 \\
& \textbf{0.5} & 28.39 & 0.9991 (0.0007) & 0.9996 (0.0005) & 9/30 & 17/30 \\
& 0.7 & 34.43 & 0.9842 (0.0036) & 0.9888 (0.0028) & 0/30 & 0/30 \\
& 0.9 & 53.62 & 0.9671 (0.0062) & 0.9799 (0.0039) & 0/30 & 0/30 \\
\midrule
\multirow{4}{*}{\(\varepsilon=1\)}
& 0.3 & 13.51 & 0.9938 (0.0021) & 0.9948 (0.0018) & 0/30 & 0/30 \\
& \textbf{0.5} & 15.00 & 0.9586 (0.0052) & 0.9639 (0.0050) & 0/30 & 0/30 \\
& 0.7 & 18.37 & 0.9436 (0.0051) & 0.9518 (0.0045) & 0/30 & 0/30 \\
& 0.9 & 29.55 & 0.9391 (0.0061) & 0.9511 (0.0044) & 0/30 & 0/30 \\
\midrule
\multirow{4}{*}{\(\varepsilon=2\)}
& 0.3 & 7.19 & 0.9492 (0.0053) & 0.9525 (0.0046) & 0/30 & 0/30 \\
& \textbf{0.5} & 7.98 & 0.9371 (0.0045) & 0.9414 (0.0044) & 0/30 & 0/30 \\
& 0.7 & 9.70 & 0.9352 (0.0040) & 0.9399 (0.0041) & 0/30 & 0/30 \\
& 0.9 & 15.86 & 0.9341 (0.0043) & 0.9418 (0.0039) & 0/30 & 0/30 \\
\bottomrule
\end{tabular}
\end{table}

Table~\ref{tab:supp_blood_calibration} reports the calibration noise, resulting thresholds, and frequency with which the search returns the upper score bound. As \(p\) increases, less privacy remains for calibration and the calibration noise increases. Nevertheless, at \(\varepsilon=0.5\), \(p=0.3\) produces the upper endpoint in all 30 repetitions for both variants, \(p=0.5\) produces it in 9 repetitions for DP-SCP-A and 17 for DP-SCP-F, and \(p=0.7\) and \(0.9\) avoid it entirely. Thus, the strong-privacy behavior also reflects the improvement in the trained classifier and the resulting score distribution rather than calibration noise alone.

\subsubsection{California Housing}
\label{app:housing}

\begin{table}[!htbp]
\centering
\caption{Experimental configuration for the California Housing analysis.}
\label{tab:supp_housing_setup}
\small\spacingset{1.1}
\renewcommand{\arraystretch}{1.08}
\begin{tabular}{p{0.24\linewidth}p{0.68\linewidth}}
\toprule
\textbf{Setting} & \textbf{Specification} \\
\midrule
Data
& \(20{,}640\) observations and eight covariates \\

Data split
& Fresh 80/20 fitting-pool/test split in each repetition \\

DP-Split sample
& Nested random half of the fitting pool for training \\

Feature preprocessing
& Clip to fixed public bounds and map coordinatewise to \([-1,1]\) \\

Lower bounds
& \([0,0,0,0,0,0,32,-125]\) \\

Upper bounds
& \([16,55,20,5,10000,20,42,-114]\) \\

Target
& Range \([0,5.00001]\); predictions projected to this range before scoring \\

Model
& \(8\)-\(128\)-\(128\)-\(64\)-\(1\) MLP with ReLU activations \\

Nonprivate training
& AdamW, learning rate \(0.001\), batch size \(256\), \(200\) epochs, weight decay \(10^{-4}\) \\

Private training
& Opacus SGD, learning rate \(0.01\), batch size \(256\), \(100\) epochs, clipping norm \(1\), Poisson sampling \\

Conformal score
& \(s(x,y)=|y-\hat f(x)|\) \\

Reported scale
& RMSE, residual scores, intervals, and widths on the original response scale \\

Private search
& \(N=10\) noisy count queries; DP-SCP-A/F use the final right endpoint and DP-Split uses the final midpoint \\

Privacy
& \(\varepsilon\in\{0.5,1,2\}\), \(\delta=10^{-5}\) \\

Allocation
& Main \(p=0.5\); sensitivity \(p\in\{0.3,0.5,0.7,0.9\}\) \\

Repetitions
& \(30\) \\
\bottomrule
\end{tabular}
\end{table}

Table~\ref{tab:supp_housing_setup} summarizes the experimental configuration for California Housing. All feature preprocessing is based on fixed public bounds, and RMSE, residual scores, prediction intervals, and interval widths are reported on the original response scale.

\begin{table}[!htbp]
\centering
\caption{California Housing results across privacy allocations at \(\varepsilon=0.5\). Values are mean (SD) over 30 repetitions. The precommitted main allocation \(p=0.5\) is indicated in bold.}
\label{tab:supp_housing_eps05}
\small\spacingset{1.1}
\setlength{\tabcolsep}{7pt}
\renewcommand{\arraystretch}{1.06}
\begin{tabular}{llcc}
\toprule
\textbf{Method} & \(p\) & \textbf{Coverage} & \textbf{Avg. width} \(\downarrow\) \\
\midrule
\multirow{4}{*}{DP-SCP-F}
& 0.3 & 0.9080 (0.0045) & 2.2005 (0.0643) \\
& \textbf{0.5} & 0.9094 (0.0044) & 2.1413 (0.0488) \\
& 0.7 & 0.9118 (0.0040) & 2.1563 (0.0509) \\
& 0.9 & 0.9160 (0.0048) & 2.2135 (0.0548) \\
\addlinespace
\multirow{4}{*}{DP-SCP-A}
& 0.3 & 0.9010 (0.0039) & 2.1198 (0.0614) \\
& \textbf{0.5} & 0.9012 (0.0048) & 2.0453 (0.0471) \\
& 0.7 & 0.9009 (0.0052) & 2.0313 (0.0460) \\
& 0.9 & 0.9014 (0.0054) & 2.0303 (0.0475) \\
\addlinespace
DP-Split & -- & 0.8996 (0.0053) & 2.1690 (0.0544) \\
\bottomrule
\end{tabular}
\end{table}

\begin{table}[!htbp]
\centering
\caption{California Housing results across privacy allocations at \(\varepsilon=1\). Values are mean (SD) over 30 repetitions. The precommitted main allocation \(p=0.5\) is indicated in bold.}
\label{tab:supp_housing_eps1}
\small\spacingset{1.1}
\setlength{\tabcolsep}{7pt}
\renewcommand{\arraystretch}{1.06}
\begin{tabular}{llcc}
\toprule
\textbf{Method} & \(p\) & \textbf{Coverage} & \textbf{Avg. width} \(\downarrow\) \\
\midrule
\multirow{4}{*}{DP-SCP-F}
& 0.3 & 0.9055 (0.0045) & 2.0846 (0.0400) \\
& \textbf{0.5} & 0.9057 (0.0044) & 2.0765 (0.0376) \\
& 0.7 & 0.9067 (0.0043) & 2.0846 (0.0362) \\
& 0.9 & 0.9101 (0.0046) & 2.1247 (0.0408) \\
\addlinespace
\multirow{4}{*}{DP-SCP-A}
& 0.3 & 0.9010 (0.0046) & 2.0358 (0.0390) \\
& \textbf{0.5} & 0.9010 (0.0047) & 2.0241 (0.0385) \\
& 0.7 & 0.9010 (0.0049) & 2.0195 (0.0370) \\
& 0.9 & 0.9013 (0.0049) & 2.0215 (0.0408) \\
\addlinespace
DP-Split & -- & 0.9004 (0.0049) & 2.1745 (0.0459) \\
\bottomrule
\end{tabular}
\end{table}

\begin{table}[!htbp]
\centering
\caption{California Housing results across privacy allocations at \(\varepsilon=2\). Values are mean (SD) over 30 repetitions. The precommitted main allocation \(p=0.5\) is indicated in bold.}
\label{tab:supp_housing_eps2}
\small\spacingset{1.1}
\setlength{\tabcolsep}{7pt}
\renewcommand{\arraystretch}{1.06}
\begin{tabular}{llcc}
\toprule
\textbf{Method} & \(p\) & \textbf{Coverage} & \textbf{Avg. width} \(\downarrow\) \\
\midrule
\multirow{4}{*}{DP-SCP-F}
& 0.3 & 0.9038 (0.0045) & 2.0524 (0.0366) \\
& \textbf{0.5} & 0.9041 (0.0042) & 2.0521 (0.0356) \\
& 0.7 & 0.9045 (0.0039) & 2.0573 (0.0356) \\
& 0.9 & 0.9064 (0.0041) & 2.0768 (0.0331) \\
\addlinespace
\multirow{4}{*}{DP-SCP-A}
& 0.3 & 0.9009 (0.0047) & 2.0215 (0.0340) \\
& \textbf{0.5} & 0.9012 (0.0040) & 2.0212 (0.0348) \\
& 0.7 & 0.9011 (0.0043) & 2.0186 (0.0358) \\
& 0.9 & 0.9012 (0.0043) & 2.0208 (0.0360) \\
\addlinespace
DP-Split & -- & 0.9002 (0.0049) & 2.1716 (0.0422) \\
\bottomrule
\end{tabular}
\end{table}

Tables~\ref{tab:supp_housing_eps05}--\ref{tab:supp_housing_eps2} report coverage and average interval width across the full allocation grid. For DP-SCP-A, moving from \(p=0.3\) to \(0.5\) produces the clearest reduction in interval width, whereas the widths change little beyond \(p=0.5\). DP-SCP-F instead becomes increasingly conservative at larger \(p\), particularly under stronger privacy. Across all four allocations and three privacy levels, DP-SCP-A remains narrower than DP-Split.

\begin{table}[!htbp]
\centering
\caption{Predictive quality of DP training on California Housing across privacy allocations. Values are test RMSE mean (SD) over 30 repetitions. DP-SCP uses the full fitting pool with the indicated allocation \(p\), whereas DP-Split uses half of the fitting pool with the full training privacy budget.}
\label{tab:supp_housing_training}
\small\spacingset{1.1}
\setlength{\tabcolsep}{5pt}
\renewcommand{\arraystretch}{1.06}
\begin{tabular}{lccccc}
\toprule
\textbf{Privacy} & \textbf{DP-Split} & \(p=0.3\) & \(\mathbf{p=0.5}\) & \(p=0.7\) & \(p=0.9\) \\
\midrule
\(\varepsilon=0.5\) & 0.7258 (0.0193) & 0.6974 (0.0195) & 0.6824 (0.0156) & 0.6805 (0.0157) & 0.6799 (0.0157) \\
\(\varepsilon=1\) & 0.7257 (0.0180) & 0.6812 (0.0156) & 0.6797 (0.0157) & 0.6789 (0.0151) & 0.6787 (0.0148) \\
\(\varepsilon=2\) & 0.7254 (0.0176) & 0.6792 (0.0154) & 0.6786 (0.0148) & 0.6785 (0.0146) & 0.6785 (0.0145) \\
\bottomrule
\end{tabular}
\end{table}

Table~\ref{tab:supp_housing_training} reports test RMSE across privacy allocations. The largest improvement occurs when \(p\) increases from \(0.3\) to \(0.5\), especially at \(\varepsilon=0.5\). Beyond \(p=0.5\), test RMSE improves little across all three privacy levels, showing diminishing predictive gains from further allocation to training.

\begin{table}[!htbp]
\centering
\caption{Calibration quantities for California Housing across privacy allocations. Thresholds are mean (SD) over 30 repetitions. The calibration noise scale and \(\tau_n\) are fixed within each \((\varepsilon,p)\) configuration.}
\label{tab:supp_housing_calibration}
\small\spacingset{1.1}
\setlength{\tabcolsep}{5pt}
\renewcommand{\arraystretch}{1.06}
\begin{tabular}{llcccc}
\toprule
\textbf{Privacy} & \(p\) & \(\boldsymbol{\sigma_q}\) & \(\boldsymbol{\tau_n}\) & \textbf{A threshold} & \textbf{F threshold} \\
\midrule
\multirow{4}{*}{\(\varepsilon=0.5\)}
& 0.3 & 25.20 & 109.24 & 1.0599 (0.0307) & 1.1003 (0.0321) \\
& \textbf{0.5} & 28.15 & 122.16 & 1.0226 (0.0235) & 1.0706 (0.0244) \\
& 0.7 & 34.41 & 149.57 & 1.0156 (0.0230) & 1.0781 (0.0255) \\
& 0.9 & 55.33 & 241.12 & 1.0151 (0.0238) & 1.1068 (0.0274) \\
\midrule
\multirow{4}{*}{\(\varepsilon=1\)}
& 0.3 & 13.51 & 58.12 & 1.0179 (0.0195) & 1.0423 (0.0200) \\
& \textbf{0.5} & 14.98 & 64.53 & 1.0120 (0.0192) & 1.0383 (0.0188) \\
& 0.7 & 18.41 & 79.54 & 1.0098 (0.0185) & 1.0423 (0.0181) \\
& 0.9 & 29.73 & 129.09 & 1.0107 (0.0204) & 1.0623 (0.0204) \\
\midrule
\multirow{4}{*}{\(\varepsilon=2\)}
& 0.3 & 7.20 & 30.50 & 1.0107 (0.0170) & 1.0262 (0.0183) \\
& \textbf{0.5} & 8.00 & 34.00 & 1.0106 (0.0174) & 1.0260 (0.0178) \\
& 0.7 & 9.78 & 41.80 & 1.0093 (0.0179) & 1.0286 (0.0178) \\
& 0.9 & 16.03 & 69.13 & 1.0104 (0.0180) & 1.0384 (0.0165) \\
\bottomrule
\end{tabular}
\end{table}

Table~\ref{tab:supp_housing_calibration} reports the calibration quantities across the four allocations. Increasing \(p\) leaves less privacy for calibration and therefore increases both \(\sigma_q\) and \(\tau_n\). Meanwhile, the DP-SCP-A threshold changes little beyond \(p=0.5\), whereas the DP-SCP-F threshold increases more visibly at larger \(p\), corresponding to the wider F intervals in Tables~\ref{tab:supp_housing_eps05}--\ref{tab:supp_housing_eps2}.

\begin{table}[!htbp]
\centering
\caption{Paired predictive and interval-width gains on California Housing at the main allocation \(p=0.5\). Differences and relative reductions are computed within repetition and reported as mean (SD). The RMSE difference is DP-Split minus DP-SCP, and the width difference is DP-Split minus DP-SCP-A.}
\label{tab:supp_housing_paired}
\small\spacingset{1.1}
\setlength{\tabcolsep}{3pt}
\renewcommand{\arraystretch}{1.06}
\begin{tabular}{lccccc}
\toprule
\textbf{Privacy}
& \textbf{RMSE diff.}
& \textbf{RMSE red.}
& \textbf{Width diff.}
& \textbf{Width red.}
& \textbf{Correlation} \\
\midrule
\(\varepsilon=0.5\)
& 0.0434 (0.0098)
& 5.96\% (1.26\%)
& 0.1237 (0.0516)
& 5.67\% (2.28\%)
& 0.531 \\

\(\varepsilon=1\)
& 0.0460 (0.0065)
& 6.33\% (0.82\%)
& 0.1504 (0.0366)
& 6.90\% (1.61\%)
& 0.502 \\

\(\varepsilon=2\)
& 0.0468 (0.0058)
& 6.44\% (0.71\%)
& 0.1504 (0.0217)
& 6.92\% (0.93\%)
& 0.317 \\
\bottomrule
\end{tabular}
\end{table}

Table~\ref{tab:supp_housing_paired} reports paired comparisons at the main allocation \(p=0.5\). Full-data training reduces RMSE by approximately \(6\%\) across all three privacy levels, while DP-SCP-A reduces interval width relative to DP-Split by \(5.67\%\), \(6.90\%\), and \(6.92\%\). The two gains are also positively associated across repetitions, with correlations \(0.531\), \(0.502\), and \(0.317\) as \(\varepsilon\) increases.

\subsection{Additional Synthetic Data Experiments}
\label{app:synthetic_experiments}
To complement the real-data studies and to validate the qualitative implications of our theory under controlled conditions, we conduct two synthetic experiments. 

\subsubsection{Experiment I: Stability vs.\ Estimation accuracy}
\label{app:exp1_stability_optimality}

This experiment visualizes the central dichotomy behind our analysis: privacy noise can degrade estimation accuracy, while the \emph{difference} between two DP-SGD trajectories trained on adjacent datasets can remain small.

We simulate binary logistic regression with $N=1000$ samples and $d=10$ features. Covariates are drawn i.i.d.\ as $X_i \sim \mathcal{N}(0, I_d), i=1,\dots,N,$ and labels are generated according to
\[
Y_i \mid X_i \sim \mathrm{Bernoulli}\!\left(\sigma(X_i^\top \theta^\star)\right),
\; \text{where}\; \sigma(u) = \frac{1}{1+e^{-u}}.
\]
We set the true parameter $\theta^\star \in \mathbb{R}^d$ to have alternating signs and decaying magnitudes:
\[
\theta^\star_j = \frac{(-1)^j}{j+1}, \; j=0,1,\dots,d-1.
\]

We train two DP-SGD runs on $D_n$ and $D_{n+1}$ under a \emph{synchronized coupling}: both runs use the same initialization, the same Poisson subsampling/masking sequence, and the same Gaussian noises. Under this coupling, the only source of trajectory divergence is the presence or absence of the extra point $z$ within minibatches.

We fix $\delta=10^{-5}$ and consider $\varepsilon \in \{0.5,1.0,2.0\}$. Across iterations $t=1,\dots,T$, we track
\begin{equation*}
\begin{aligned}
\text{Estimation error: } &\|\theta_t - \theta^\star\|_2,
\\
\text{Stability gap: } &\|\theta_t(D_n) - \theta_t(D_{n+1})\|_2.
\end{aligned}
\end{equation*}

Figure~\ref{fig:exp1} demonstrates a sharp separation between optimality and stability. As $\varepsilon$ increases (weaker privacy noise), the estimation error curve improves, whereas the stability gap remains essentially near zero across all privacy regimes. This empirically supports the perspective adopted in our theory that even when privacy noise prevents convergence to the exact optimum, a shared-randomness coupling yields strong algorithmic stability under add/delete adjacency, which is the quantity directly controlling the validity gap in our full-data conformal construction.

\begin{figure}[!htpb]
  \centering
  \includegraphics[width=0.7\textwidth]{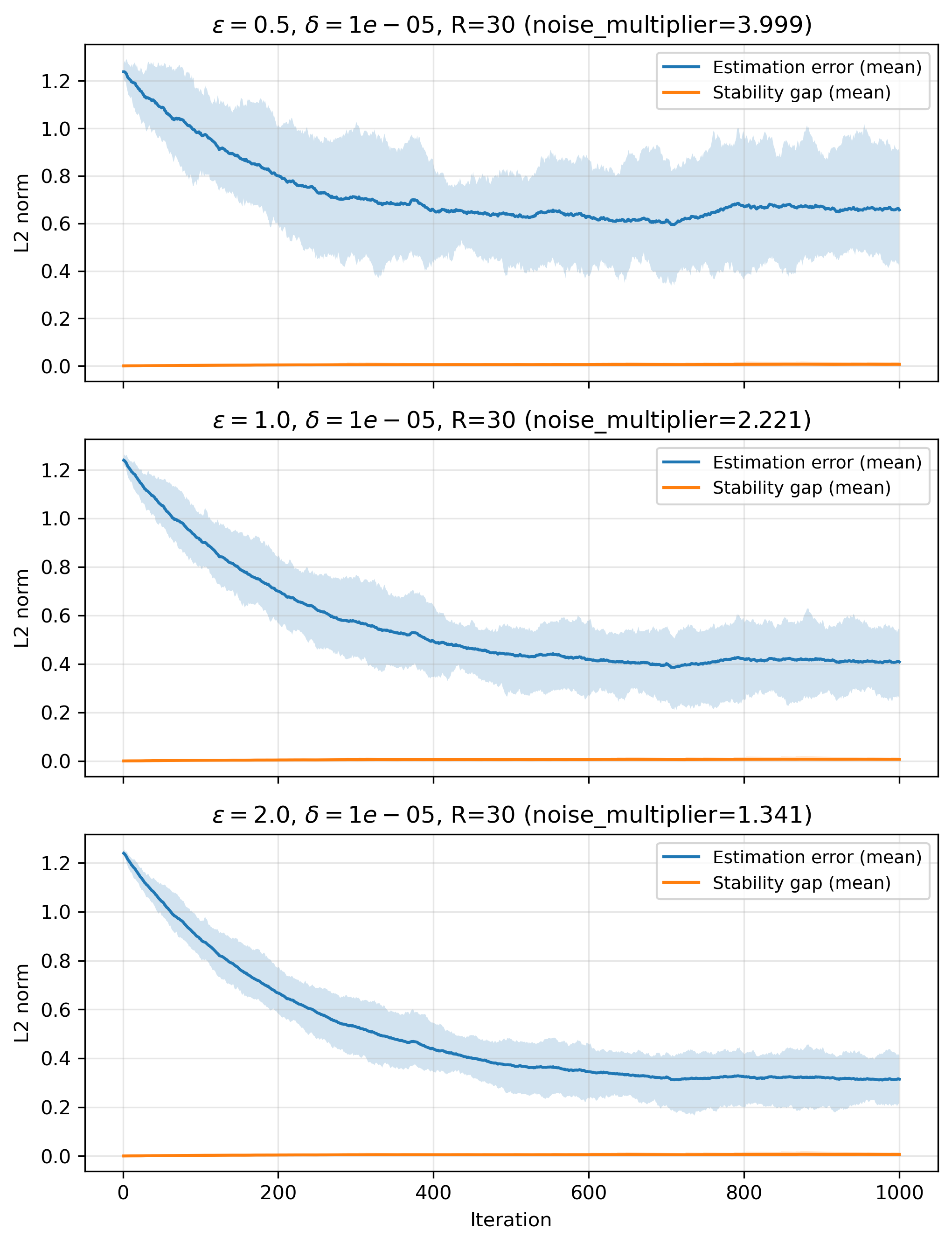}
  \caption{Trajectory stability vs.\ estimation error under synchronized coupling. Each panel corresponds to $\varepsilon\in\{0.5,1,2\}$ (with $\delta=10^{-5}$), reporting the mean over $R=30$ runs with a shaded uncertainty band.}
  \label{fig:exp1}
\end{figure}

\subsubsection{Experiment II: Sample-Size Scaling}
\label{app:exp2_n_scaling}

This experiment investigates how prediction-set quality evolves with the sample size $n$.

We generate synthetic multi-class classification data with $K=5$ classes and $d=10$ informative features using \texttt{sklearn.make\_classification}. To create a clean yet non-trivial task, we set $\texttt{class\_sep}=0.6$ and $\texttt{flip\_y}=0.01$. We vary the total sample size $n\in\{10000,15000,20000,25000,30000\}$ and consider privacy budgets $\varepsilon\in\{0.5,1.0,2.0\}$ with fixed $\delta=10^{-5}$. For each $(n,\varepsilon)$, we draw $n$ examples from a fixed pool, with an independent test set of size $2000$, train a two-layer MLP with widths $(16,16)$ and ReLU activations via DP-SGD, and evaluate DP-Split, DP-SCP-F, and DP-SCP-A. We use batch size $32$, $50$ epochs, learning rate $10^{-3}$, and max grad norm $1$. Within each trial, features are standardized using statistics computed from the selected training split, and the same transformation is applied to calibration and test. For DP-SCP methods, we use a training privacy allocation $p=1/2$. We report Coverage, average set size (average set size), and Singleton rate (singleton rate) as the mean and standard deviation over $30$ independent trials. For coverage, each trial computes the empirical coverage over the full test set.

\begin{table}[!htbp]
\centering
\caption{Performance comparison for privacy budget $\varepsilon = 0.5$. We report the mean and standard deviation over $30$ independent trials.}
\label{tab:results_eps_0.5}
\begin{tabular}{llccc}
\toprule
Sample Size ($n$) & Method & Coverage & Avg. set size & Singleton rate \\
\midrule
\multirow{3}{*}{10000}
 & DP-SCP-F & 0.927 (0.009) & 3.243 (0.126) & 0.114 (0.019) \\
 & DP-SCP-A & 0.902 (0.007) & \textbf{2.976} (0.123) & \textbf{0.149} (0.021) \\
 & DP-Split & 0.904 (0.011) & 3.339 (0.139) & 0.060 (0.024) \\
\addlinespace
\multirow{3}{*}{15000}
 & DP-SCP-F & 0.920 (0.006) & 2.824 (0.116) & 0.186 (0.023) \\
 & DP-SCP-A & 0.905 (0.006) & \textbf{2.662} (0.105) & \textbf{0.214} (0.023) \\
 & DP-Split & 0.908 (0.007) & 3.064 (0.124) & 0.120 (0.020) \\
\addlinespace
\multirow{3}{*}{20000}
 & DP-SCP-F & 0.918 (0.007) & 2.586 (0.076) & 0.233 (0.021) \\
 & DP-SCP-A & 0.906 (0.007) & \textbf{2.470} (0.063) & \textbf{0.258} (0.020) \\
 & DP-Split & 0.907 (0.006) & 2.857 (0.111) & 0.161 (0.020) \\
\addlinespace
\multirow{3}{*}{25000}
 & DP-SCP-F & 0.914 (0.006) & 2.393 (0.071) & 0.279 (0.018) \\
 & DP-SCP-A & 0.904 (0.006) & \textbf{2.298} (0.073) & \textbf{0.302} (0.021) \\
 & DP-Split & 0.908 (0.006) & 2.729 (0.115) & 0.190 (0.024) \\
\addlinespace
\multirow{3}{*}{30000}
 & DP-SCP-F & 0.913 (0.005) & 2.279 (0.073) & 0.311 (0.022) \\
 & DP-SCP-A & 0.905 (0.005) & \textbf{2.202} (0.071) & \textbf{0.332} (0.024) \\
 & DP-Split & 0.907 (0.006) & 2.537 (0.107) & 0.235 (0.023) \\
\bottomrule
\end{tabular}
\end{table}

\begin{table}[!htbp]
\centering
\caption{Performance comparison for privacy budget $\varepsilon = 1.0$. We report the mean and standard deviation over $30$ independent trials.}
\label{tab:results_eps_1.0}
\begin{tabular}{llccc}
\toprule
Sample Size ($n$) & Method & Coverage & Avg. set size & Singleton rate \\
\midrule
\multirow{3}{*}{10000}
 & DP-SCP-F & 0.917 (0.006) & 2.958 (0.122) & 0.148 (0.019) \\
 & DP-SCP-A & 0.905 (0.005) & \textbf{2.828} (0.125) & \textbf{0.168} (0.020) \\
 & DP-Split & 0.904 (0.007) & 3.288 (0.107) & 0.063 (0.021) \\
\addlinespace
\multirow{3}{*}{15000}
 & DP-SCP-F & 0.913 (0.006) & 2.599 (0.104) & 0.222 (0.023) \\
 & DP-SCP-A & 0.905 (0.005) & \textbf{2.519} (0.103) & \textbf{0.238} (0.024) \\
 & DP-Split & 0.908 (0.005) & 3.016 (0.104) & 0.123 (0.023) \\
\addlinespace
\multirow{3}{*}{20000}
 & DP-SCP-F & 0.912 (0.006) & 2.395 (0.071) & 0.273 (0.022) \\
 & DP-SCP-A & 0.905 (0.006) & \textbf{2.336} (0.065) & \textbf{0.286} (0.022) \\
 & DP-Split & 0.907 (0.005) & 2.813 (0.106) & 0.168 (0.021) \\
\addlinespace
\multirow{3}{*}{25000}
 & DP-SCP-F & 0.910 (0.006) & 2.240 (0.064) & 0.318 (0.020) \\
 & DP-SCP-A & 0.904 (0.006) & \textbf{2.192} (0.063) & \textbf{0.331} (0.020) \\
 & DP-Split & 0.908 (0.005) & 2.681 (0.118) & 0.197 (0.024) \\
\addlinespace
\multirow{3}{*}{30000}
 & DP-SCP-F & 0.908 (0.005) & 2.129 (0.053) & 0.353 (0.019) \\
 & DP-SCP-A & 0.904 (0.005) & \textbf{2.091} (0.053) & \textbf{0.364} (0.019) \\
 & DP-Split & 0.907 (0.005) & 2.496 (0.086) & 0.239 (0.022) \\
\bottomrule
\end{tabular}
\end{table}

\begin{table}[H]
\centering
\caption{Performance comparison for privacy budget $\varepsilon = 2.0$. We report the mean and standard deviation over $30$ independent trials.}
\label{tab:results_eps_2.0}
\begin{tabular}{llccc}
\toprule
Sample Size ($n$) & Method & Coverage & Avg. set size & Singleton rate \\
\midrule
\multirow{3}{*}{10000}
 & DP-SCP-F & 0.912 (0.006) & 2.854 (0.111) & 0.161 (0.023) \\
 & DP-SCP-A & 0.905 (0.006) & \textbf{2.778} (0.104) & \textbf{0.174} (0.022) \\
 & DP-Split & 0.904 (0.006) & 3.266 (0.096) & 0.066 (0.020) \\
\addlinespace
\multirow{3}{*}{15000}
 & DP-SCP-F & 0.909 (0.006) & 2.517 (0.086) & 0.235 (0.022) \\
 & DP-SCP-A & 0.904 (0.007) & \textbf{2.469} (0.083) & \textbf{0.246} (0.022) \\
 & DP-Split & 0.909 (0.005) & 3.014 (0.096) & 0.122 (0.022) \\
\addlinespace
\multirow{3}{*}{20000}
 & DP-SCP-F & 0.908 (0.005) & 2.342 (0.072) & 0.284 (0.023) \\
 & DP-SCP-A & 0.905 (0.006) & \textbf{2.310} (0.071) & \textbf{0.292} (0.023) \\
 & DP-Split & 0.908 (0.005) & 2.799 (0.101) & 0.170 (0.021) \\
\addlinespace
\multirow{3}{*}{25000}
 & DP-SCP-F & 0.907 (0.006) & 2.192 (0.058) & 0.332 (0.019) \\
 & DP-SCP-A & 0.904 (0.006) & \textbf{2.164} (0.057) & \textbf{0.340} (0.019) \\
 & DP-Split & 0.908 (0.005) & 2.666 (0.104) & 0.199 (0.023) \\
\addlinespace
\multirow{3}{*}{30000}
 & DP-SCP-F & 0.907 (0.005) & 2.093 (0.057) & 0.365 (0.022) \\
 & DP-SCP-A & 0.904 (0.005) & \textbf{2.071} (0.056) & \textbf{0.371} (0.021) \\
 & DP-Split & 0.906 (0.006) & 2.485 (0.075) & 0.242 (0.020) \\
\bottomrule
\end{tabular}
\end{table}

The results in Tables~\ref{tab:results_eps_0.5}, \ref{tab:results_eps_1.0}, and \ref{tab:results_eps_2.0} corroborate the intended efficiency trade-off. We summarize the main observations through four aspects, namely statistical validity, sample-size scaling, privacy--utility trends across $\varepsilon$, and the fixed-buffer--unbuffered gap.

Across all privacy regimes and sample sizes, all methods remain close to the nominal level $1-\alpha=0.9$. DP-SCP-F is consistently the most conservative, which aligns with the additional stability correction in the buffered right-endpoint search. For instance, at $\varepsilon=0.5$ and $n=10{,}000$, DP-SCP-F attains coverage $0.927\,(0.009)$, whereas DP-SCP-A and DP-Split yield $0.902\,(0.007)$ and $0.904\,(0.011)$, respectively. As $n$ increases, this conservatism weakens. At $\varepsilon=2.0$ and $n=30{,}000$, the three coverage values tighten to $0.907\,(0.005)$, $0.904\,(0.005)$, and $0.906\,(0.006)$.

The full-data advantage becomes increasingly clear as $n$ grows. All methods improve with larger sample size, with efficiency increasing through smaller prediction sets and informativeness increasing through higher singleton rates. The gain is steeper for DP-SCP than for DP-Split, which is consistent with avoiding the $n/2$ training bottleneck. At $\varepsilon=2.0$ and $n=30{,}000$, DP-SCP-A attains average set size $2.071\,(0.056)$ and singleton rate $0.371\,(0.021)$, compared with $2.485\,(0.075)$ and $0.242\,(0.020)$ for DP-Split. The same pattern is already visible under stronger privacy. At $\varepsilon=0.5$ and $n=30{,}000$, DP-SCP-A reduces the average set size from $2.537\,(0.107)$ to $2.202\,(0.071)$ and raises the singleton rate from $0.235\,(0.023)$ to $0.332\,(0.024)$.

Holding $n$ fixed, larger $\varepsilon$ generally improves utility, as privacy noise weakens in both training and calibration. For example, at $n=30{,}000$, DP-SCP-A improves from average set size $2.202\,(0.071)$ and singleton rate $0.332\,(0.024)$ at $\varepsilon=0.5$ to $2.071\,(0.056)$ and $0.371\,(0.021)$ at $\varepsilon=2.0$. DP-Split shows the same qualitative trend, but remains consistently less informative than DP-SCP at every $(n,\varepsilon)$ combination. Even in the high-noise regime $\varepsilon=0.5$, the benefit of full-data training remains substantial. At $n=10{,}000$, DP-SCP-A achieves average set size $2.976\,(0.123)$ compared with $3.339\,(0.139)$ for DP-Split, while the singleton rate increases from $0.060\,(0.024)$ to $0.149\,(0.021)$.

DP-SCP-F remains systematically more conservative than DP-SCP-A, as reflected in both higher coverage and larger prediction sets. This gap is largest at small $n$ and under stronger privacy. At $\varepsilon=0.5$ and $n=10{,}000$, the average set sizes are $3.243\,(0.126)$ for DP-SCP-F and $2.976\,(0.123)$ for DP-SCP-A, while the corresponding coverages are $0.927\,(0.009)$ and $0.902\,(0.007)$. As either $n$ or $\varepsilon$ increases, the gap shrinks. At $\varepsilon=2.0$ and $n=30{,}000$, the average set sizes become $2.093\,(0.057)$ and $2.071\,(0.056)$, and the singleton rates are $0.365\,(0.022)$ and $0.371\,(0.021)$, respectively. This supports the intended interpretation of DP-SCP-F as a conservative fixed-buffer safeguard and DP-SCP-A as a sharper asymptotic alternative whose behavior moves closer to DP-SCP-F in larger-sample regimes.

\end{document}